\documentclass[letterpaper]{article} %
\usepackage{aaai2027}  %
\usepackage[hyphens]{url}  %
\usepackage{graphicx} %
\usepackage{natbib}  %
\usepackage{caption} %
\usepackage{algorithm}
\usepackage{algpseudocode}
\usepackage{newfloat}
\usepackage{listings}
\DeclareCaptionStyle{ruled}{labelfont=normalfont,labelsep=colon,strut=off} %
\floatstyle{ruled}
\newfloat{listing}{tb}{lst}{}
\floatname{listing}{Listing}

\usepackage[utf8]{inputenc} %
\usepackage{url}            %
\usepackage{booktabs}       %
\usepackage{amsfonts}       %
\usepackage{nicefrac}       %
\usepackage{xcolor}         %
\usepackage{amsmath}
\usepackage{multirow}
\usepackage{array}
\usepackage{colortbl}
\usepackage{amssymb}
\usepackage{mathtools}
\usepackage[most]{tcolorbox}
\definecolor{algblue}{RGB}{86,142,150}
\definecolor{algpink}{RGB}{220,80,150}
\definecolor{algviolet}{RGB}{110,110,255}
\definecolor{alggray}{RGB}{70,70,70}
\definecolor{algnavy}{RGB}{38,55,120}

\usepackage{subcaption}
\newcommand{\tblstrut}{\rule{0pt}{2.6ex}}
\newcommand{\xtwo}{\ensuremath{\mathord{\scriptscriptstyle\times2}}} %
\usepackage{tabularx} %
\usepackage{pifont}
\newcommand{\cmark}{\ding{51}}
\newcommand{\xmark}{\ding{55}}
\usepackage{amsthm}
\usepackage{comment}
\theoremstyle{plain}
\newtheorem{proposition}{Proposition}

\usepackage[export]{adjustbox}

\usepackage{float}
\usepackage{xcolor}

\usepackage{bm}
\definecolor{lightgray}{HTML}{D0D3D8}
\definecolor{clementine}{HTML}{FF6666}
\definecolor{mustard}{HTML}{FFD43D}
\definecolor{afmblue}{HTML}{3A98FD}
\definecolor{lightafmblue}{HTML}{66B2FF}
\definecolor{yellow}{HTML}{DEB91C}

\newlength{\colw}
\newlength{\colww}
\newlength{\colwww}
\newcommand{\SuppTOCLine}[2]{%
  \noindent\makebox[\linewidth][l]{#1\dotfill\ #2}%
}

\title{Continuous Adversarial MeanFlow \\for Few-Step Generative Adaptation from Any Parameterization}

\author{
    Yara Bahram\setcounter{footnote}{1}\footnote{Work partially done during an internship at Electronic Arts. $^*$These authors contributed equally.}\textsuperscript{*\rm 1,\rm 2,\rm 3,\rm 4},
    Zahra Dehghani\textsuperscript{*\rm 1,\rm 2,\rm 3,\rm 5},
    Mélodie Desbos\textsuperscript{\rm 1,\rm 2,\rm 3},\\
    Eric Granger\textsuperscript{\rm 1,\rm 2,\rm 3,\rm 5},
    Pablo Piantanida\textsuperscript{\rm 2,\rm 5,\rm 6},
    Mohammadhadi Shateri\textsuperscript{\rm 1,\rm 2,\rm 3}
}

\affiliations{
    \textsuperscript{\rm 1}LIVIA, \textsuperscript{\rm 2}ILLS, \textsuperscript{\rm 3}ÉTS Montreal, \textsuperscript{\rm 4}Electronic Arts,\\
    \textsuperscript{\rm 5}Mila - Quebec AI Institute,
    \textsuperscript{\rm 6}CNRS, CentraleSupélec - Université Paris-Saclay\\
    yara.mohammadi-bahram@livia.etsmtl.ca,
    zahra.dehghani-tafti.1@ens.etsmtl.ca
}

\begin{document}

\maketitle
\begin{abstract}

Training fast generators on new domains with limited data remains challenging for two reasons. First, adapting a pretrained diffusion or flow model leaves its costly multi-step sampling intact, and existing acceleration methods are tied to the source parameterization ($x$, $\epsilon$, $v$, or $u$), so heterogeneous pretrained models share no common acceleration target. Second, while adversarial post-training sharpens generation quality, existing criteria each cover only one end of the step range. Discrete criteria compare jump endpoints and lose their signal as the jump shrinks, while continuous criteria score the instantaneous velocity and misjudge the large jumps of few-step sampling. Neither fits MeanFlow (MF) models, which predict average velocities over intervals of any length.   We propose \textbf{MeanFlow-Transfer (\texttt{MF-T})}, which adapts a source domain model of any parameterization and compresses it into a few-step MF generator through a common \emph{velocity map}. We then introduce \textbf{Continuous Adversarial MeanFlow (\texttt{CAMF})}, which scores the change of a learned potential along each predicted interval jump of MF. \texttt{CAMF} needs no gradient penalty or Jacobian-vector product (JVP). We also prove it stays well-conditioned for every interval length and recovers the instantaneous criterion as the interval vanishes. Adapting four ImageNet-based source models (DiT$^\epsilon$, SiT$^v$, JiT$^x$, iMF$^u$) to five target domains, \texttt{MF-T}~+~\texttt{CAMF} approaches or exceeds the fine-tuned many-step teacher in FID and FDD at up to $125\times$ fewer Neural Function Evaluations (NFEs). Further, \texttt{CAMF} improves \texttt{MF-T}'s few-step FID by $\sim30\%$ on average, outperforming prior adversarial post-training on iMF at every NFE\footnote{\texttt{Project page: }\url{https://yasaman-dt.github.io/CAMFT/}}.

  \end{abstract}

\section{Introduction}
\label{sec:intro}

\begin{figure}[ht]
    \centering
    \includegraphics[width=0.41\textwidth]{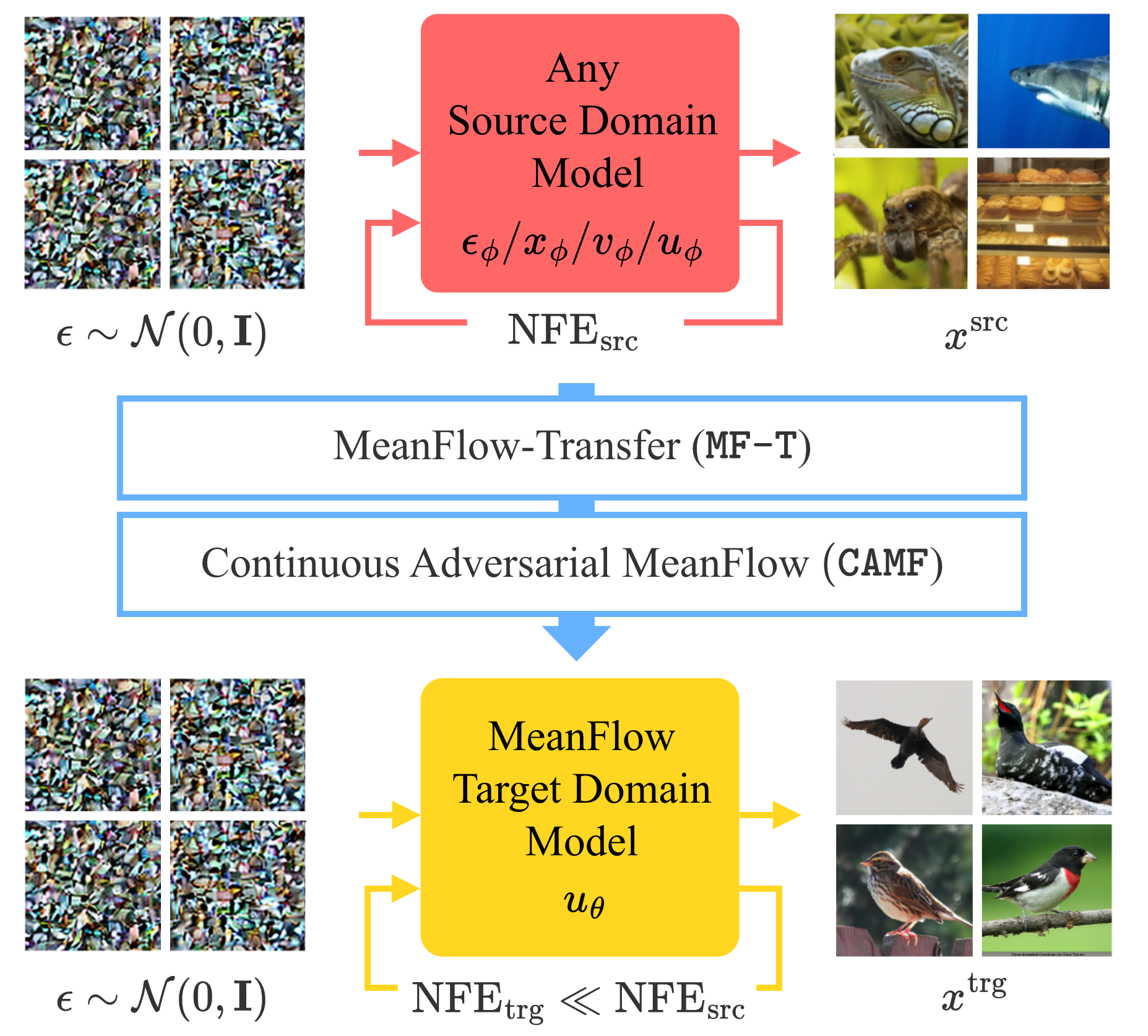}
    \caption{\textbf{\texttt{MF-T}~+~\texttt{CAMF}} adapts a source-domain diffusion or flow model with any output parameterization into a high-quality few-step MeanFlow (MF) generator 
    in a target domain with $\text{NFE}_\text{trg} \mathop{\ll} \text{NFE}_\text{src}$ ($\text{NFE}_\text{trg} \sim \text{NFE}_\text{src}$ if source model is MF). Source and target domains: ImageNet~\cite{deng2009imagenet} and Birds~\cite{wah2011caltech}.}
    \label{fig:mfa_intro}
    \vspace{-10pt}
\end{figure}

Pretrained diffusion and flow-based generators provide strong generic priors for image synthesis, making them attractive source models for transfer to new domains under limited data~\citep{ouyang2024transfer,cao2024few,moon2022fine,xie2023difffit,zhong2025domain,bahram2026dogfit,zhong2024diffusion,bahram2026uni}. Yet, adaptation leaves the source generator's multi-step procedure intact, limiting use in applications that require fast and interactive generation. In parallel, few-step generation has advanced rapidly, but these methods are typically standalone generation paradigms rather than recipes for transferring a pretrained model to a new domain~\cite{salimans2022progressive,song2023consistency,yin2024one,sauer2024adversarial,geng2025mean}. Acceleration and adaptation have thus been pursued separately, and where combined, only within a single source parameterization~\cite{bahram2026uni}. 

This exposes the first obstacle: pretrained generators predict different quantities (clean sample $x$, noise $\epsilon$, instantaneous velocity $v$, or average velocity $u$), so no adapt-and-accelerate procedure applies across model families. MeanFlow (MF) is a natural target. By modeling the average velocity over an interval rather than the instantaneous velocity, it attains high-quality few-step generation~\cite{geng2025mean,geng2026improved}, and Decoupled MeanFlow (DMF) post-trains a flow-matching model into a few-step MF generator~\cite{lee2026decoupled}, though within one parameterization and domain. What remains is an interface bringing heterogeneous sources into MF and adapting them to a new domain.

\begin{figure}[t]
    \centering
    \includegraphics[width=\linewidth]{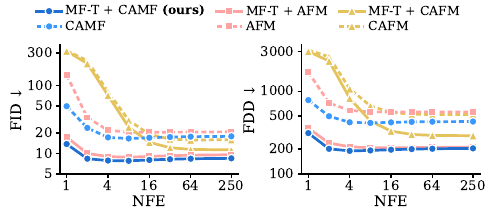}
    \par\vspace{-6pt}
    \caption{\textbf{\texttt{MF-T}~+~\texttt{CAMF}} gives the best target-domain quality on iMF-XL/2 while \textbf{\texttt{MF-T}} initialization consistently improves performance. FID and FDD averaged over five target datasets.}
    \label{fig:headline_nfe}
    \vspace{-10pt}
\end{figure}

A second obstacle is intrinsic to few-step generation. When each sampling step is a large transport jump, regression averages over plausible outputs, producing blurred, less faithful results~\cite{sauer2024adversarial,sauer2024fast}. MF is no exception, and we find this intensifies under limited target data. Adversarial post-training is a strong remedy, but existing criteria each cover only one end of the step range (Fig.~\ref{fig:3_background}). Adversarial flow models (AFMs)~\cite{lin2025adversarial} compare real and generated endpoints of discrete jumps. As the jump shrinks, the endpoints shrink toward the same point and the signal vanishes. AFMs also rely on gradient penalties, discriminator augmentation, and resets for stability~\cite{lin2025adversarial}. Continuous adversarial flow models (CAFMs)~\cite{lin2026continuous} instead use the discriminator's Jacobian-vector product (JVP) to score the instantaneous velocity. Few-step samplers never execute this velocity, and a finite jump along it misses the true endpoint. MFs predict average velocities over intervals of every length, so they need a criterion that remains valid from one-step jumps to the instantaneous limit.

We propose MeanFlow-Transfer (\texttt{MF-T}) and an adversarial post-training stage, Continuous Adversarial MeanFlow (\texttt{CAMF}), addressing both obstacles (Fig.~\ref{fig:mfa_intro}). \texttt{MF-T} converts a source model's output into a shared instantaneous velocity through a \emph{velocity map}, a common interface for $x$-, $\epsilon$-, $v$-, or $u$-predicting sources. It then initializes an MF generator from the source weights and fine-tunes it on target data with the improved-MF (iMF) objective~\cite{geng2026improved}, either directly from the source or after standard fine-tuning. The velocity map keeps transfer stable and yields the best few-step quality, whereas Diff2Flow-style schedule alignment~\cite{schusterbauer2025diff2flow} destabilizes fine-tuning under domain shift at low Neural Function Evaluations (NFEs). \texttt{CAMF} then refines this generator by scoring the change of a learned potential along each predicted interval jump. Unlike AFMs, it needs no gradient penalty, discriminator augmentation, or resets, and unlike CAFMs, it needs no JVP. We prove this score is well-conditioned for every interval length and recovers the CAFM criterion as the interval vanishes.

We evaluate on ImageNet-pretrained DiT$^\epsilon$, SiT$^v$, JiT$^x$, and iMF$^u$ adapted to five target domains. \texttt{MF-T} alone substantially outperforms standard few-step fine-tuning in FID and FDD. Adding \texttt{CAMF} lowers \texttt{MF-T}'s few-step FID by around $30\%$ on average, and on iMF it beats prior adversarial post-training over the full NFE range (Fig.~\ref{fig:headline_nfe}). The gain holds even when data is scarce: on iMF, \texttt{MF-T}~+~\texttt{CAMF} with only $10\%$ of the target data at $2$ NFE outperforms full-data \texttt{MF-T} at $4$ NFE. Consequently, \texttt{MF-T}~+~\texttt{CAMF} approaches or exceeds the fine-tuned many-step teacher at up to $125\times$ fewer NFEs.

\noindent\textbf{Contributions.} \textit{(i)} \texttt{MF-T}, a fine-tuning recipe that adapts a pretrained $x$-, $\epsilon$-, $v$-, or $u$-parameterized model and compresses it into a few-step MF generator through a velocity map. We show that velocity mapping, rather than schedule alignment, stabilizes transfer into MF. \textit{(ii)} \texttt{CAMF}, an adversarial post-training objective for MF's finite-interval average velocities. We prove that it is well-conditioned for every interval length for a Lipschitz discriminator and recovers the instantaneous criterion as the interval vanishes. \textit{(iii)} An evaluation across four source parameterizations and five target domains, showing consistent few-step gains over fine-tuning, and over prior adversarial post-training on iMF.

\section{Related Work}
\label{sec:relatedwork}

\begin{figure*}[t]
\centering
\begin{subfigure}[b]{0.160\textwidth}
    \centering
    \raisebox{10pt}{%
        \includegraphics[width=\linewidth]{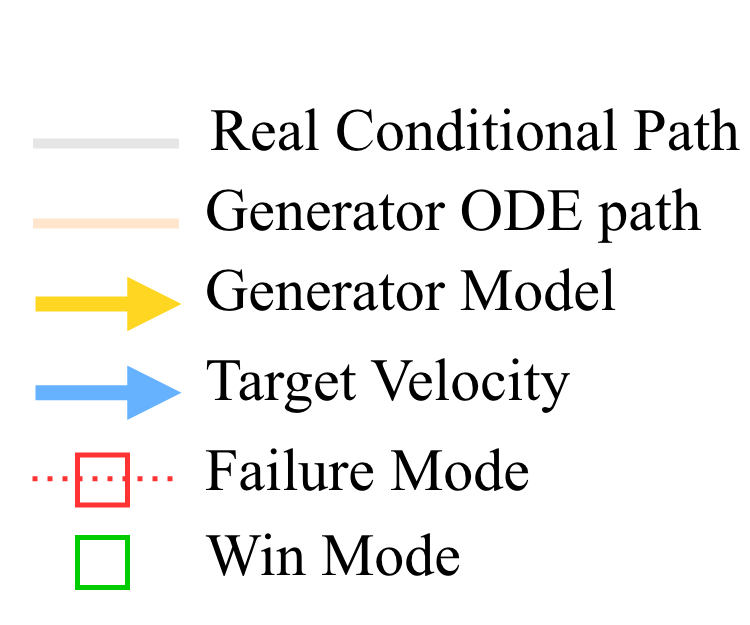}%
    }
    \label{fig:legend}
\end{subfigure}
\begin{subfigure}[b]{0.265\textwidth}
    \centering
    \includegraphics[width=\linewidth]{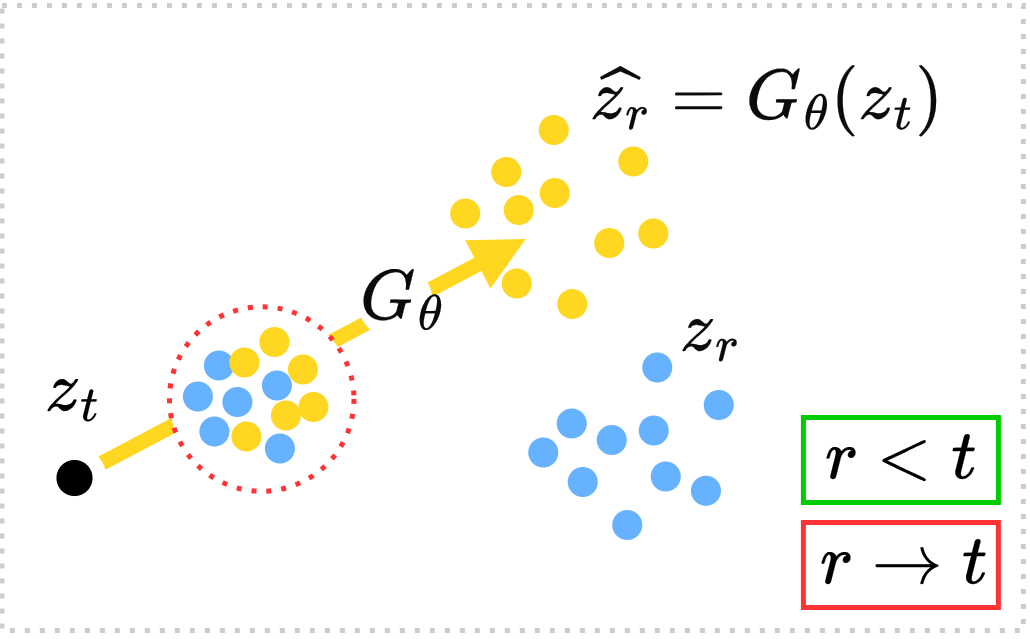}\vspace{-4pt}
    \caption{AFM}
    \label{fig:bg_afm}
\end{subfigure}
\hfill
\begin{subfigure}[b]{0.265\textwidth}
    \centering
    \includegraphics[width=\linewidth]{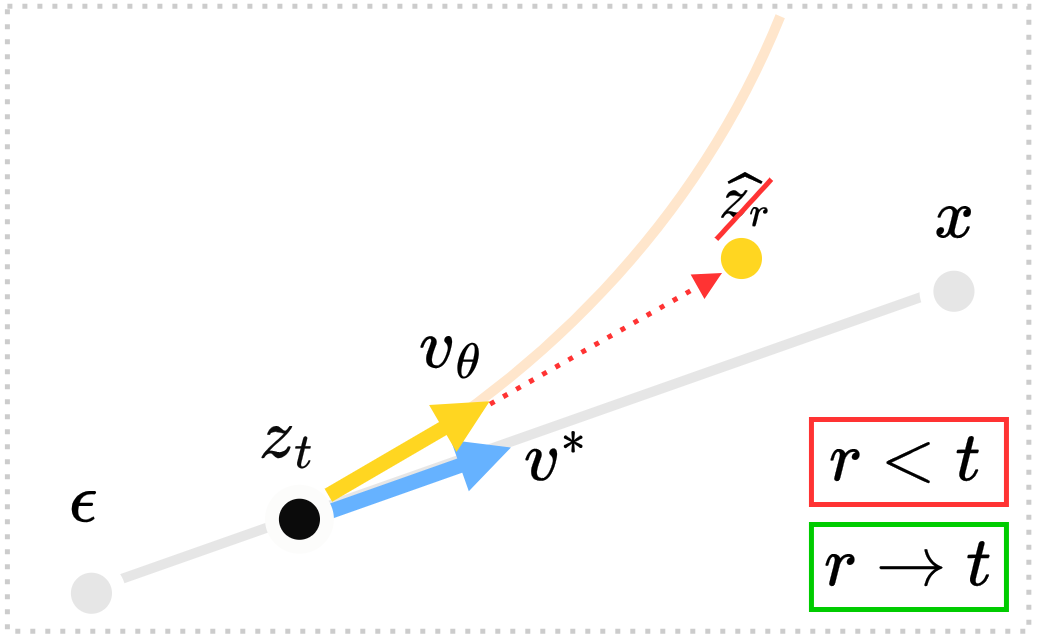}\vspace{-4pt}
    \caption{CAFM}
    \label{fig:bg_cafm}
\end{subfigure}
\hfill
\begin{subfigure}[b]{0.265\textwidth}
    \centering
    \includegraphics[width=\linewidth]{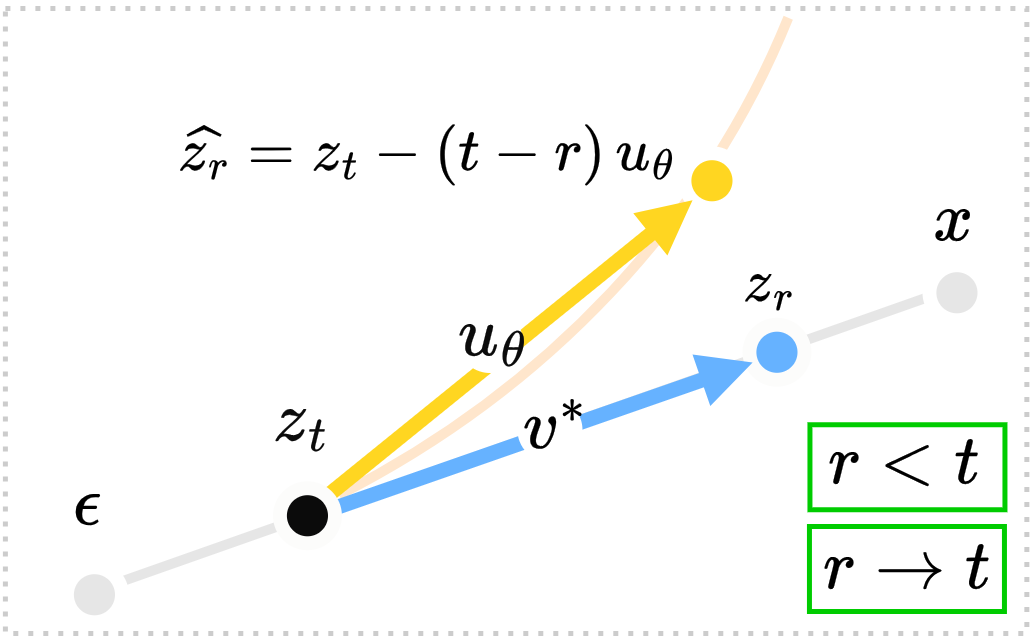}\vspace{-4pt}
    \caption{\texttt{CAMF} (ours)}
    \label{fig:bg_camf}
\end{subfigure}
\caption{\textbf{Adversarial criteria for flow models.}
(a) AFM: the generator maps $z_t$ to an endpoint $\widehat{z}_r=G_\theta(z_t,t,r)$, and the discriminator compares it with the real endpoint $z_r$. As $r\to t$, the two endpoints shrink into the same state and the signal vanishes.
(b) CAFM: the discriminator scores the instantaneous velocity $v_\theta$ at $z_t$, but a few-step MF sampler takes finite jumps of size $t-r$, and a jump along $v_\theta$ misses the true endpoint.
(c) \texttt{CAMF}: the average velocity $u_\theta$ gives the exact endpoint $\widehat{z}_r=z_t-(t-r)\,u_\theta$, which is compared with the real endpoint $z_r$ reached along $v^*$.
}
\vspace{-10pt}
\label{fig:3_background}
\end{figure*}

\noindent\textbf{Accelerating Diffusion and Flow Models.}
Recent work has devoted substantial effort to accelerating the sampling speed of diffusion and flow models. Distillation-based approaches compress a pretrained many-step sampler into a few-step generator while preserving sample quality~\cite{salimans2022progressive,yin2024one,sauer2024adversarial}, while adversarial objectives further recover lost perceptual quality under aggressive distillation~\cite{sauer2024adversarial,sauer2024fast,lin2025diffusion,novack2025fast}. Consistency Models learn few-step generators by enforcing self-consistency along denoising trajectories, either by distillation from a pretrained model or by training from scratch~\cite{song2023consistency}. This idea has been further extended to \emph{flow maps} and particularly MF, which models the \emph{average velocity} between arbitrary timestep pairs~\cite{geng2025mean}. 
Accelerated models, however, remain confined to their training domain and teacher parameterization. Consequently, their flexibility for model-agnostic domain transfer is limited.

\noindent\textbf{Adaptation and the Acceleration Gap.}
Pretrained diffusion and flow models transfer well to new domains under limited data, motivating extensive work that controls how source priors are forgotten, retained, or re-injected during target-domain adaptation~\cite{ouyang2024transfer,hur2024expanding,zhong2024diffusion,zhong2025domain}. 
However, most of these approaches still rely on multi-step sampling at inference, leaving the \emph{sampling cost} of adapted generators unaddressed. In response, recent work explores the intersection of \emph{adaptation} and \emph{acceleration}~\cite{miao2024tuning,chadebec2025flash,yin2024improved,hsiao2024plug,luo2023lcm}: Uni-DAD unifies the two in a single framework~\cite{bahram2026uni}, and DogFit uses domain-guided fine-tuning for efficient target-domain sampling~\cite{bahram2026dogfit}. Yet none of these methods simultaneously achieves fast sampling, high target-domain fidelity, and compatibility across heterogeneous pretrained diffusion and flow generators.

\noindent\textbf{Prediction and Loss Spaces.} A recent line of work decouples a model's output parameterization from its training objective: JiT predicts the clean sample $x$ directly in pixel space while training under a velocity objective~\cite{li2026back}, and pMF extends this prediction/loss decoupling to the MF formulation~\cite{lu2026onestep}. Diff2Flow post-trains a diffusion model with a flow-matching objective by mapping both the output parameterization and noise schedule, showing that naive flow-matching fine-tuning underperforms without this explicit mapping~\cite{schusterbauer2025diff2flow}.

\section{Proposed Method}
\label{sec:3_proposed_method}

\begin{figure*}[t]
\centering

\newlength{\mainfigheight}
\setlength{\mainfigheight}{4.5cm}

\begin{subfigure}[b]{0.35\textwidth}
    \centering
    \includegraphics[width=\linewidth]{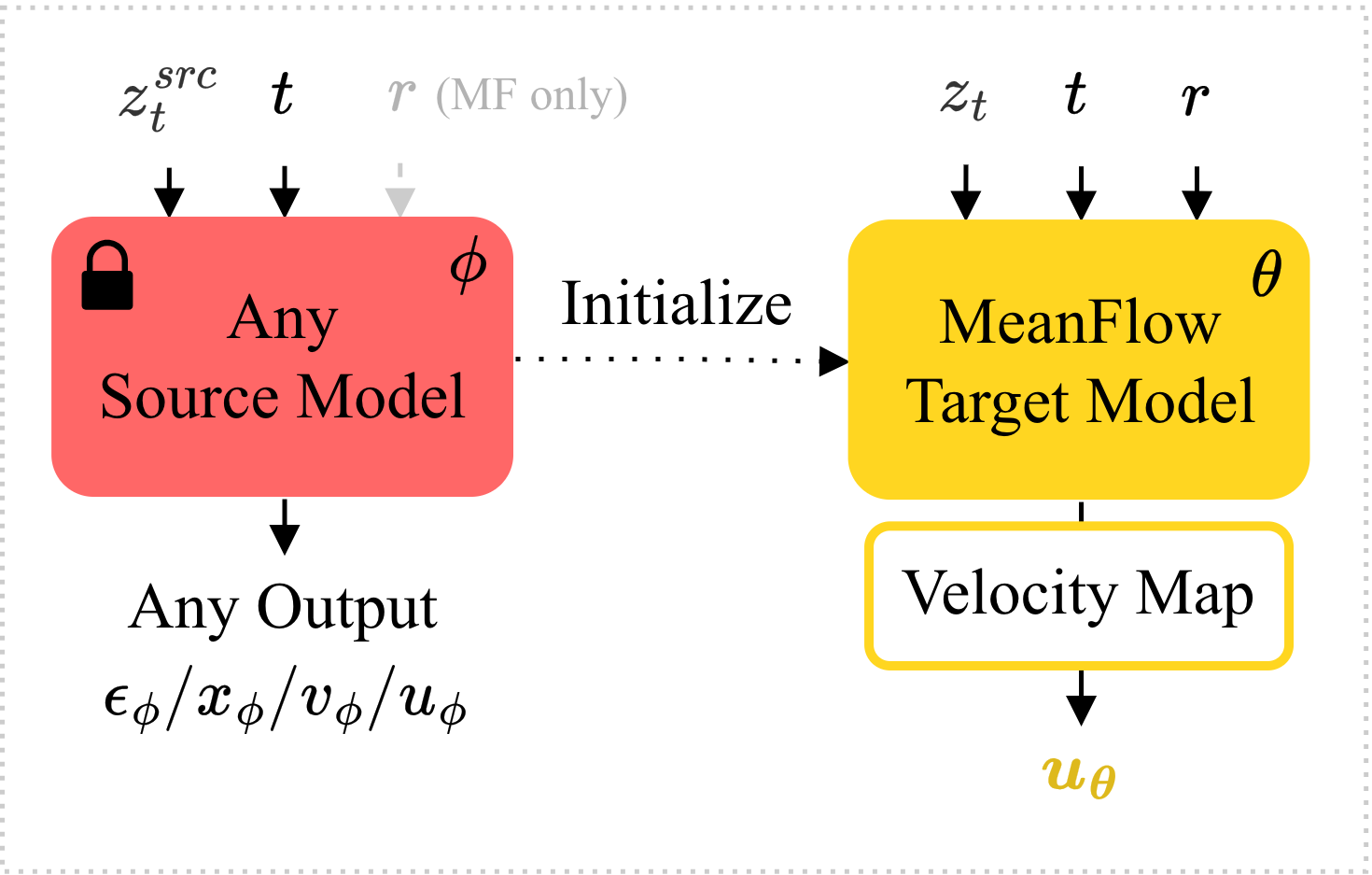}\vspace{-4pt}
    \caption{MeanFlow-Transfer (\texttt{MF-T})}
    \label{fig:a}
\end{subfigure}
\hfill
\begin{subfigure}[b]{0.63\textwidth}
    \centering
    \includegraphics[width=\linewidth]{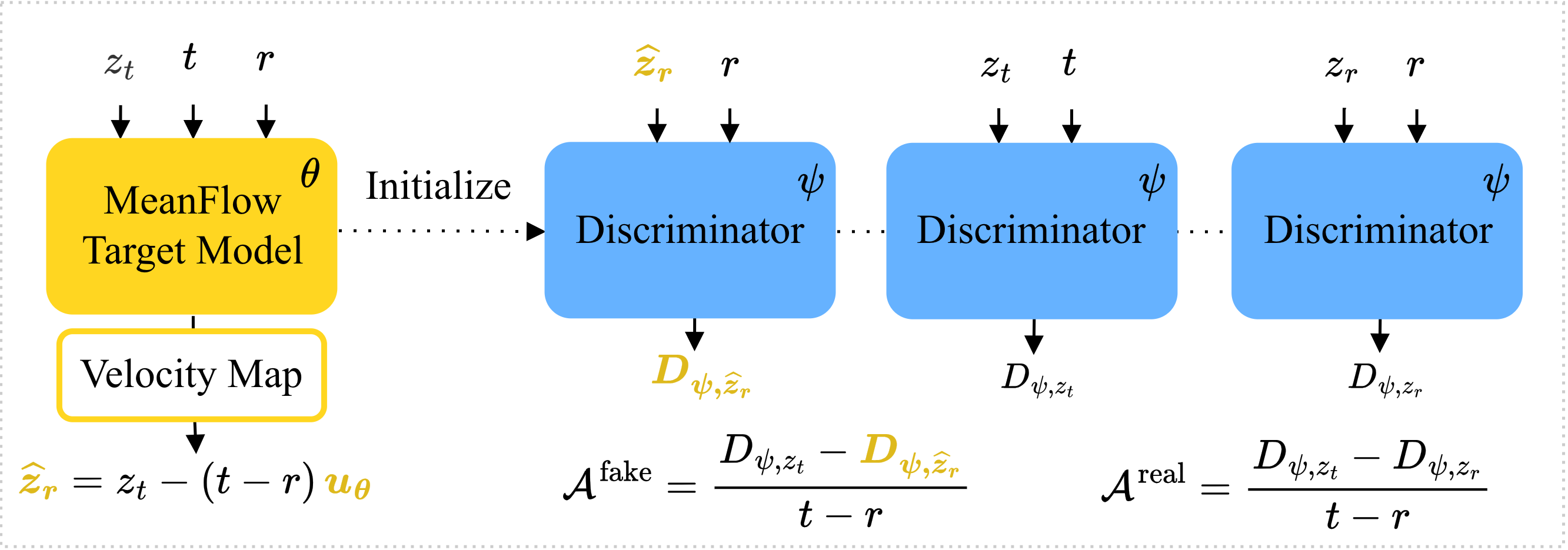}\vspace{-4pt}
    \caption{Continuous Adversarial MeanFlow (\texttt{\texttt{CAMF}})}
    \label{fig:c}
\end{subfigure}
\hfill
\caption{\textbf{Overview of our framework.}
(a) \texttt{MF-T}: a pretrained source generator with an arbitrary output parameterization is converted by the velocity map into a shared instantaneous velocity and adapted into a target-domain MeanFlow (MF) model that predicts the finite-interval average velocity \textcolor{yellow}{$\bm{u_{\theta}}$}.
(b) \texttt{\texttt{CAMF}}: the MF model transports $z_t$ to the predicted endpoint $\textcolor{yellow}{\bm{\widehat{z}_r}} = z_t-(t-r)\textcolor{yellow}{\bm{u_{\theta}}}$, and the discriminator compares the
average change of the potential $D_\psi$ along the real and model-predicted target-domain transport segments, through $\mathcal{A}^{\text{real}}$
and $\mathcal{A}^{\text{fake}}$.
}
\vspace{-10pt}
\label{fig:3_main_method}
\end{figure*}

We address the source-to-target adaptation of pretrained generative models under heterogeneous source parameterizations. A pretrained \emph{source
 generator} with parameters $\phi$ may be an $x$-, $\epsilon$-, $v$-, or $u$-predictor, producing samples $x^{\mathrm{src}}$ from a source
 distribution $p_{\mathrm{src}}$. Given limited target-domain data from $p_{\mathrm{trg}}$, our goal is to obtain a few-step \emph{target generator} $u_\theta$ with parameters $\theta$
 that is both adapted to $p_{\mathrm{trg}}$ and fast. Our framework reaches this goal in two stages: MeanFlow-Transfer (\texttt{MF-T}) and Continuous Adversarial MeanFlow (\texttt{CAMF}) (see Fig.~\ref{fig:3_main_method}). %

\subsection{Background}
\label{sec:3_1_background}

\paragraph{Flow Matching (FM).}
FM~\cite{lipman2023flow,liu2023flow,albergo2023building} learns a time-dependent velocity field transporting a Gaussian prior
$\epsilon\sim\mathcal{N}(0,I)$ to data $x\sim p_{\mathrm{data}}$. Let $z_t$ be an intermediate state on a probability path between $x$ and
$\epsilon$, with instantaneous velocity $v_t=\tfrac{dz_t}{dt}$. FM targets the marginal velocity $v(z,t)
    =
    \mathbb{E}_{p_t(v_t|z_t)}
    \left[
       v_t
    \right]
$,
which is not directly accessible. For the linear path $z_t=(1-t)x+t\epsilon$, the conditional velocity is $v^*(x,\epsilon,t)=\epsilon-x$, and a
network $v_\theta$ is trained by regression:
\begin{equation}\label{eq:FM_regression}
\mathcal{L}_{\mathrm{FM}}=\mathbb{E}_{t,x,\epsilon}\big[\|v_\theta(z_t,t)-v^*(x,\epsilon,t)\|^2\big],
\end{equation}
whose minimizer recovers the marginal field $v(z,t)$~\cite{lipman2023flow}. The formulation applies unchanged to both conditional ($c$) and unconditional ($\varnothing$) generation. We omit $t$ and $c$ and write $v(z_t)$ when clear. The distinction between conditional and marginal velocities is detailed in Appx.~\ref{app:conditional_velocity}.

\paragraph{MeanFlow (MF).}
While FM learns the instantaneous velocity $v(z_t)$, MF~\cite{geng2025mean} models the average velocity over an interval $[r,t]$:
\begin{equation}
u(z_t,r,t)\triangleq\frac{1}{t-r}\int_r^t v(z_\tau)d\tau.
\label{eq:mf_average}
\end{equation}

We write $u(z_t)$ when $(r,t)$ is clear from context. This integral is intractable to supervise directly, so MF instead uses the \emph{MF
identity} $u(z_t)=v(z_t)-(t-r)\frac{d}{dt}u(z_t)$, which links the instantaneous and average velocities~\cite{geng2025mean}. The total derivative is a Jacobian-vector product\footnote{For a function $g$ at primal $p$ along tangent $\tau$, $\texttt{JVP}_{p}(g;\tau)$ denotes
the directional derivative.} (\texttt{JVP}) where $\tfrac{d}{dt}u(z_t)=\texttt{JVP}_{z_t,r,t}(u;v,0,1)=\partial_z uv+\partial_t
u$.
This lets the regression objective be rewritten for a $u$-predictor. iMF~\cite{geng2026improved} defines the \emph{compound velocity}:
\begin{equation}
V_{\theta}(z_t)\triangleq u_{\theta}(z_t)+(t-r)\texttt{sg}\big(\texttt{JVP}_{z_t,r,t}(u_\theta;v_\theta,0,1)\big),
\end{equation}
where $\texttt{sg}$ is stop-gradient and the \texttt{JVP} uses the model's predicted instantaneous velocity $v_\theta$. To allow classifier-free guidance (CFG)~\cite{ho2021classifierfree} at sampling time, iMF regresses a \emph{guided target velocity} $v_{\mathrm{g}}^*
    \coloneqq
    v^*
    +
    \bigl(1-\tfrac{1}{w}\bigr)
    \bigl(
        v_{\theta}(z_t , c) - v_{\theta}(z_t , \varnothing)
    \bigr)$,
where $w$ is the guidance scale.
The model is then trained with an FM-like regression objective, $\mathcal{L}_{\mathrm{iMF}}=\mathbb{E}_{t,r,x,\epsilon,w}\big[\|V_{\theta}(z_t,w)-\texttt{sg}(v^{*}_g)\|^2\big]$, where we omit $w$ when clear from context. After training, iMF samples in a few steps by applying $z_r=z_t-(t-r)u_\theta(z_t,r,t)$ over a coarse sequence of intervals.

\paragraph{Adversarial Flow Models (AFMs).}
\label{sec:3_2_AFM}
AFMs~\cite{lin2025adversarial} replace FM's regression target with a learned adversarial criterion. Their generator is a
flow map $G_\theta(z_t,t,r)$ that transports a state $z_t$ to a lower-time state $z_r$, and a discriminator $D_\psi(\cdot,r)$ scores the generated
\emph{endpoint state} $\widehat z_r=G_\theta(z_t,t,r)$ against the real state $z_r$. This raw-endpoint score loses its signal as the interval
$t-r$ shrinks, and AFMs rely on gradient penalties, discriminator augmentation, and resets for stable training~\cite{lin2025adversarial} (Fig.~\ref{fig:bg_afm}).

\paragraph{Continuous Adversarial Flow Models (CAFMs).}
\label{sec:3_2_CAFM}
CAFMs~\cite{lin2026continuous} instead score a \emph{difference} of a learned potential, and pass to continuous time. They introduce a scalar potential
$D_\psi(z_t,t):\mathbb{R}^n\times[0,1]\to\mathbb{R}$ and read a velocity through its directional derivative along the flow, where
$\tfrac{d}{dt}D_\psi(z_t,t)=\texttt{JVP}_{z_t,t}(D_\psi;v,1)=\partial_z D_\psi v+\partial_t D_\psi$.
Contrasting the target velocity $v^*$ against the model velocity $v_\theta$ with a least-squares game
$f_{\mathrm{ls}}(a,b)=(a-1)^2+(b+1)^2$, the generator and discriminator are trained on reversed pairs:
\begin{align}
\mathcal{L}_{\mathrm{CAFM}}^{G}&=\mathbb{E}\big[f_{\mathrm{ls}}\big(\texttt{JVP}_{z_t,t}(D_\psi;v_\theta,1),\,\texttt{JVP}_{z_t,t}(D_\psi;v^{*},1)\big)
  \big], \nonumber\\
\mathcal{L}_{\mathrm{CAFM}}^{D}&=\mathbb{E}\big[f_{\mathrm{ls}}\big(\texttt{JVP}_{z_t,t}(D_\psi;v^{*},1),\,\texttt{JVP}_{z_t,t}(D_\psi;v_\theta,1)\big)
  \big] \nonumber\\
&\quad+\lambda_{\mathrm{cp}}\,\mathbb{E}\big[D_\psi(z_t,t)^2\big].
\end{align}

\noindent where $\mathbb{E}[\cdot]\equiv \mathbb{E}_{x,\epsilon,t}[\cdot]$. The last term is a \emph{centering penalty}~(cp) as the criterion depends only on differences and derivatives of $D_\psi$, its absolute value can drift, and penalizing $\mathbb{E}[D_\psi^2]$ fixes this gauge (Fig.~\ref{fig:bg_cafm}).

\subsection{Continuous Adversarial MeanFlow Model}
\label{sec:3_CAMF}
CAFM's score $\texttt{JVP}_{z_t,t}(D_\psi;v,1)$ judges an \emph{instantaneous} tangent, a velocity an MF generator never executes, as each of its steps is a finite jump of size $t-r$ (Fig.~\ref{fig:bg_cafm}). We therefore introduce an adversarial post-training stage, initialized from a target-domain MF checkpoint. Given a target-domain sample $x\sim p_\mathrm{trg}$, noise $\epsilon\sim\mathcal{N}(0,I)$, and two times $0\leq r<t\leq1$, the states on the target interpolation path are $z_t=(1-t)x+t\epsilon$ and $z_r=(1-r)x+r\epsilon$. The MF generator predicts an average velocity over $[r,t]$ and induces the \emph{lower-time model endpoint}:
  \begin{equation}
    \widehat{z}_r = z_t-(t-r)\,u_\theta(z_t,r,t,w).
    \label{eq:MF-T_predicted_endpoint}
    \end{equation}

Applied to an MF generator, CAFM trains only the diagonal term $u_\theta(\cdot,t,t,w)$ and it provides no direct gradient to $u_\theta(\cdot,r,t,w)$ for $r<t$. A criterion for MF must therefore score the finite jump of Eq.~\eqref{eq:MF-T_predicted_endpoint} itself  (Fig.~\ref{fig:bg_camf}).

\noindent\textbf{From instantaneous to interval scores.}
Just as MF averages the instantaneous velocity over $[r,t]$ (Eq.~\eqref{eq:mf_average}), we average
CAFM's instantaneous score along the straight path of this jump, $\widehat z_s=z_t-(t-s)\,u_\theta$, $s\in[r,t]$, and define the
\emph{finite-interval discriminator score} as this average. By the fundamental
theorem of calculus, it requires only two evaluations of $D_\psi$ and unlike CAFM, no
\texttt{JVP}:
\begin{align}
\mathcal{A}^{\text{fake}}(z_t,\widehat z_r;t,r)
&\triangleq \frac{1}{t-r}\int_r^t \texttt{JVP}_{\widehat z_s,s}\big(D_\psi;u_\theta,1\big)\,ds \nonumber\\
&= \frac{D_\psi(z_t,t)-D_\psi(\widehat z_r,r)}{t-r}.
\label{eq:finite_interval_discriminator}
\end{align}
where we always sample $t-r\geq10^{-3}$ to avoid large numerical errors in the $1/(t-r)$ factor. We omit the $r$ and $t$ when clear and write $\mathcal{A}^{\text{fake}}(z_t,\widehat z_r)$. Along the real path $z_s$, whose velocity is $v^{*}$, the same identity gives the real score $\mathcal{A}^{\text{real}}(z_t,z_r)$. The adversarial criterion
contrasts these two lower-time endpoints, driving the generator's endpoints toward the real target-domain endpoints with a least-squares game $f_{\mathrm{ls}}$~\cite{lin2026continuous}: 
\begin{equation}
    \mathcal{L}_{\texttt{CAMF}}^{G}
    =
    \mathbb{E}_{x,\epsilon,t,r}
    \Big[
    f_{\mathrm{ls}}\!\big(
      \mathcal{A}^\text{fake}(z_t,\widehat{z}_r),\;
      \mathcal{A}^\text{real}(z_t,z_r)
      \big)
    \Big],
    \label{eq:cafm_gen}
\end{equation}
where only the fake score depends on the generator while the real score is generator-independent. During the discriminator update, $\widehat{z}_r$ is detached from the generator's computation graph, and the discriminator optimizes the reversed pair:
\begin{align}  \mathcal{L}_{\texttt{CAMF}}^{D}
  &=%
      \mathbb{E}_{x,\epsilon,t,r}
      \Big[
      f_{\mathrm{ls}}\!\big(
        \mathcal{A}^\text{real}(z_t,z_r),\;
        \mathcal{A}^\text{fake}(z_t,\widehat{z}_r)
        \big)
      \Big] \label{eq:MF-T_complete_discriminator} \\
      + \lambda_{\mathrm{cp}} &
  \mathbb{E}_{x,\epsilon,t,r}\Big[
  D_\psi(z_t,t)^2
  +
 D_\psi(z_r,r)^2
  +
  D_\psi(\widehat{z}_r,r)^2
  \Big],\nonumber
\end{align}
where we extend CAFM's CP term to control the magnitude of $D_\psi$ at all three evaluated states of \texttt{CAMF}, $z_t$, $z_r$ and $\widehat z_r$ (more in Appx.~\ref{app:centering_ablation}, Tab.~\ref{tab:cp-ablation}). We also investigate whether an additional
optimal-transport regularizer is needed to stabilize adversarial training, finding that omitting it yields better results (Appx.~\ref{app:ot-ablation},~Tab.~\ref{tab:ot-ablation}).

\noindent\textbf{\texttt{CAMF} under different interval lengths.}
We show that, for a differentiable discriminator, the adversarial scores of \texttt{CAMF} converge to those of CAFM as the interval vanishes (Prop.~\ref{prop:camf_consistency}), and that, for a Lipschitz one, \texttt{CAMF} is well-conditioned for every finite interval $r<t$ (Prop.~\ref{prop:camf_conditioning}).

\begin{figure*}[t!]
    \centering
    \includegraphics[width=\linewidth]{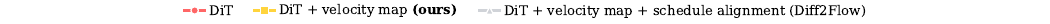}\par\vspace{2pt}  %
    \begin{subfigure}[b]{0.49\linewidth}
        \centering
        \includegraphics[width=\linewidth]{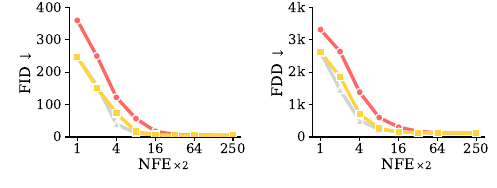}\par\vspace{-4pt}
        \caption{Source domain generation quality before fine-tuning}
        \label{fig:3_analysis_1}
    \end{subfigure}\hfill
    \begin{subfigure}[b]{0.49\linewidth}
        \centering
        \includegraphics[width=\linewidth]{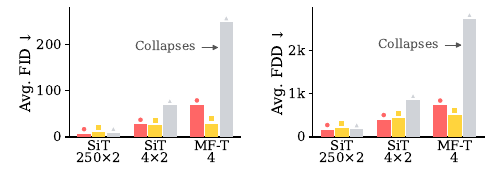}\par\vspace{-4pt}
        \caption{Target domain generation quality after fine-tuning}
        \label{fig:3_analysis_bars}
    \end{subfigure}
    \vspace{-3pt}
    \caption{\textbf{Velocity map vs.\ schedule alignment when transferring DiT.} (a) Both reparameterizations gain few-step efficiency on ImageNet generation before fine-tuning. (b) After fine-tuning using the new objectives (DiT $\rightarrow$ SiT and DiT $\rightarrow$ \texttt{MF-T}) on the target datasets (CUB, Cars, ArtB), schedule alignment destabilizes average generation quality at low NFE, most severely for \texttt{MF-T}. Conversely, the velocity map remains stable and improves generalization.}
    \label{fig:3_analysis}
    \vspace{-10pt}
\end{figure*}

\begin{proposition}[Consistency with CAFM]
\label{prop:camf_consistency}
Fix $(z_t,t)$ and let $\Delta=t-r$. Let $D_\psi$ be differentiable, let the
real endpoint $z_r$ satisfy $z_t-z_r=\Delta\,v^{*}$, and let the predicted
endpoint be $\widehat z_r=z_t-\Delta\,u_\theta(z_t,r,t,w)$, with
$u_\theta(z_t,r,t,w)\to u_\theta(z_t,t,t,w)$ as $r\to t$. Then:
\begin{equation}
\begin{aligned}
\mathcal{A}^{\textnormal{real}}(z_t,z_r;t,r)
&\;\xrightarrow[\;r\to t\;]{}\;
\textnormal{\texttt{JVP}}_{z_t,t}\!\big(D_\psi;\,v^{*},\,1\big),\\[2pt]
\mathcal{A}^{\textnormal{fake}}(z_t,\widehat z_r;t,r)
&\;\xrightarrow[\;r\to t\;]{}\;
\textnormal{\texttt{JVP}}_{z_t,t}\!\big(D_\psi;\,v_\theta,\,1\big),
\\ v_\theta &= u_\theta(z_t,t,t,w).
\end{aligned}
\end{equation}
Thus, as the interval vanishes, the adversarial scores of \textnormal{\texttt{CAMF}} converge to those of \textnormal{CAFM} (proof in Appx.~\ref{app:x_proof_CAMF}).

\end{proposition}

\begin{proposition}[Conditioning for all intervals]
\label{prop:camf_conditioning}
Let $\Delta=t-r>0$, write
$\mathcal{A}^\textnormal{fake}=\mathcal{A}^\textnormal{fake}(z_t,\widehat z_r;t,r)$, and assume
$D_\psi$ is $L_z$-Lipschitz in space and $L_t$-Lipschitz in time. Then, for every $\Delta$:
\begin{align}
|\mathcal{A}^{\textnormal{fake}}| &\le L_t+L_z\,\|u_\theta\|, \label{eq:camf_bound}\\[2pt]
\nabla_\theta\mathcal{L}^{G}_{\textnormal{\texttt{CAMF}}}
&=2\,\mathbb{E}_{x,\epsilon,t,r}\big[(\mathcal{A}^{\textnormal{fake}}-1)\,
\nabla_z D_\psi(\widehat z_r,r)^{\top}\partial_\theta u_\theta\big]. \nonumber
\end{align}
Hence, the fake score is bounded and the generator gradient carries no $1/\Delta$ factor, so \textnormal{\texttt{CAMF}} is well-conditioned for every interval length (proof in Appx.~\ref{app:x_conditioning}).
\end{proposition}

\texttt{CAMF} is guidance-aware, as the generator and discriminator share a sampled CFG scale $w$. Sampling is also unchanged from MF, as the discriminator is used only during training (\texttt{CAMF} post-training in Appx.~\ref{app:training_algorithms} Alg.~\ref{alg:caimf-training}).

\subsection{MeanFlow-Transfer}
\label{sec:3_universal_mapping}

\texttt{CAMF} operates on an MF generator that \texttt{MF-T} obtains for pretrained sources with different parameterizations ($x$, $\epsilon$, $v$, or $u$). Any such prediction inverts to the same velocity: on the linear
path $z_t=(1-t)x+t\epsilon$, we recover $v$ from each parameterization  (Tab.~\ref{tab:3_unified_velocity_recovery}). On a non-linear source schedule, such as DiT's DDPM schedule, we map its $\epsilon$-prediction to that schedule's velocity (Tab.~\ref{tab:3_unified_velocity_recovery}, Appx.~\ref{app:a_diff2flow}).
We call this conversion the \emph{velocity map}, which wraps the target MF generator's output. We clamp the denominator at a small constant ($\leq 10^{-3}$) to avoid singular values for $x$ ($t=0$) and $\epsilon$ ($t=1$). For single-timestep sources ($x$-, $\epsilon$-, or $v$-predictors), we follow DMF~\cite{lee2026decoupled} and add a second timestep input to the model, so that it can be fine-tuned into an MF model. Fine-tuning uses the iMF objective (Sec.~\ref{sec:3_1_background}, \texttt{MF-T} initialization and training in Algs.~\ref{alg:mfa_init} and~\ref{alg:mfa_train} in Appx.~\ref{app:training_algorithms}).

\begin{table}[h]
\centering

\begin{minipage}[t]{\linewidth}
  \vspace{0pt}
  \centering
  \small
  \setlength{\tabcolsep}{8pt}
  \renewcommand{\arraystretch}{1.5}
  \captionof{table}{\textbf{Velocity map} on each source prediction space for representative source model and schedules.}
  \vspace{-3pt}
  \label{tab:3_unified_velocity_recovery}
  \resizebox{\linewidth}{!}{%
  \begin{tabular}{c c c c}
  \toprule
  \rowcolor{gray!7}\multicolumn{4}{c}{\emph{Source prediction space}} \\
  \arrayrulecolor{black!10}\cmidrule(lr){1-4}\arrayrulecolor{black}
  $\bm{x}$ & $\bm{\epsilon}$ & $\bm{v}$ & $\bm{u}$ \\
  \midrule
  \rowcolor{gray!7}\multicolumn{4}{c}{\emph{Recovered instantaneous velocity on the linear path}} \\
  \arrayrulecolor{black!10}\cmidrule(lr){1-4}\arrayrulecolor{black}
  \rule[-2.6ex]{0pt}{6.4ex}$\dfrac{z_t - x_{\phi}(z_t)}{t}$
  & $\dfrac{\epsilon_{\phi}(z_t) - z_t}{1-t}$
  & $v_{\phi}(z_t)$
  & $u_{\phi}(z_t,t,t)$ \\
  \midrule
  \rowcolor{gray!7}\multicolumn{4}{c}{\emph{Source model and schedule}} \\
  \arrayrulecolor{black!10}\cmidrule(lr){1-4}\arrayrulecolor{black}
  JiT -- Linear & DiT -- DDPM & SiT -- Linear & iMF -- Linear \\
  \midrule
  \rowcolor{gray!7}\multicolumn{4}{c}{\textbf{\emph{Recovered instantaneous velocity used} $\boldsymbol{v_{\phi}(z_t)}$}} \\
  \arrayrulecolor{black!10}\cmidrule(lr){1-4}\arrayrulecolor{black}
  \rule[-2.6ex]{0pt}{6.4ex}$\dfrac{z_t - x_{\phi}(z_t)}{t}$
  & $\dfrac{\dot{\bar\alpha}_\tau}{2\bar\alpha_\tau}\Big(z_t-\dfrac{\epsilon_\phi(z_t)}{\sqrt{1-\bar\alpha_\tau}}\Big)$\textsuperscript{\dag}
  & $v_{\phi}(z_t)$
  & $u_{\phi}(z_t,t,t)$ \\
  \bottomrule
  \end{tabular}%
  }
  \begin{flushleft}
  \scriptsize \vspace{-1pt}
\textsuperscript{\dag} $\tau=999\,t$, $\bar\alpha_\tau$ is DiT's cumulative noise schedule, and $\dot{\bar\alpha}_\tau=d\bar\alpha_\tau/dt$.
  \vspace{-10pt}
  \end{flushleft}
\end{minipage}

\end{table}

\paragraph{Velocity map vs. schedule alignment.}
When transferring DiT ($\epsilon$-prediction, DDPM schedule), our velocity map queries DiT on the linear-path state as is and converts its $\epsilon$-prediction to the velocity of the DDPM schedule. Schedule alignment, as in Diff2Flow~\cite{schusterbauer2025diff2flow}, maps the DDPM path onto the linear path, by reparameterizing time and rescaling the state, and then converts the prediction to the linear-path velocity. We find that both gain few-step efficiency on the source domain, while schedule alignment destabilizes fine-tuning under domain shift and low NFE, most severely for the MF objective (Fig.~\ref{fig:3_analysis}, per-dataset numbers in Appx.~Tab.~\ref{tab:diff2flow_full}). We hypothesize this stems from the JVP's brittleness in the MF identity. We therefore adopt the velocity map without schedule alignment for DiT.

  \begin{table*}[t!]
    \centering
    \scriptsize
    \setlength{\tabcolsep}{2.6pt}
    \renewcommand{\arraystretch}{1.2}
    \caption{\textbf{Generation quality comparison with FID and FDD metrics}. \textcolor{afmblue}{AFM}, \textcolor{afmblue}{CAFM} and \textcolor{afmblue}{\texttt{CAMF}} as stand-alone methods are only tested on iMF-XL/2 where the initial model is MF-based (with \texttt{MF-T}~$\equiv$~FT on iMF). \underline{Underlined}: best average per NFE and model initialization. \textbf{Bold}: best average per model initialization.}
    \vspace{-0.2cm}
  \begin{tabular}{
  >{\scriptsize}c
  >{\scriptsize}l
  >{\scriptsize}c
  *{5}{>{\scriptsize}c}
  >{\scriptsize}c
  !{\color{black!20}\vrule}
  *{5}{>{\scriptsize}c}
  >{\scriptsize}c}
  \toprule
  & &
  & \multicolumn{6}{c}{\textbf{FID} $\downarrow$}
  & \multicolumn{6}{c}{\textbf{FDD} $\downarrow$} \\
  \cmidrule(lr){4-9}\cmidrule(lr){10-15}
  \textbf{ImageNet Init.} & \textbf{Method} & \textbf{NFE}
  & ArtB. & Calt. & CUB & Food & Cars & \textbf{Avg.}
  & ArtB. & Calt. & CUB & Food & Cars & \textbf{Avg.} \\
  \midrule

     & \textcolor{afmblue}{AFM} & $250$ 
    & ${\color{gray!90}16.07}$ & ${\color{gray!90}52.61}$ & ${\color{gray!90}8.16}$ & ${\color{gray!90}12.75}$ & ${\color{gray!90}14.81}$ & $20.88$
    & ${\color{gray!90}265.33}$ & ${\color{gray!90}967.30}$ & ${\color{gray!90}552.30}$ & ${\color{gray!90}507.39}$ & ${\color{gray!90}518.98}$ & $562.26$ \\
 & \textcolor{afmblue}{CAFM} & $250$ 
    & ${\color{gray!90}10.42}$ & ${\color{gray!90}36.09}$ & ${\color{gray!90}10.47}$ & ${\color{gray!90}11.94}$ & ${\color{gray!90}10.95}$ & $15.97$
    & ${\color{gray!90}263.60}$ & ${\color{gray!90}740.90}$ & ${\color{gray!90}595.70}$ & ${\color{gray!90}600.10}$ & ${\color{gray!90}414.00}$ & $522.86$ \\
     & \textcolor{afmblue}{\texttt{CAMF}} & $250$ 
    & ${\color{gray!90}12.12}$ & ${\color{gray!90}40.43}$ & ${\color{gray!90}13.06}$ & ${\color{gray!90}11.39}$ & ${\color{gray!90}12.57}$ & $17.91$
    & ${\color{gray!90}217.29}$ & ${\color{gray!90}785.96}$ & ${\color{gray!90}423.50}$ & ${\color{gray!90}430.79}$ & ${\color{gray!90}291.90}$ & $429.89$ \\
   \arrayrulecolor{black!10}\cmidrule(lr){2-15}\arrayrulecolor{black}& \texttt{MF-T} & $250$ 
    & ${\color{gray!90}11.39}$ & ${\color{gray!90}25.36}$ & ${\color{gray!90}4.19}$ & ${\color{gray!90}7.85}$ & ${\color{gray!90}7.93}$ & $11.34$
& ${\color{gray!90}280.86}$ & ${\color{gray!90}426.73}$ & ${\color{gray!90}128.30}$ & ${\color{gray!90}371.05}$ & ${\color{gray!90}244.67}$ & $290.32$\\
     & \texttt{MF-T}~+~AFM & $250$ 
    & ${\color{gray!90}9.74}$ & ${\color{gray!90}22.52}$ & ${\color{gray!90}3.09}$ & ${\color{gray!90}8.22}$ & ${\color{gray!90}4.47}$ & $9.61$
    & ${\color{gray!90}158.32}$ & ${\color{gray!90}360.22}$ & ${\color{gray!90}76.29}$ & ${\color{gray!90}311.92}$ & ${\color{gray!90}150.45}$ & $211.44$ \\
     & \texttt{MF-T}~+~CAFM & $250$ 
    & ${\color{gray!90}7.61}$ & ${\color{gray!90}25.45}$ & ${\color{gray!90}8.01}$ & ${\color{gray!90}6.08}$ & ${\color{gray!90}7.02}$ & $10.83$
    & ${\color{gray!90}206.91}$ & ${\color{gray!90}448.87}$ & ${\color{gray!90}237.93}$ & ${\color{gray!90}350.44}$ & ${\color{gray!90}205.48}$ & $289.93$ \\
 \rowcolor{gray!7}\cellcolor{white}{} & \texttt{MF-T}~+~\texttt{CAMF} & $250$ 
    & ${\color{gray!90}8.20}$ & ${\color{gray!90}22.59}$ & ${\color{gray!90}2.80}$ & ${\color{gray!90}4.99}$ & ${\color{gray!90}3.92}$ & $\underline{8.50}$
    & ${\color{gray!90}167.31}$ & ${\color{gray!90}369.96}$ & ${\color{gray!90}87.77}$ & ${\color{gray!90}274.53}$ & ${\color{gray!90}120.06}$ & $\underline{203.93}$ \\
\arrayrulecolor{black!10}\cmidrule(lr){2-15}\arrayrulecolor{black}
& \textcolor{afmblue}{AFM} & $4$
    & ${\color{gray!90}11.71}$ & ${\color{gray!90}66.72}$ & ${\color{gray!90}10.06}$ & ${\color{gray!90}10.57}$ & ${\color{gray!90}10.83}$ & $21.98$
    & ${\color{gray!90}255.88}$ & ${\color{gray!90}1097.15}$ & ${\color{gray!90}589.22}$ & ${\color{gray!90}487.06}$ & ${\color{gray!90}493.51}$ & $584.56$ \\
  & \textcolor{afmblue}{CAFM} & $4$ 
    & ${\color{gray!90}81.56}$ & ${\color{gray!90}112.93}$ & ${\color{gray!90}77.55}$ & ${\color{gray!90}94.11}$ & ${\color{gray!90}38.99}$ & $81.03$
    & ${\color{gray!90}641.90}$ & ${\color{gray!90}1376.30}$ & ${\color{gray!90}1083.80}$ & ${\color{gray!90}1426.90}$ & ${\color{gray!90}815.90}$ & $1068.96$ \\
  & \textcolor{afmblue}{\texttt{CAMF}} & $4$
  & ${\color{gray!90}10.80}$ & ${\color{gray!90}44.42}$ & ${\color{gray!90}9.34}$ & ${\color{gray!90}10.60}$ &
  ${\color{gray!90}11.19}$ & $17.27$
  & ${\color{gray!90}221.40}$ & ${\color{gray!90}812.31}$ & ${\color{gray!90}367.57}$ & ${\color{gray!90}409.74}$ &
  ${\color{gray!90}293.50}$ & $420.90$ \\
  \arrayrulecolor{black!10}\cmidrule(lr){2-15}\arrayrulecolor{black}
    & \texttt{MF-T} & $4$
  & ${\color{gray!90}10.60}$ & ${\color{gray!90}27.54}$ & ${\color{gray!90}7.26}$ & ${\color{gray!90}10.58}$ & ${\color{gray!90}9.69}$ & $13.13$
  & ${\color{gray!90}273.28}$ & ${\color{gray!90}471.09}$ & ${\color{gray!90}310.62}$ & ${\color{gray!90}465.38}$ & ${\color{gray!90}276.53}$ & $359.38$\\  
  & \texttt{MF-T}~+~AFM & $4$
  & ${\color{gray!90}8.67}$ & ${\color{gray!90}22.19}$ & ${\color{gray!90}3.01}$ & ${\color{gray!90}7.03}$ &
  ${\color{gray!90}4.13}$ & $9.01$
  & ${\color{gray!90}180.29}$ & ${\color{gray!90}371.61}$ & ${\color{gray!90}75.72}$& ${\color{gray!90}298.54}$ &
  ${\color{gray!90}141.25}$ & $213.48$ \\
  \multirow{-12}{*}{%
  \begin{tabular}{@{}c@{}}
    \textbf{iMF-XL/2}\\[-1pt]
    \cite{geng2026improved}
  \end{tabular}%
}
& \texttt{MF-T}~+~CAFM & $4$
    & ${\color{gray!90}97.02}$ & ${\color{gray!90}69.42}$ & ${\color{gray!90}79.00}$ & ${\color{gray!90}89.89}$ & ${\color{gray!90}52.33}$ & $77.53$
    & ${\color{gray!90}657.82}$ & ${\color{gray!90}779.09}$ & ${\color{gray!90}891.57}$ & ${\color{gray!90}1337.62}$ & ${\color{gray!90}760.64}$ & $885.35$ \\
\rowcolor{gray!7}\cellcolor{white}{} & \texttt{MF-T}~+~\texttt{CAMF} & $4$
  & ${\color{gray!90}6.71}$ & ${\color{gray!90}22.10}$ & ${\color{gray!90}2.88}$ & ${\color{gray!90}4.65}$ &
  ${\color{gray!90}3.07}$ & $\underline{\mathbf{7.88}}$
  & ${\color{gray!90}146.40}$ & ${\color{gray!90}378.09}$ & ${\color{gray!90}84.24}$ & ${\color{gray!90}239.28}$ &
  ${\color{gray!90}102.37}$ & $\underline{\mathbf{190.08}}$ \\
  
  \hline

    & FT & $250\xtwo$
  & ${\color{gray!90}10.21}$ & ${\color{gray!90}24.20}$ & ${\color{gray!90}4.19}$ & ${\color{gray!90}6.81}$ & ${\color{gray!90}6.93}$ & $10.47$
  & ${\color{gray!90}205.05}$ & ${\color{gray!90}399.26}$ & ${\color{gray!90}111.66}$ & ${\color{gray!90}343.67}$ & ${\color{gray!90}185.82}$ & $249.09$

  \\
      & \texttt{MF-T} & $250$ 
    & ${\color{gray!90}15.86}$ & ${\color{gray!90}26.20}$ & ${\color{gray!90}4.88}$ & ${\color{gray!90}9.77}$ & ${\color{gray!90}5.71}$ & $12.48$
    & ${\color{gray!90}277.49}$ & ${\color{gray!90}466.68}$ & ${\color{gray!90}189.45}$ & ${\color{gray!90}447.08}$ & ${\color{gray!90}184.35}$ & $313.01$ \\
      & \texttt{MF-T}~+~AFM & $250$ 
    & ${\color{gray!90}13.29}$ & ${\color{gray!90}25.14}$ & ${\color{gray!90}4.29}$ & ${\color{gray!90}9.96}$ & ${\color{gray!90}4.76}$ & $11.49$
    & ${\color{gray!90}201.70}$ & ${\color{gray!90}429.61}$ & ${\color{gray!90}174.52}$ & ${\color{gray!90}322.71}$ & ${\color{gray!90}134.62}$ & $252.63$ 
    
  \\
    & \texttt{MF-T}~+~CAFM & $250$ 
    & ${\color{gray!90}14.59}$ & ${\color{gray!90}32.50}$ & ${\color{gray!90}11.67}$ & ${\color{gray!90}7.82}$ & ${\color{gray!90}4.80}$ & $14.27$
    & ${\color{gray!90}296.33}$ & ${\color{gray!90}569.39}$ & ${\color{gray!90}363.89}$ & ${\color{gray!90}402.29}$ & ${\color{gray!90}175.83}$ & $361.54$ \\
    \rowcolor{gray!7}\cellcolor{white}{} & \texttt{MF-T}~+~\texttt{CAMF} & $250$ 
    & ${\color{gray!90}13.44}$ & ${\color{gray!90}24.05}$ & ${\color{gray!90}3.83}$ & ${\color{gray!90}7.27}$ & ${\color{gray!90}3.64}$ & $\underline{10.45}$
    & ${\color{gray!90}233.39}$ & ${\color{gray!90}410.29}$ & ${\color{gray!90}129.84}$ & ${\color{gray!90}312.04}$ & ${\color{gray!90}116.03}$ & $\underline{\mathbf{240.32}}$
    \\
  \arrayrulecolor{black!10}\cmidrule(lr){2-15}\arrayrulecolor{black}
  & FT & $4\xtwo$
  & ${\color{gray!90}36.11}$ & ${\color{gray!90}33.26}$ & ${\color{gray!90}14.33}$ & ${\color{gray!90}22.22}$ & ${\color{gray!90}34.83}$ & $28.15$
  & ${\color{gray!90}457.69}$ & ${\color{gray!90}543.04}$ & ${\color{gray!90}310.21}$ & ${\color{gray!90}620.82}$ & ${\color{gray!90}624.10}$ & $511.17$ 
  \\
        & \texttt{MF-T} & $4$
  & ${\color{gray!90}14.56}$ & ${\color{gray!90}25.80}$ & ${\color{gray!90}5.25}$ &${\color{gray!90}9.44}$ & ${\color{gray!90}4.77}$ & $11.96$
  & ${\color{gray!90}274.67}$ & ${\color{gray!90}474.45}$ & ${\color{gray!90}205.34}$ & ${\color{gray!90}431.78}$ & ${\color{gray!90}178.80}$ & $313.01$ \\
    \multirow{-6}{*}{
      \begin{tabular}{@{}c@{}}
        \textbf{SiT-XL/2}\\[-1pt]
        \cite{ma2024sit}
      \end{tabular}
    }
    & \texttt{MF-T}~+~AFM & $4$ 
    & ${\color{gray!90}9.28}$ & ${\color{gray!90}24.20}$ & ${\color{gray!90}3.87}$ & ${\color{gray!90}6.89}$ & ${\color{gray!90}3.94}$ & $\underline{\mathbf{9.63}}$
    & ${\color{gray!90}184.68}$ & ${\color{gray!90}431.50}$ & ${\color{gray!90}168.38}$ & ${\color{gray!90}287.26}$ & ${\color{gray!90}132.61}$ & $\underline{240.89}$ \\
    & \texttt{MF-T}~+~CAFM & $4$ 
    & ${\color{gray!90}106.67}$ & ${\color{gray!90}80.56}$ & ${\color{gray!90}75.10}$ & ${\color{gray!90}90.80}$ & ${\color{gray!90}33.16}$ & $77.26$
    & ${\color{gray!90}825.85}$ & ${\color{gray!90}884.84}$ & ${\color{gray!90}917.84}$ & ${\color{gray!90}1166.04}$ & ${\color{gray!90}525.45}$ & $864.00$ \\

  \rowcolor{gray!7}\cellcolor{white}{} &  \texttt{MF-T}~+~\texttt{CAMF} & $4$
  & ${\color{gray!90}12.74}$ & ${\color{gray!90}24.02}$ & ${\color{gray!90}3.91}$ &${\color{gray!90}7.48}$ & ${\color{gray!90}3.33}$ & $10.30$
  & ${\color{gray!90}238.46}$ & ${\color{gray!90}413.38}$ & ${\color{gray!90}133.85}$ & ${\color{gray!90}311.51}$ & ${\color{gray!90}108.73}$ &
  $241.19$ \\
  
  \hline

      \multirow{4}{*}{%
      \begin{tabular}{@{}c@{}}
        \textbf{DiT-XL/2}\\[-1pt]
        {\scriptsize \cite{peebles2023scalable}}
      \end{tabular}
    }  
    & FT & $250\xtwo$
  & ${\color{gray!90}11.90}$ & ${\color{gray!90}24.47}$ & ${\color{gray!90}7.75}$ &${\color{gray!90}9.26}$ & ${\color{gray!90}5.95}$ & $11.87$
  & ${\color{gray!90}211.02}$ & ${\color{gray!90}394.66}$ & ${\color{gray!90}167.75}$ & ${\color{gray!90}277.95}$ & ${\color{gray!90}164.10}$ & $243.10$
  \\
  \arrayrulecolor{black!10}\cmidrule(lr){2-15}\arrayrulecolor{black}
    & FT  & $8\xtwo$
    & ${\color{gray!90}142.86}$ & ${\color{gray!90}67.94}$ & ${\color{gray!90}61.16}$ & ${\color{gray!90}89.77}$ & ${\color{gray!90}99.02}$ & $92.15$
    & ${\color{gray!90}991.02}$ & ${\color{gray!90}812.76}$ & ${\color{gray!90}738.70}$ & ${\color{gray!90}1284.66}$ & ${\color{gray!90}1141.83}$ & $993.79$ \\
  & \texttt{MF-T}  & $8$
    & ${\color{gray!90}33.48}$ & ${\color{gray!90}32.17}$ & ${\color{gray!90}8.27}$ & ${\color{gray!90}33.44}$ & ${\color{gray!90}9.83}$ & $23.44$
    & ${\color{gray!90}325.72}$ & ${\color{gray!90}475.90}$ & ${\color{gray!90}222.81}$ & ${\color{gray!90}614.30}$ & ${\color{gray!90}259.65}$ & $379.68$ \\
  \rowcolor{gray!7}\cellcolor{white}{} & \texttt{MF-T}~+~\texttt{CAMF} & $8$
    & ${\color{gray!90}10.62}$ & ${\color{gray!90}31.18}$ & ${\color{gray!90}3.70}$ & ${\color{gray!90}6.46}$ & ${\color{gray!90}3.90}$ &
  $\underline{\mathbf{11.17}}$
    & ${\color{gray!90}162.34}$ & ${\color{gray!90}477.88}$ & ${\color{gray!90}100.97}$ & ${\color{gray!90}201.57}$ & ${\color{gray!90}115.86}$ &
  $\underline{\mathbf{211.72}}$ \\

  \hline
  \multirow{4}{*}{
  \begin{tabular}{@{}c@{}}
  \textbf{JiT-H/16}\\[-1pt]
  {\scriptsize \cite{li2026back}}
  \end{tabular}}
  & FT & $50\xtwo$
  & ${\color{gray!90}15.27}$ & ${\color{gray!90}35.72}$ & ${\color{gray!90}5.68}$ &${\color{gray!90}12.47}$ & ${\color{gray!90}15.70}$ & $\mathbf{16.97}$
  & ${\color{gray!90}359.10}$ & ${\color{gray!90}602.10}$ & ${\color{gray!90}278.70}$ & ${\color{gray!90}425.08}$ & ${\color{gray!90}332.44}$ & $\mathbf{399.48}$
  \\
  \arrayrulecolor{black!10}\cmidrule(lr){2-15}\arrayrulecolor{black}
  & FT & $4\xtwo$
  & ${\color{gray!90}44.03}$ & ${\color{gray!90}66.05}$ & ${\color{gray!90}25.03}$ &${\color{gray!90}37.60}$ & ${\color{gray!90}40.98}$ & $42.74$
  & ${\color{gray!90}599.80}$ & ${\color{gray!90}939.96}$ & ${\color{gray!90}524.17}$ & ${\color{gray!90}820.81}$ & ${\color{gray!90}767.51}$ & $730.45$
  \\
  & \texttt{MF-T} & $4$
  & ${\color{gray!90}24.07}$ & ${\color{gray!90}60.98}$ & ${\color{gray!90}14.65}$ & ${\color{gray!90}24.80}$ & ${\color{gray!90}21.22}$ & $29.14$
  & ${\color{gray!90}360.67}$ & ${\color{gray!90}913.56}$ & ${\color{gray!90}417.14}$ & ${\color{gray!90}768.58}$ & ${\color{gray!90}619.10}$ & $615.81$
  \\
    \rowcolor{gray!7}\cellcolor{white}{}& \texttt{MF-T}~+~\texttt{CAMF} & $4$
    & ${\color{gray!90}19.62}$ & ${\color{gray!90}53.02}$ & ${\color{gray!90}7.29}$ & ${\color{gray!90}18.07}$ & ${\color{gray!90}12.55}$ &
  $\underline{22.11}$
    & ${\color{gray!90}317.47}$ & ${\color{gray!90}770.15}$ &${\color{gray!90}311.59}$ & ${\color{gray!90}589.29}$ & ${\color{gray!90}380.78}$ &
  $\underline{473.86}$ \\
  \bottomrule
  \end{tabular}
    \label{tab:main_labeled_init_teachers}
    \begin{minipage}{\textwidth}
\raggedright
\scriptsize
\textsuperscript{\dag} At the corresponding NFE settings, FT produces samples
that are perceptually indistinguishable from noise.
\end{minipage} 
\vspace{-12pt}
  \end{table*}

\section{Results and Discussion}
\label{sec:results}

\subsection{Experimental Setup}
\label{sec:Experimental_Setup}

As source models, we use ImageNet generators spanning all four prediction parameterizations: DiT-XL/2~\cite{peebles2023scalable} ($\epsilon$), SiT-XL/2~\cite{ma2024sit} ($v$), JiT-H/16~\cite{li2026back} ($x$), and iMF-XL/2~\cite{geng2026improved} ($u$). We adapt each to a diverse set of labeled target domains: ArtBench (ArtB.)~\citep{liao2022artbench}, Caltech (Calt.)~\cite{griffin2007caltech}, CUB-Birds (CUB)~\cite{wah2011caltech}, Food~\cite{bossard2014food}, and Stanford-Cars (Cars)~\citep{krause20133d}, chosen to span diverse semantic domains and domain-shift magnitudes~\cite{bahram2026dogfit}. We evaluate \texttt{MF-T} and \texttt{CAMF} both separately and combined, comparing against standard fine-tuning (FT) and the adversarial baselines, AFM~\cite{lin2025adversarial} and CAFM~\cite{lin2026continuous}. We report quality at different NFEs using FID and FD$_{\mathrm{DINOv2}}$ (FDD) computed between $10$K generated samples and the full target dataset at $256\times256$ resolution. We also report IS and fidelity-diversity trade-offs using Precision-Recall and Density-Coverage. For FT, CFG doubles the NFE (e.g., $250\xtwo$). We apply a fixed training budget for each model family within which all models reach their best FID and FDD. For checkpoint selection, we use the lowest-FID checkpoint at $4$ NFE for few-step models and at $64$ NFE for many-step models. Results using the final checkpoints keep \texttt{MF-T}~+~\texttt{CAMF} superior, with occasional instability in CAFM, and are reported in Appx.~\ref{app:checkpoint_selection}~Tab.~\ref{tab:posttrain_final} (Hyperparameters, training details, and computational costs in Appx.~\ref{sec:additional_experimental}).

\newcommand{\dittwolinenfe}{\includegraphics[width=\colwww,height=\colwww,keepaspectratio]{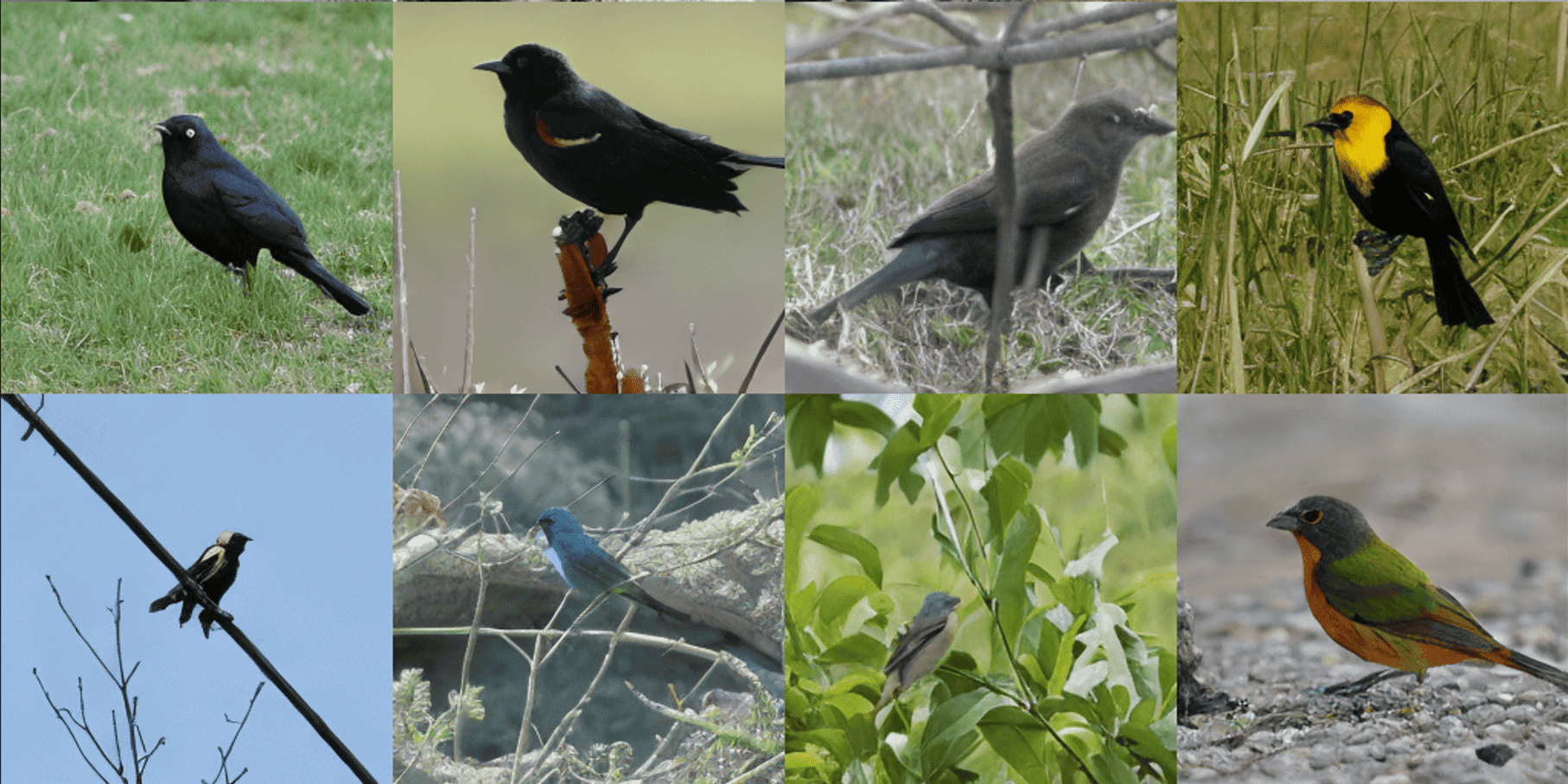}}

\newcommand{\jittwolinenfe}{\includegraphics[width=\colwww,height=\colwww,keepaspectratio]{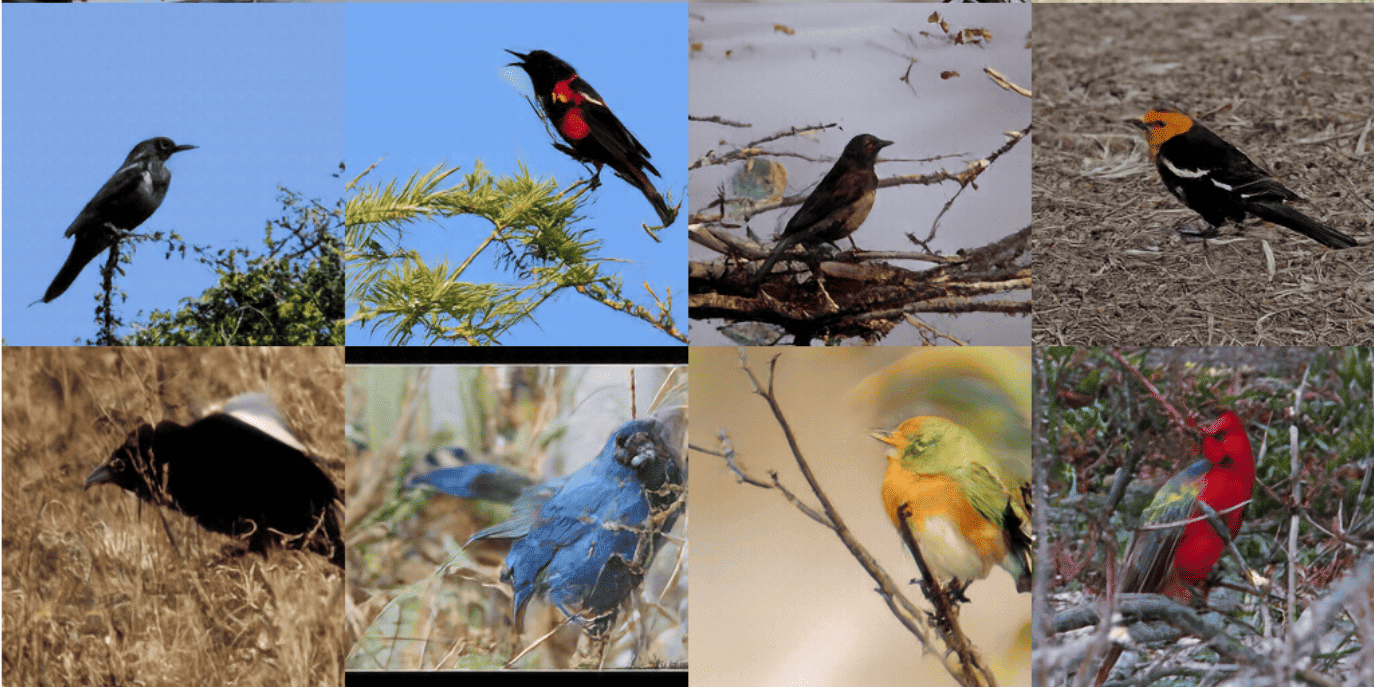}}

\newcommand{\imftwolinesfournfe}{\includegraphics[width=\colwww,height=\colwww,keepaspectratio]{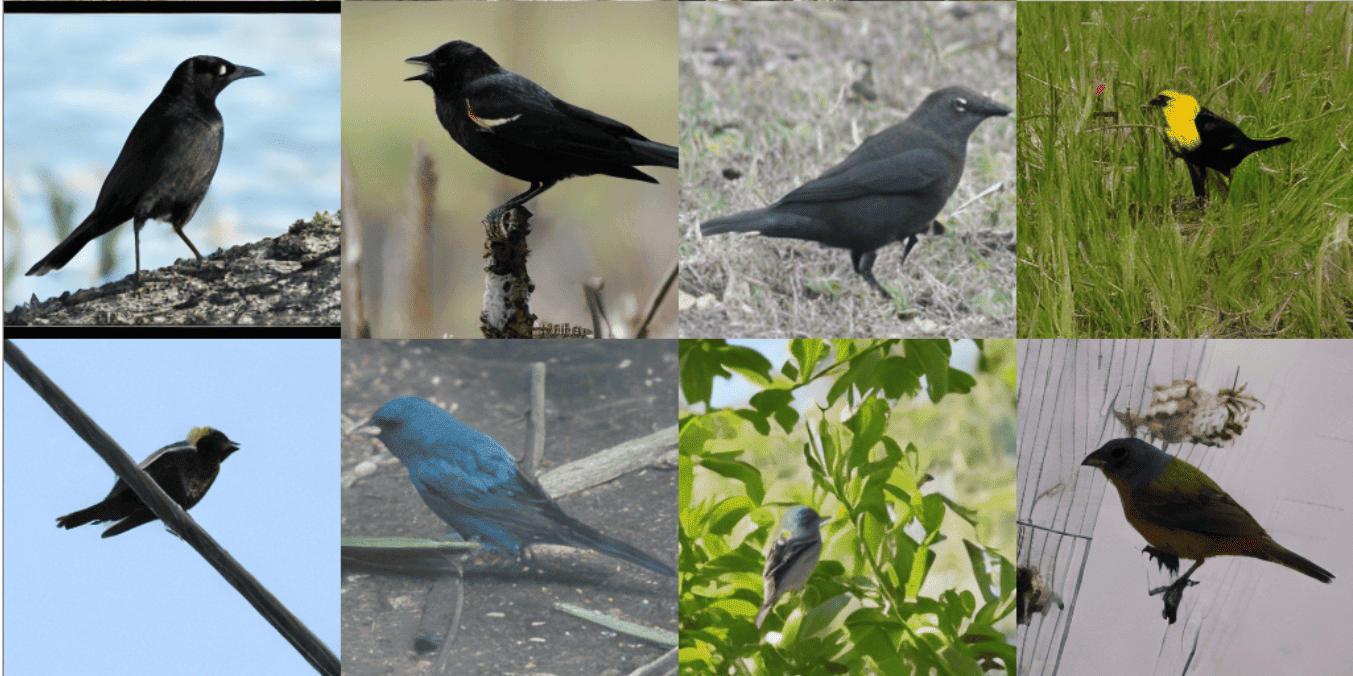}}

\newcommand{\sittwolinenfe}{\includegraphics[width=\colwww,height=\colwww,keepaspectratio]{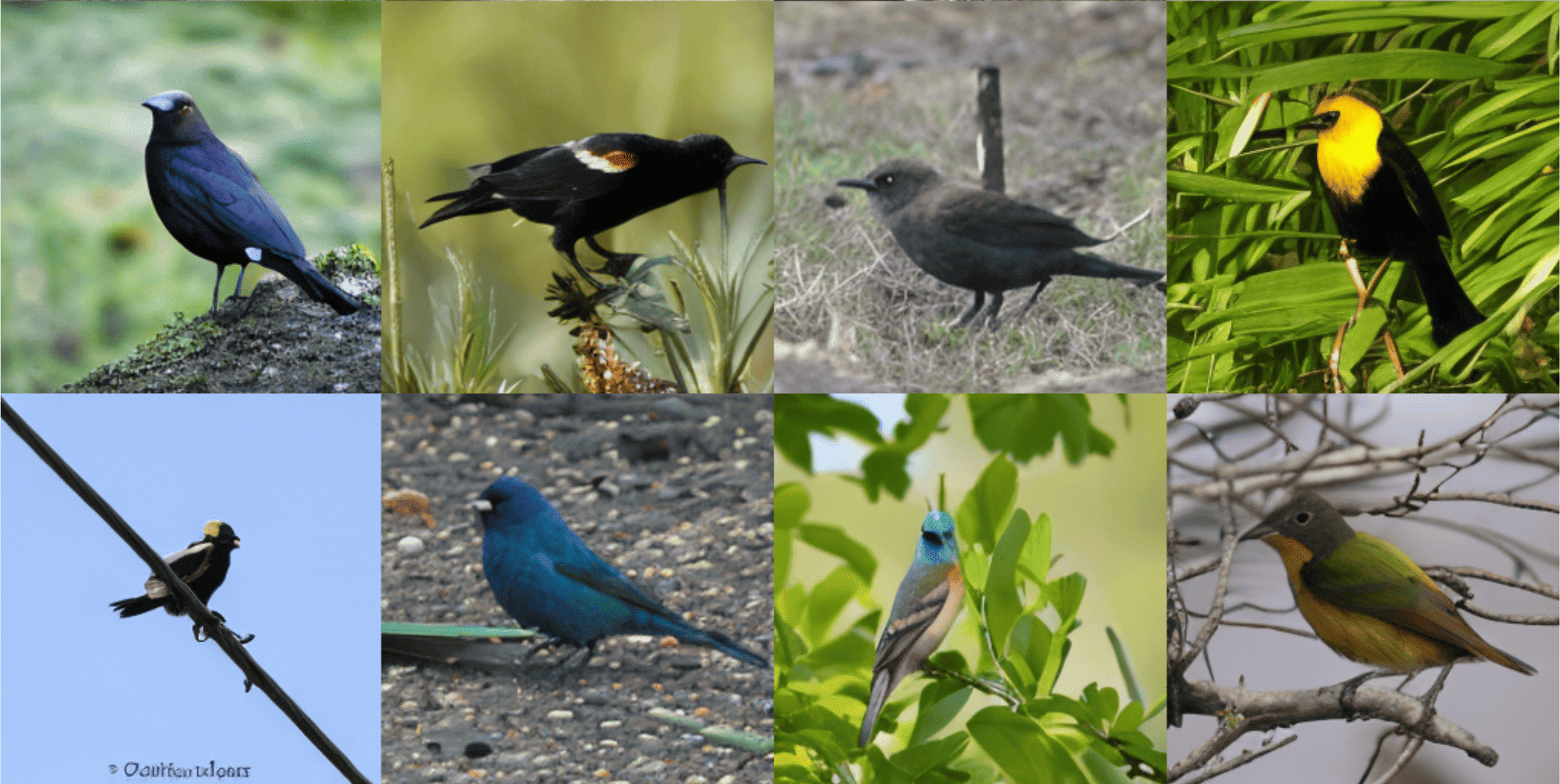}}

\newcommand{\dittwolinenfeart}{\includegraphics[width=\colwww,height=\colwww,keepaspectratio]{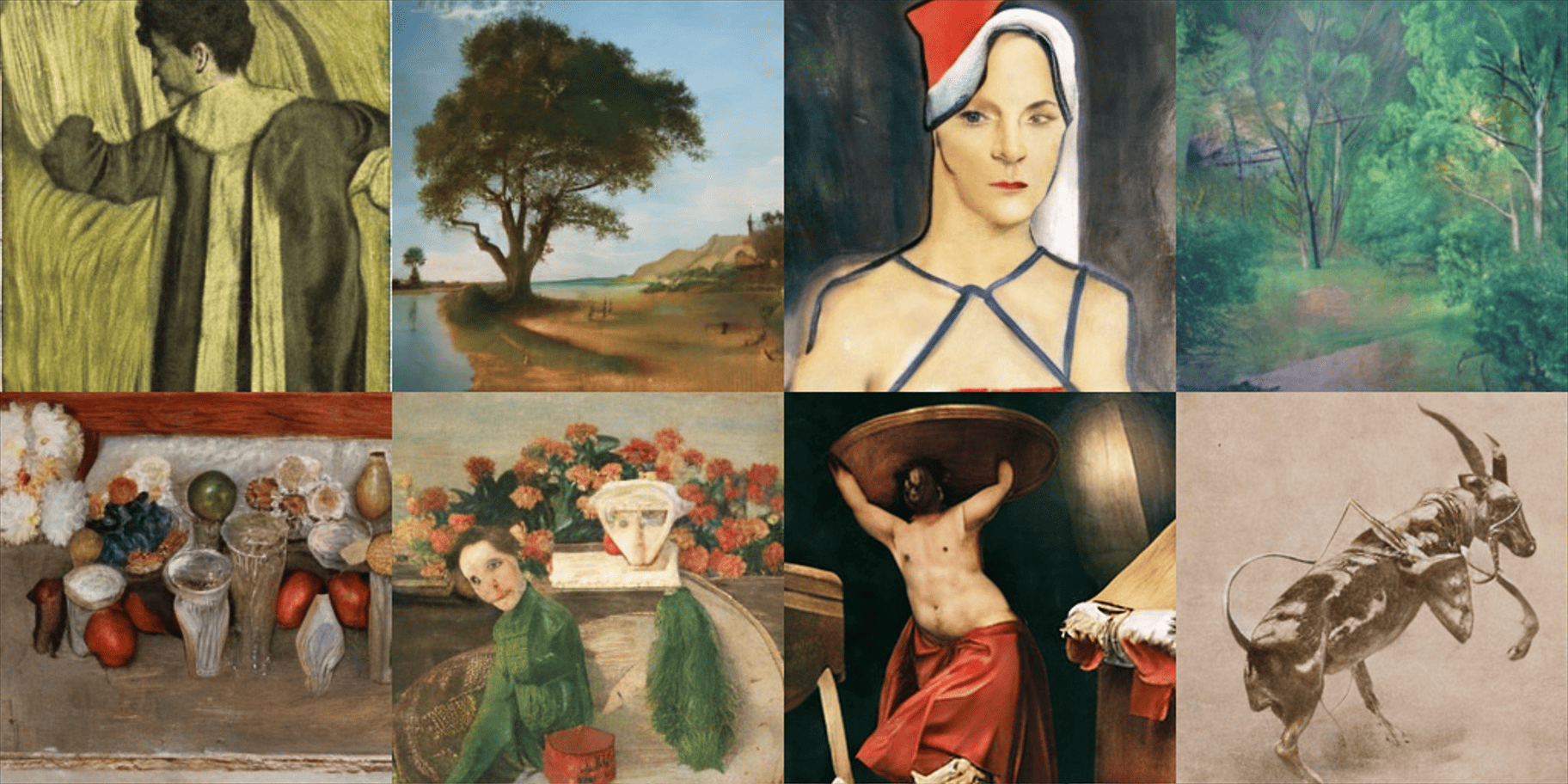}}

\newcommand{\jittwolinenfeart}{\includegraphics[width=\colwww,height=\colwww,keepaspectratio]{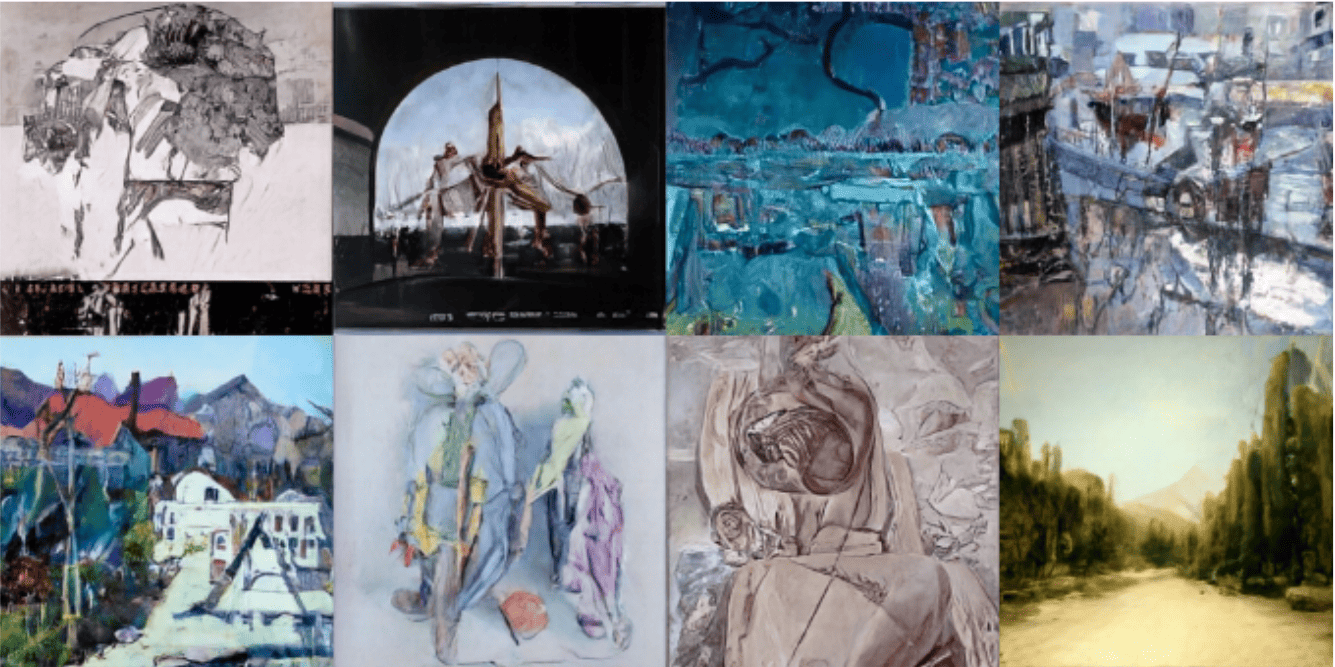}}

\newcommand{\imftwolinesfournfeart}{\includegraphics[width=\colwww,height=\colwww,keepaspectratio]{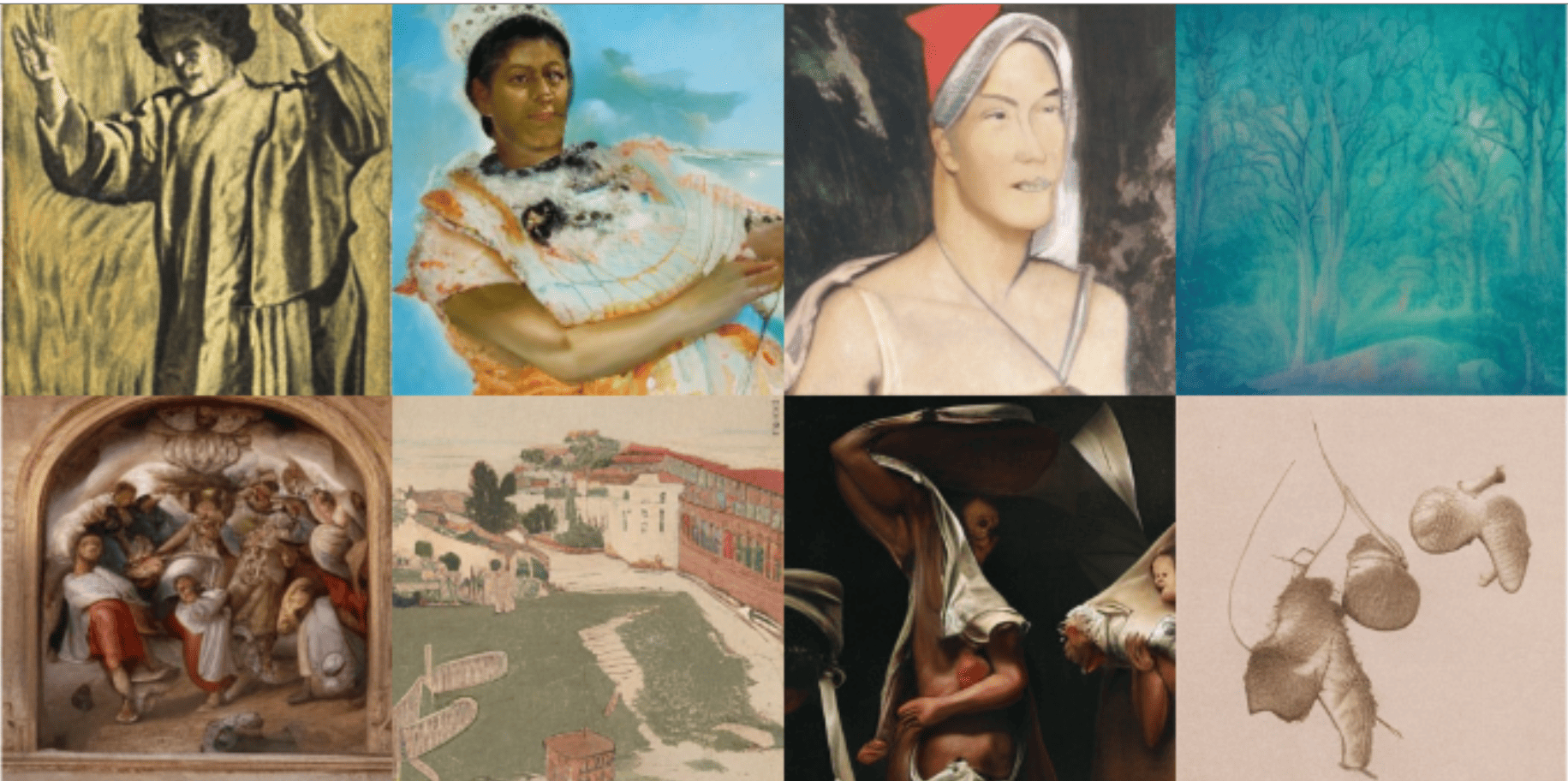}}

\newcommand{\sittwolinenfeart}{\includegraphics[width=\colwww,height=\colwww,keepaspectratio]{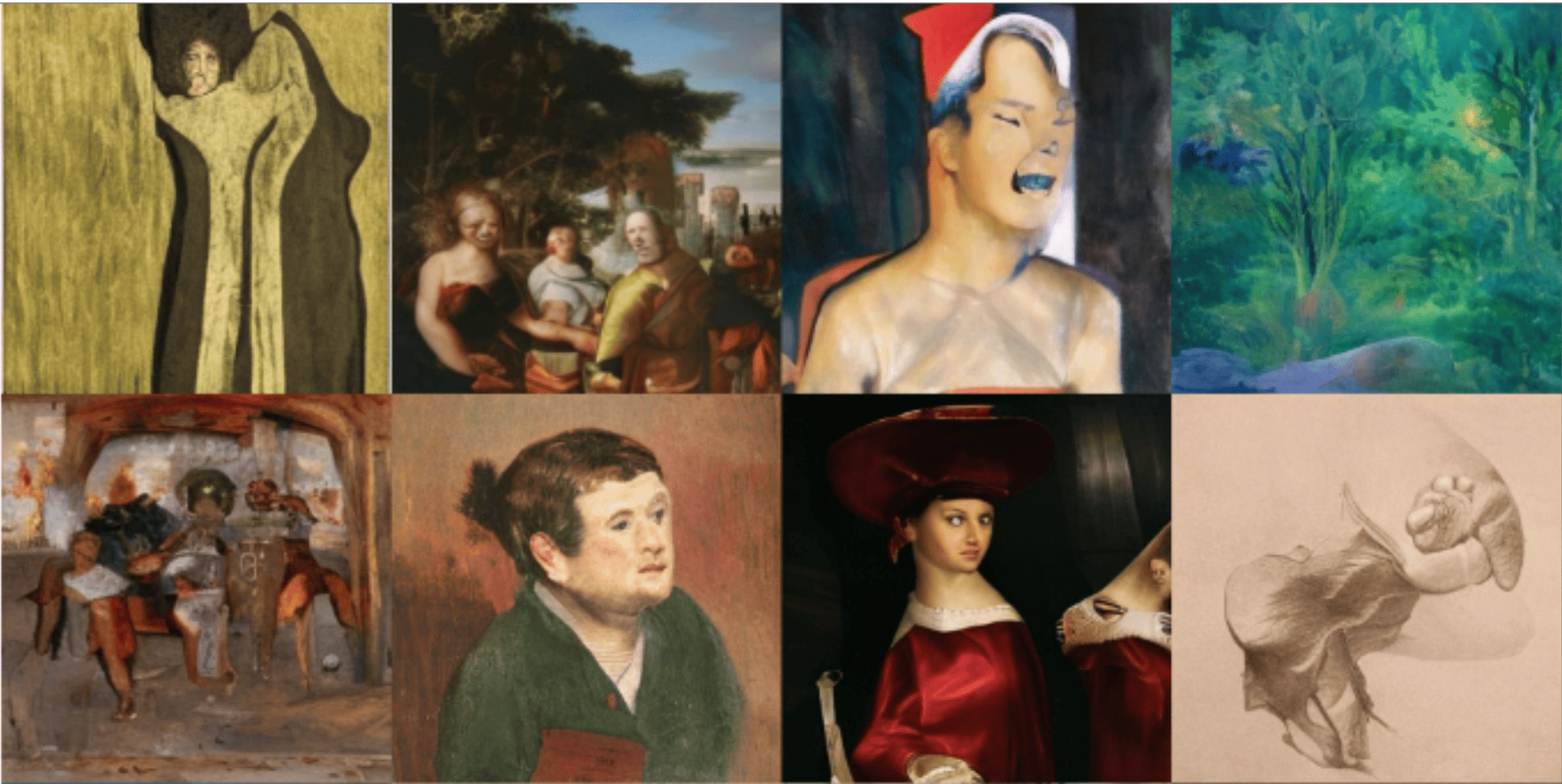}}

\newcommand{\dittwolinenfecars}{\includegraphics[width=\colwww,height=\colwww,keepaspectratio]{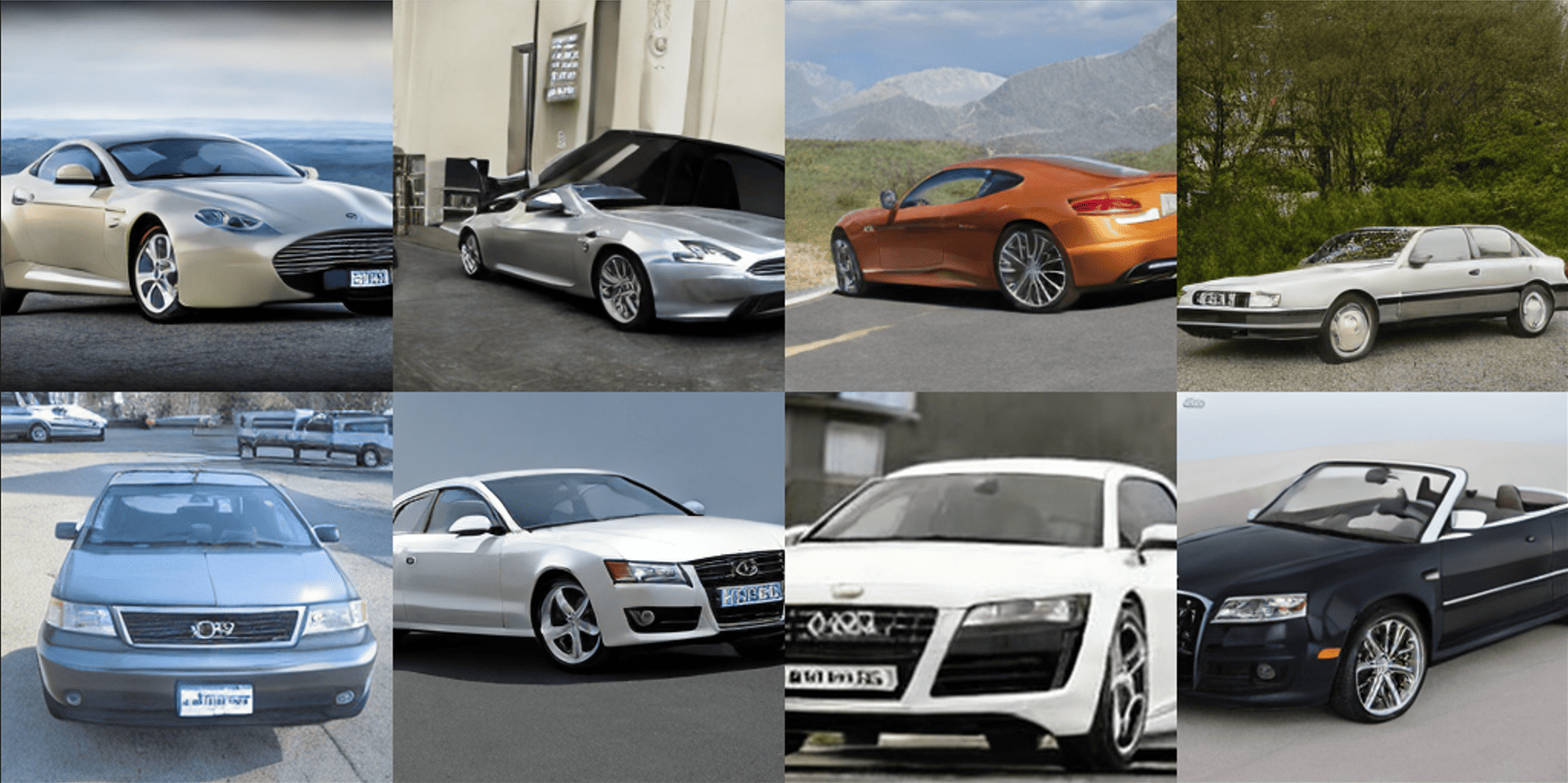}}

\newcommand{\jittwolinenfecars}{\includegraphics[width=\colwww,height=\colwww,keepaspectratio]{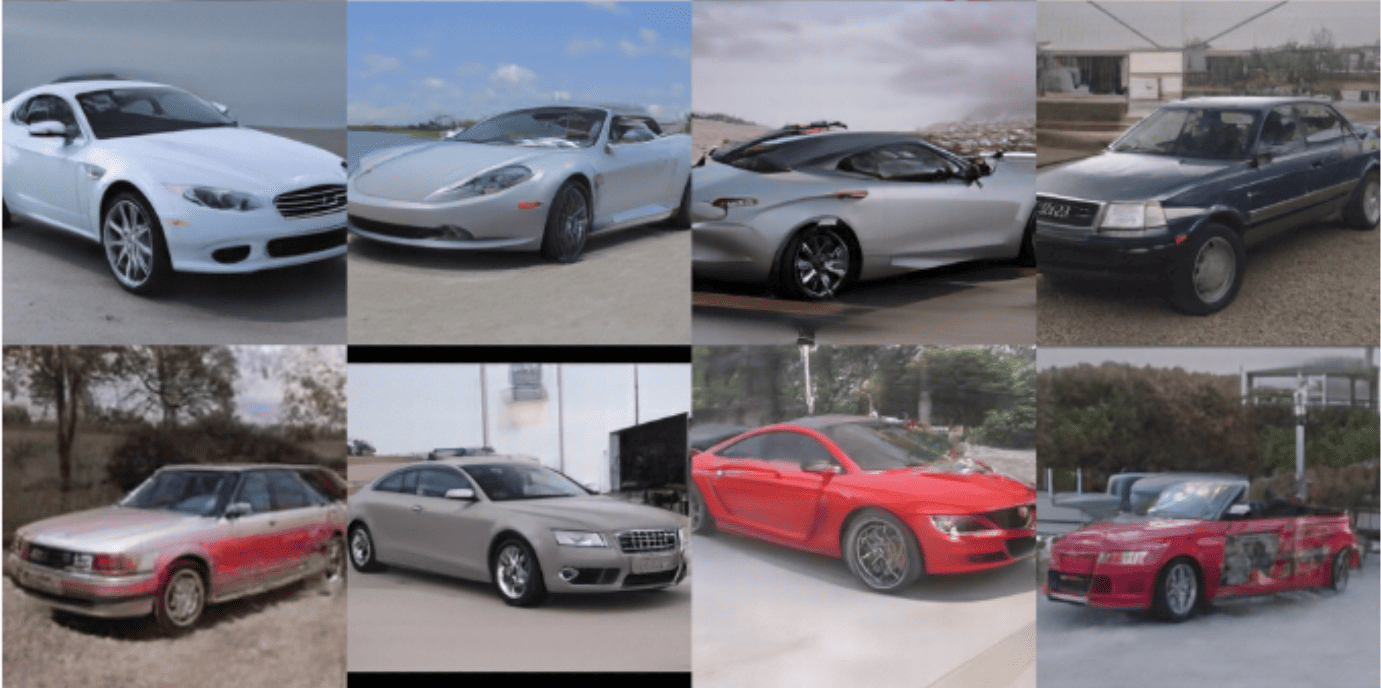}}

\newcommand{\imftwolinesfournfecars}{\includegraphics[width=\colwww,height=\colwww,keepaspectratio]{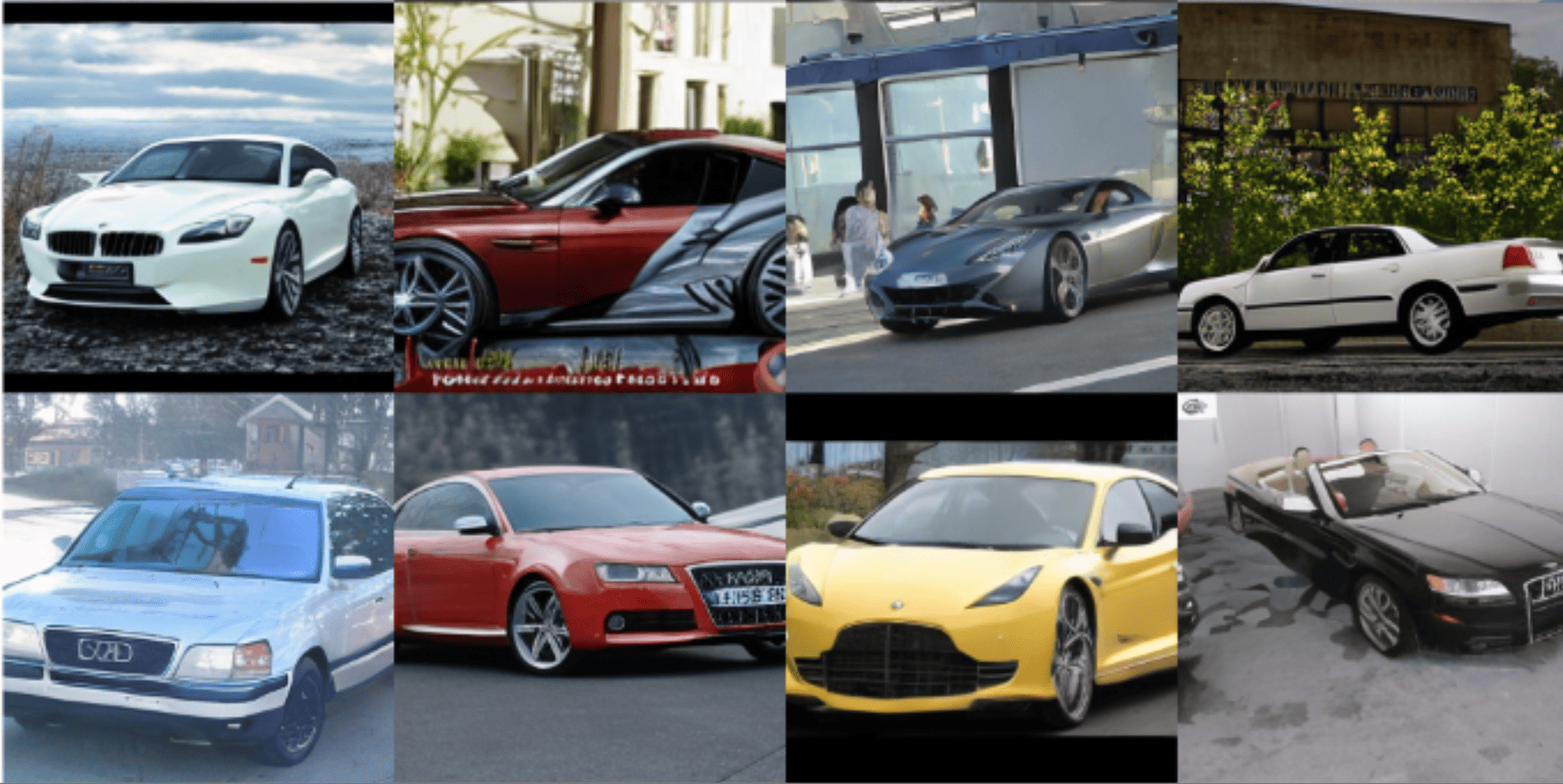}}

\newcommand{\sittwolinenfecars}{\includegraphics[width=\colwww,height=\colwww,keepaspectratio]{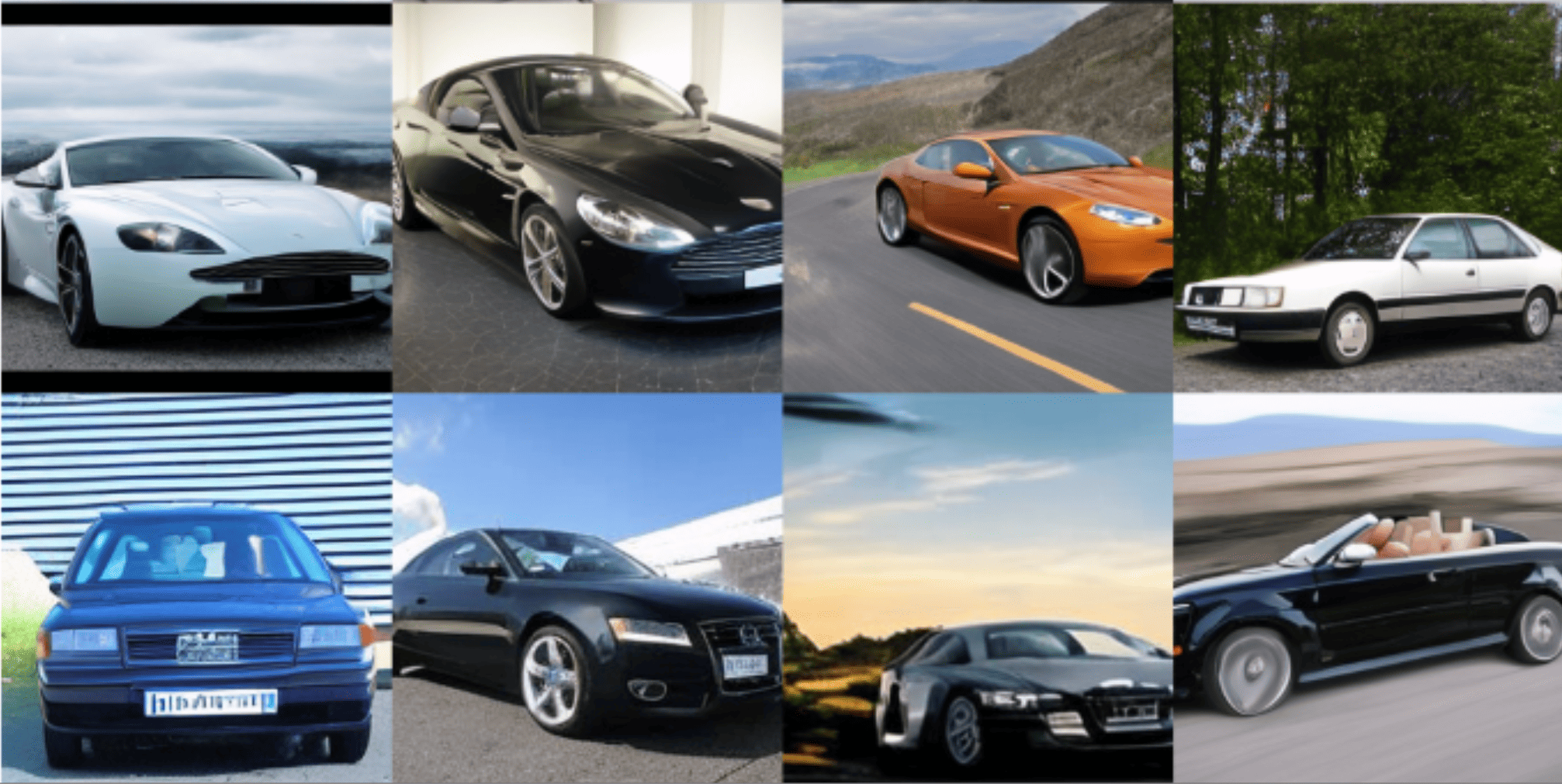}}

\newcolumntype{L}[1]{>{\raggedright\arraybackslash}m{#1}}
\newcolumntype{M}[1]{>{\centering\arraybackslash}m{#1}}
\newcolumntype{Y}{>{\centering\arraybackslash}X}
\begin{figure*}[t!]
\centering
\scriptsize
\setlength{\tabcolsep}{1.5pt}
\renewcommand{\arraystretch}{1.0}
\newcommand{\glab}[1]{\rotatebox{90}{\makebox[0.06\textwidth]{\textbf{#1}}}}
\begin{tabular}{c c c !{\hspace{6pt}} c c}
 & \multicolumn{2}{c}{\textbf{w/ \texttt{MF-T}}} & \multicolumn{2}{c}{\textbf{w/o \texttt{MF-T}}} \\
\cmidrule(lr){2-3}\cmidrule(lr){4-5}
 & $4$ NFE & $250$ NFE & $4$ NFE & $250$ NFE \\[1pt]
\glab{\texttt{CAMF}} & \includegraphics[width=0.236\textwidth]{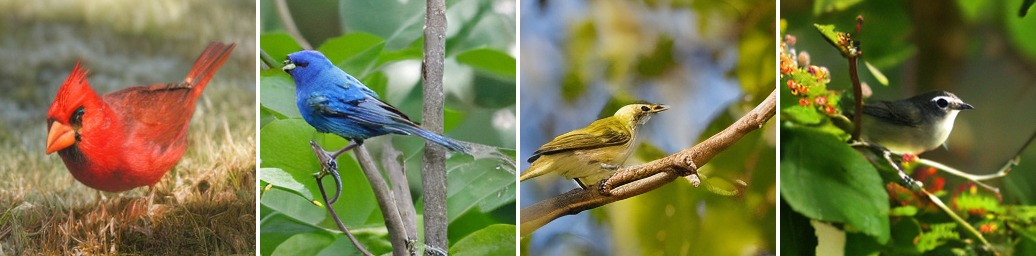} & \includegraphics[width=0.236\textwidth]{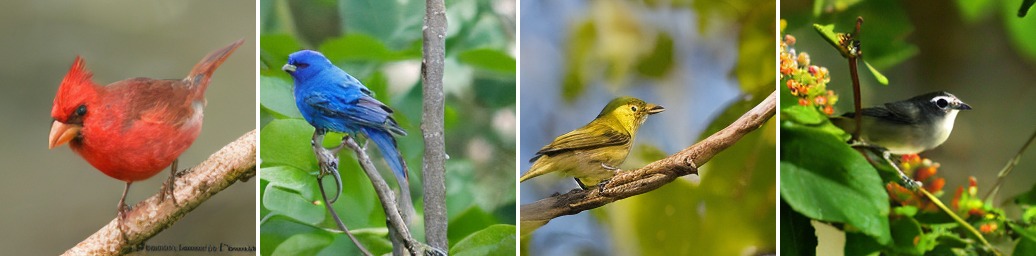} & \includegraphics[width=0.236\textwidth]{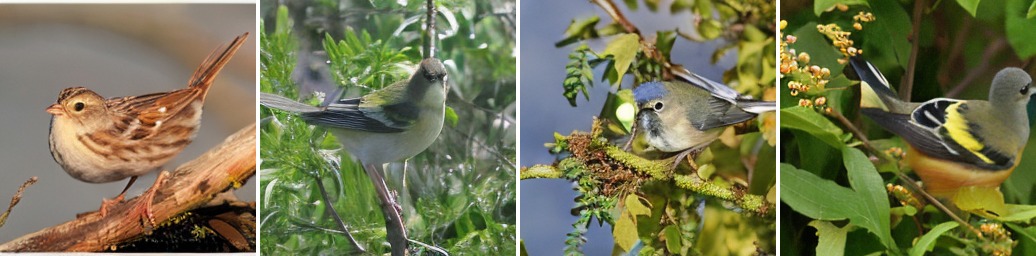} & \includegraphics[width=0.236\textwidth]{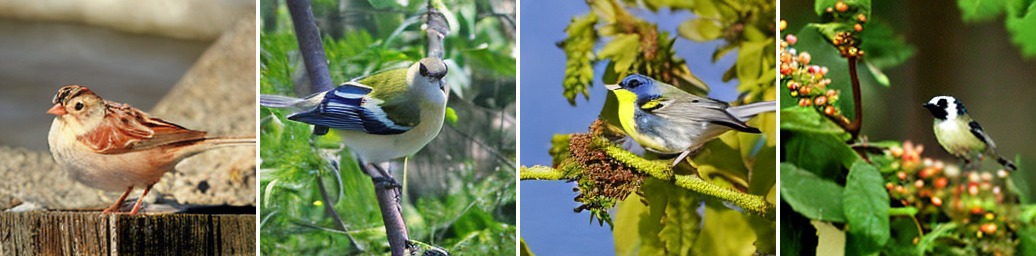} \\
\glab{AFM} & \includegraphics[width=0.236\textwidth]{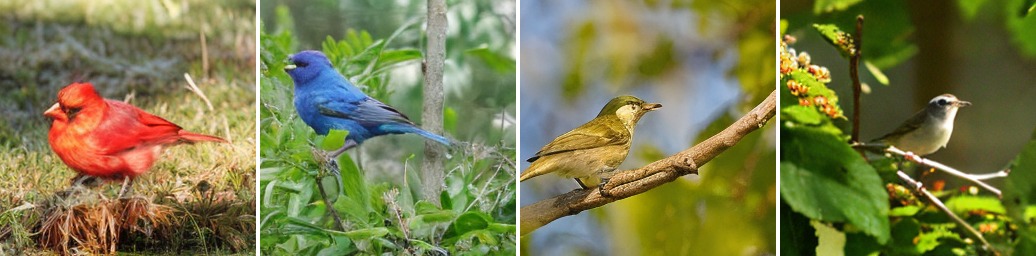} & \includegraphics[width=0.236\textwidth]{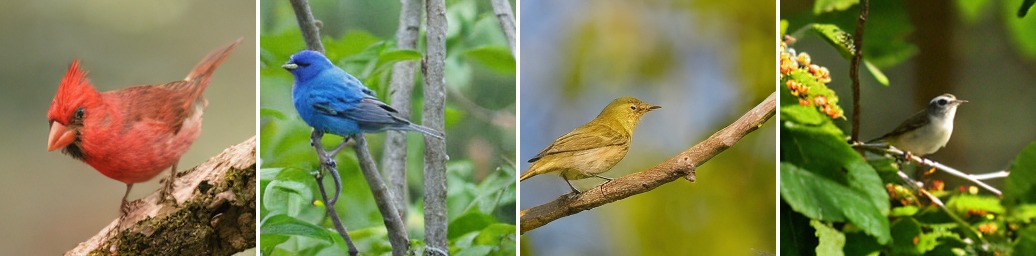} & \includegraphics[width=0.236\textwidth]{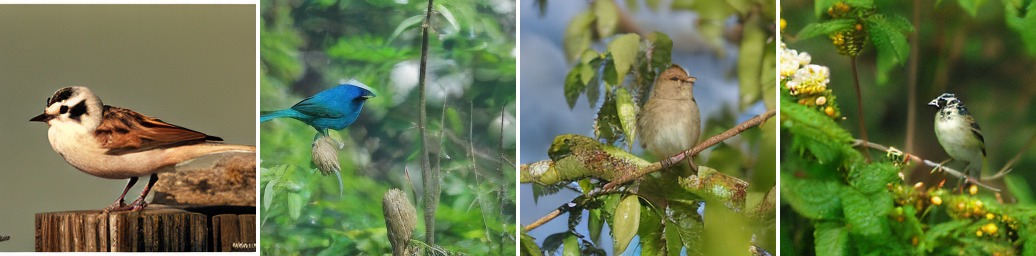} & \includegraphics[width=0.236\textwidth]{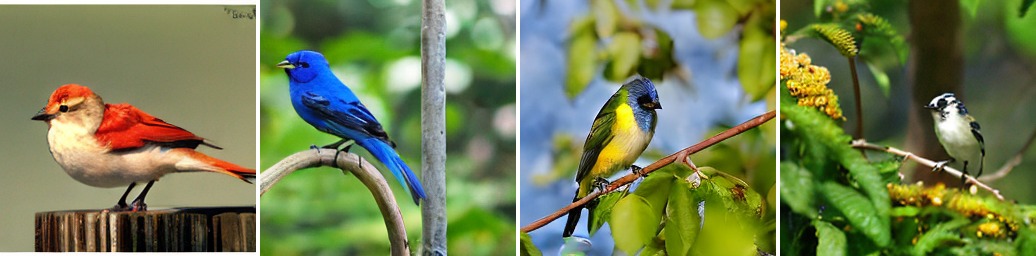} \\
\glab{CAFM} & \includegraphics[width=0.236\textwidth]{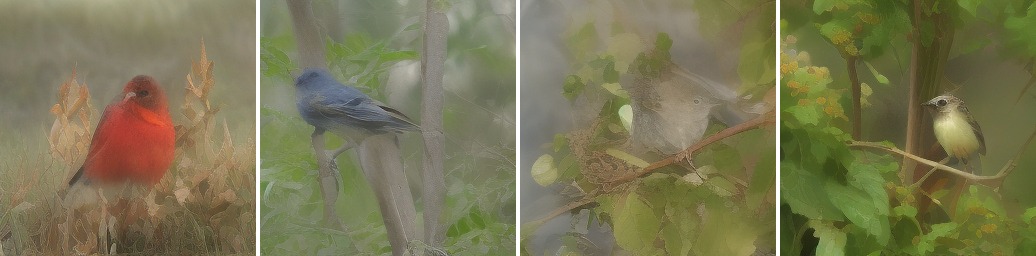} & \includegraphics[width=0.236\textwidth]{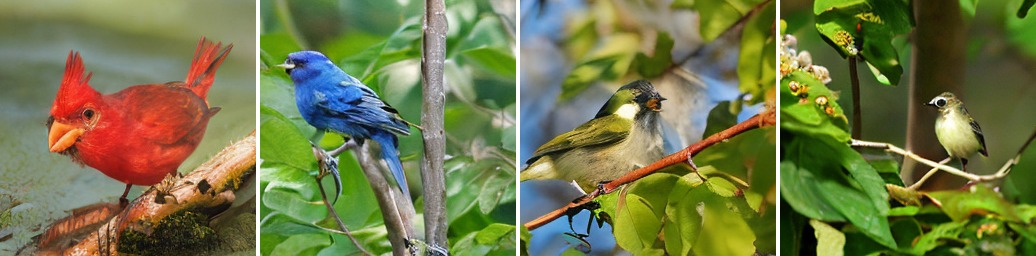} & \includegraphics[width=0.236\textwidth]{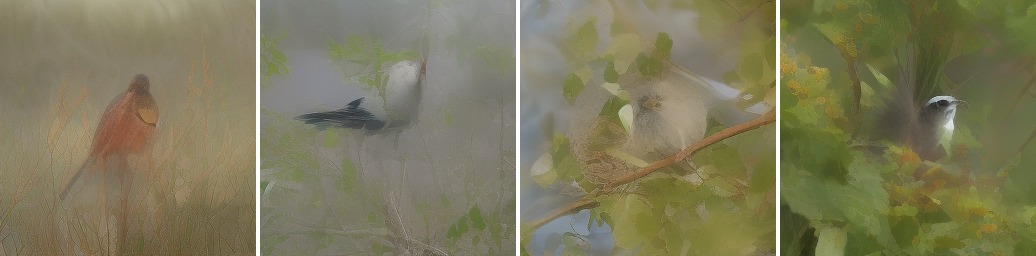} & \includegraphics[width=0.236\textwidth]{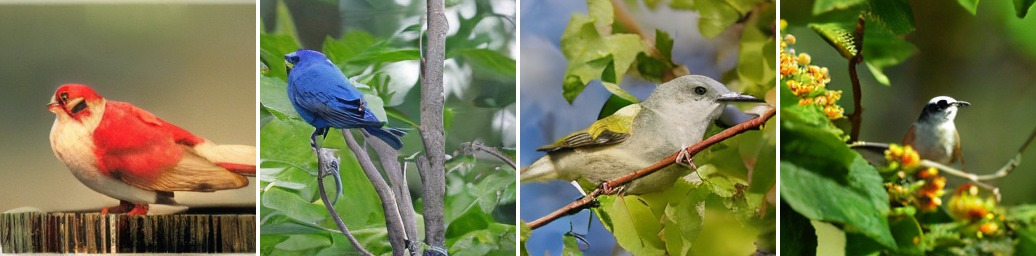} \\
\end{tabular}
\caption{\textbf{Quality comparison of generated samples}. Model: iMF-XL/2, Target domain: Birds (CUB). Unconditional samples.}
\label{fig:qualitative_main}

\vspace{6pt}
\input{src/figures/4_main_trend_row_body}

\end{figure*}

\subsection{Main Results}
\label{sec:main_resuls}

\noindent\textbf{\texttt{MF-T}~+~\texttt{CAMF} on iMF-XL/2.}
AFM, CAFM, and \texttt{CAMF} can each be applied directly to the iMF-XL/2 source or as post-training after \texttt{MF-T} (Tab.~\ref{tab:main_labeled_init_teachers}, Fig.~\ref{fig:headline_nfe}, Fig.~\ref{fig:qualitative_main}).
\texttt{MF-T} alone beats every adversarial criterion applied directly (FID$: 13.13$, FDD$: 359.38$ at $4$ NFE), remaining critical for domain adaptation. Initialized from \texttt{MF-T}, every criterion improves, and \texttt{CAMF} improves the most: its FID and FDD roughly halve compared with applying it directly to the source (FID$: 7.88$, FDD$: 190.08$ at $4$ NFE). \texttt{MF-T}~+~\texttt{CAMF} outperforms \texttt{MF-T}~+~AFM and \texttt{MF-T}~+~CAFM at every NFE from $1$ to $250$ (Fig.~\ref{fig:headline_nfe}), as the adversarial stage builds on the adapted transport map that \texttt{MF-T} supplies. In contrast, CAFM degrades sharply below $16$ NFE (Fig.~\ref{fig:qualitative_main}), which we attribute to the instantaneous velocity it scores being far from the jumps a few-step sampler takes.

\noindent\textbf{\texttt{MF-T}~+~\texttt{CAMF} on any source parameterization.}
We apply the pipeline to other parameterizations: SiT-XL/2, DiT-XL/2, and JiT-H/16 (Tab.~\ref{tab:main_labeled_init_teachers}). Across all four parameterizations, \texttt{MF-T} alone outperforms few-step FT, and \texttt{CAMF} lowers its few-step FID further by around $30\%$ on average.
On SiT, \texttt{MF-T}~+~\texttt{CAMF} outperforms AFM and CAFM post-training at $1$ and $250$ NFE and matches AFM at $4$ NFE (FDD $241.19$ vs.\ $240.89$).
Across all parameterizations, \texttt{MF-T}~+~\texttt{CAMF} yields high-quality few-step target-domain generation (Appx.~\ref{app:additional_qualitative}~Fig.~\ref{fig:qualitative-appendix}) and improves FID over both few-step FT and \texttt{MF-T} alone on every target dataset. On SiT and DiT, the few-step model exceeds its own fine-tuned many-step teacher with $125\times$ and $62.5\times$ fewer NFEs, respectively, while it narrows the gap on JiT. Applied to \texttt{MF-T}, \texttt{CAMF} raises precision, density, and coverage while recall stays nearly unchanged, so the sharper samples do not come from dropped modes (Appx.~\ref{app:fidelity_diversity}, Tab.~\ref{tab:precision_recall}).
At $1$ NFE, \texttt{CAMF} is the only criterion that improves both FID and FDD of \texttt{MF-T}, lowering FID by about $20\%$ on iMF and SiT, while CAFM collapses (Appx.~\ref{app:one-step_results}~Tab.~\ref{tab:onestep}). \texttt{MF-T}~+~\texttt{CAMF} also attains the best or tied-best IS at low NFE (Appx.~\ref{app:inception_score}~Tab.~\ref{tab:main_is}).

\subsection{Ablation Studies}
\label{sec:ablation}

\noindent \textbf{Target data availability and NFE.} Fig.~\ref{fig:main_dataeff} demonstrates the effectiveness of \texttt{CAMF} under limited target-domain data. \texttt{MF-T}~+~\texttt{CAMF} trained with only $10\%$ of the data at $2$ NFE outperforms, on average, \texttt{MF-T} trained with up to $100\%$ of the data at $4$ NFE in both FID and FDD. This highlights \texttt{CAMF}'s ability to improve generation quality in data-scarce settings (per-dataset numbers in Appx.~\ref{app:data_efficiency}~Tab.~\ref{tab:mft_data_efficiency}).

\noindent \textbf{2-stage vs. 1-stage \texttt{MF-T}.} Fig.~\ref{fig:main_one_vs_two} compares 1-stage \texttt{MF-T} training with 2-stage training across NFE budgets. As NFE decreases, 1-stage training becomes increasingly advantageous despite using half the training steps, outperforming the 2-stage approach at the lowest NFE for SiT and DiT in both FID and FDD (per-dataset numbers in Appx.~\ref{app:two_stage}~Tab.~\ref{tab:one_vs_two_nfe}).

\noindent \textbf{\texttt{CAMF} component ablation.} Tab.~\ref{tab:camf_ablation} examines the impact of the discriminator score components and training interval lengths. The complete \texttt{CAMF} formulation achieves the best average FID and FDD, while removing either component degrades performance. Restricting training to shorter intervals reduces $4$ NFE quality, highlighting the importance of learning over a broad range of transport intervals.

\begin{table}[h]
\centering
\scriptsize
\setlength{\tabcolsep}{2.0pt}
\renewcommand{\arraystretch}{1.36}  %
\caption{\textbf{\texttt{CAMF} component ablation}, iMF-XL/2 at $4$ NFE. With $\Delta = t-r$, training intervals are restricted to $[0.001,\Delta_{\max}]$. \textbf{Bold}: best average.}
\vspace{-0.2cm}

\label{tab:camf_ablation}

\resizebox{\linewidth}{!}{%
\begin{tabular}{
>{\scriptsize}l
*{3}{>{\scriptsize}c}
*{2}{>{\scriptsize}c}
>{\scriptsize}c
!{\color{black!20}\vrule}
*{2}{>{\scriptsize}c}
>{\scriptsize}c}
\toprule
& & & & \multicolumn{3}{c}{\textbf{FID} $\downarrow$} & \multicolumn{3}{c}{\textbf{FDD} $\downarrow$} \\
\cmidrule(lr){5-7}\cmidrule(lr){8-10}
\textbf{Ablation}
& $\bm{D_\psi(z_t,t)}$ & $\bm{1/\Delta}$ & $\bm{\Delta_{\max}}$ & CUB & Cars & \textbf{Avg.} & CUB & Cars & \textbf{Avg.} \\
\midrule
\rowcolor{gray!7}\texttt{MF-T}~+~\texttt{CAMF} & \cmark & \cmark & $1$ & ${\color{gray!90}2.88}$ & ${\color{gray!90}3.07}$ & $\mathbf{2.98}$ & ${\color{gray!90}84.24}$ & ${\color{gray!90}102.37}$ & $\mathbf{93.31}$ \\
\hline
\multirow{3}{*}{\renewcommand{\arraystretch}{0.9}\begin{tabular}{@{}l@{}}Score\\components\end{tabular}}
& \xmark & \cmark & $1$ & ${\color{gray!90}2.81}$ & ${\color{gray!90}3.86}$ & $3.34$ & ${\color{gray!90}80.70}$ & ${\color{gray!90}143.60}$ & $112.15$ \\
& \cmark & \xmark & $1$ & ${\color{gray!90}2.73}$ & ${\color{gray!90}3.76}$ & $3.25$ & ${\color{gray!90}80.14}$ & ${\color{gray!90}144.69}$ & $112.42$ \\
& \xmark & \xmark & $1$ & ${\color{gray!90}3.06}$ & ${\color{gray!90}4.02}$ & $3.54$ & ${\color{gray!90}86.36}$ & ${\color{gray!90}145.57}$ & $115.97$ \\
\hline
\multirow{2}{*}{\renewcommand{\arraystretch}{0.9}\begin{tabular}{@{}l@{}}Training\\intervals\end{tabular}}
& \cmark & \cmark & $0.25$ & ${\color{gray!90}2.68}$ & ${\color{gray!90}3.51}$ & $3.10$ & ${\color{gray!90}81.26}$ & ${\color{gray!90}124.31}$ & $102.79$ \\
& \cmark & \cmark & $0.001$ & ${\color{gray!90}5.22}$ & ${\color{gray!90}7.31}$ & $6.27$ & ${\color{gray!90}159.64}$ & ${\color{gray!90}258.69}$ & $209.17$ \\
\bottomrule
\end{tabular}%
}%
\vspace{-10pt}  %
\end{table}

\section{Conclusion}
\label{sec:conclusion}

We presented MeanFlow-Transfer (\texttt{MF-T}), which turns pretrained diffusion and flow models with $x$-, $\epsilon$-, $v$-, or $u$-prediction into few-step generators on a new domain under limited data. We introduced Continuous Adversarial MeanFlow (\texttt{CAMF}), which refines the model through an adversarial objective that scores finite-interval transport segments. Across four sources and five target domains, \texttt{MF-T}~+~\texttt{CAMF} improves few-step FID and FDD over \texttt{MF-T} alone, and over prior adversarial methods on iMF, and approaches or exceeds the fine-tuned many-step teacher at up to $125\times$ fewer NFEs. 

\paragraph{Future work.}
We reach strong 4-step quality, but our 1-step quality lags. We study class-conditional generation at $256\times256$ and do not explore higher resolutions or text-to-image and video transfer. Multi-teacher distillation from multiple source parameterizations is left for future work.

\section*{Acknowledgments}
This research was supported by Electronic Arts, the Natural Sciences and Engineering Research Council of Canada, and the Digital Research Alliance of Canada.

\bibliography{aaai2027}

\clearpage
\appendix

\begin{center}
\begin{minipage}{\linewidth}
\small
\hrule
\vspace{0.6em}

\begin{center}
    \textbf{Table of Contents}
\end{center}

\hrule
\vspace{0.7em}

\SuppTOCLine{
    \textbf{\ref{sec:x_additional_method} \quad
    Additional Methodological Details}
}{\pageref{sec:x_additional_method}}

\SuppTOCLine{
    \quad \ref{app:a_diff2flow} \quad
    Diffusion to Flow-Matching
}{\pageref{app:a_diff2flow}}

\SuppTOCLine{
    \quad \ref{app:x_proof_CAMF} \quad
    Proof of Proposition~\ref{prop:camf_consistency}
}{\pageref{app:x_proof_CAMF}}

\SuppTOCLine{
    \quad \ref{app:x_conditioning} \quad
    Proof of Proposition~\ref{prop:camf_conditioning}
}{\pageref{app:x_conditioning}}

\SuppTOCLine{
    \quad \ref{app:conditional_velocity} \quad
    Conditional versus Marginal Velocity
}{\pageref{app:conditional_velocity}}

\SuppTOCLine{
    \quad \ref{app:centering_ablation} \quad
    Extended Centering Penalty
}{\pageref{app:centering_ablation}}

\SuppTOCLine{
    \quad \ref{app:ot-ablation} \quad
    Optimal-Transport Regularization
}{\pageref{app:ot-ablation}}

\vspace{0.4em}

\SuppTOCLine{
    \textbf{\ref{sec:additional_experimental} \quad
    Additional Experimental Details}
}{\pageref{sec:additional_experimental}}

\SuppTOCLine{
    \quad \ref{app:Hyperparameters} \quad
    Hyperparameters
}{\pageref{app:Hyperparameters}}

\SuppTOCLine{
    \quad \ref{app:training_algorithms} \quad
    Training Algorithms
}{\pageref{app:training_algorithms}}

\SuppTOCLine{
    \quad \ref{app:compute_cost} \quad
    Computational Cost
}{\pageref{app:compute_cost}}

\vspace{0.4em}

\SuppTOCLine{
    \textbf{\ref{sec:additional_results} \quad
    Additional Results}
}{\pageref{sec:additional_results}}

\SuppTOCLine{
    \quad \ref{app:one-step_results} \quad
    1-step Generation
}{\pageref{app:one-step_results}}

\SuppTOCLine{
    \quad \ref{app:inception_score} \quad
    Inception Score
}{\pageref{app:inception_score}}

\SuppTOCLine{
    \quad \ref{app:fidelity_diversity} \quad
    Fidelity and Diversity of \texttt{CAMF}
}{\pageref{app:fidelity_diversity}}

\SuppTOCLine{
    \quad \ref{app:two_stage} \quad
    2-stage vs.\ 1-stage \texttt{MF-T}
}{\pageref{app:two_stage}}

\SuppTOCLine{
    \quad \ref{app:checkpoint_selection} \quad
    Checkpoint Selection
}{\pageref{app:checkpoint_selection}}

\SuppTOCLine{
    \quad \ref{app:data_efficiency} \quad
    Data Efficiency vs. NFE
}{\pageref{app:data_efficiency}}

\SuppTOCLine{
    \quad \ref{app:additional_qualitative} \quad
    Additional Qualitative Results
}{\pageref{app:additional_qualitative}}

\vspace{0.3em}
\hrule

\end{minipage}
\end{center}

\section{Additional Methodological Details}
\label{sec:x_additional_method}

\subsection{Diffusion to Flow-Matching}
\label{app:a_diff2flow}

Source models are pretrained under a variety of noise schedules and prediction targets. For example, a DiT~\cite{peebles2023scalable} uses a DDPM-based noise schedule and noise prediction target $\epsilon$. 
Using the common diffusion interpolant:
\begin{equation}
    z_\tau = \alpha_\tau x + \sigma_\tau \epsilon,
\qquad \epsilon \sim \mathcal{N}(0, I),
\end{equation}
where $\alpha$ and $\sigma$ denote noise schedules. We consider three ways of sampling from this model: the native discrete reverse-diffusion sampler (DiT), the same $\epsilon$-model reinterpreted as a velocity on the diffusion path (DiT + velocity map), and the same $\epsilon$-model transported onto a linear flow-matching path (DiT + velocity map + schedule alignment, as in Diff2Flow~\cite{schusterbauer2025diff2flow}).

\noindent\textbf{DiT}:  This is the original diffusion model without any change to its output parameterization or noise schedule. The model predicts noise, and its sampling follows a discrete reverse diffusion chain. We have:
\begin{equation}
    \hat{\epsilon}_\phi = \epsilon_\phi(z_\tau,\tau), \qquad
\hat{x} = \frac{z_\tau - \sigma_\tau \hat{\epsilon}_\phi}{\alpha_\tau}.
\end{equation}

\noindent\textbf{DiT + velocity map}: This is still an $\epsilon$ model with its output parameterization mapped into diffusion-path velocity $v$. We have:
\begin{equation}
    v_\tau
=
\frac{d z_\tau}{d\tau}
=
\dot{\alpha}_\tau x + \dot{\sigma}_\tau \epsilon.
\end{equation}
Using the same $\epsilon$-based model:
\begin{equation}
    \hat{x} = \frac{z_\tau - \sigma_\tau \hat{\epsilon}_\phi}{\alpha_\tau},
\qquad
\hat{v}_\tau
\approx
\dot{\alpha}_\tau \hat{x} + \dot{\sigma}_\tau \hat{\epsilon}_\phi.
\end{equation}
Using finite differences, the new noise schedule can be written as:
\begin{equation}
    \dot{\alpha}_\tau \approx \frac{\alpha_{\tau'}-\alpha_\tau}{\tau'-\tau},
\qquad
\dot{\sigma}_\tau \approx \frac{\sigma_{\tau'}-\sigma_\tau}{\tau'-\tau},
 \label{eq:diffusion_schedule_derivatives}
\end{equation}
where $\tau'$ is the next time-step after $\tau$.
Then, sampling can use an Euler update as:
\begin{equation}
    z_{\tau'} = z_\tau + (\tau' - \tau)\hat{v}_\tau.
    \label{eq:dm_euler_update}
\end{equation}

\noindent\textbf{DiT + velocity map + schedule alignment}: This follows the same $\epsilon$ model, but sampled on a linear flow-matching path after Diff2Flow-style time and state alignment~\cite{schusterbauer2025diff2flow}. We use the linear flow-matching path:
\begin{equation}
    z_t = (1-t) x + t\epsilon,
\qquad
v_t = \frac{d z_t}{dt} = \epsilon - x.
\end{equation}
Using Diff2Flow-style alignment, we align the diffusion state and time (noise schedule alignment):
\begin{equation}
    t = \frac{\sigma_\tau}{\alpha_\tau + \sigma_\tau},
\qquad
z_\tau = (\alpha_\tau + \sigma_\tau)\,z_t.
\end{equation}
Then we query the $\epsilon$ model on the aligned noise schedule and map the output parameterization to $v$:
\begin{equation}
    \hat{x} = \frac{z_\tau - \sigma_\tau \hat{\epsilon}_\phi}{\alpha_\tau},
\qquad
\hat{v}_t = \hat{\epsilon}_\phi - \hat{x}.
\end{equation}
Sampling then can use Euler updates:
\begin{equation}
    z_{t'} = z_t + (t'-t)\hat{v}_t.
\end{equation}
Tab.~\ref{tab:diff2flow_full} reports the per-dataset target-domain results of the three variants after
fine-tuning, and Fig.~\ref{fig:3_analysis}(b) shows their averages.

\noindent\textbf{Time-parameterization convention.}
For readability we omit time reparameterization from the figures and
derivations, but it is applied throughout. Source models do not share a time convention: DiT indexes discrete steps from $999$ down to $0$~\cite{peebles2023scalable}, whereas SiT and the flow-matching families use continuous $t\in[0,1]$ increasing toward data~\cite{ma2024sit}. When initializing from a source model we compose the necessary affine map (a flip and a rescaling in the DiT case) so that the target generator receives times on the scale the source network was trained to expect. Omitting this map leaves the initialized model evaluating its own weights far outside their training range.

\subsection{Proof of Proposition~\ref{prop:camf_consistency}}
  \label{app:x_proof_CAMF}

 \texttt{CAMF} scores a generated sample
  through the finite-interval quantity
  $\mathcal{A}^{\text{fake}}(z_t,\widehat z_r;t,r)$ and a real sample through
  $\mathcal{A}^{\text{real}}(z_t,z_r;t,r)$, where both are built from the same potential
  $D_\psi$ and the same upper-time anchor $(z_t,t)$. The proposition asserts that
  \emph{both} scores are first-order-consistent approximations of the
  instantaneous velocity-space score of CAFM (the generated one at the
  student velocity, the real one at the ground-truth velocity), and that each
  recovers its instantaneous counterpart exactly as $r\to t$.
  
 \begin{proof}
  Let $\Delta = t-r > 0$, so $r = t-\Delta$.

  \noindent\textbf{Generated branch.}
  The predicted endpoint is $\widehat z_r = z_t-\Delta\,u_\theta(z_t,r,t)$. A
  first-order Taylor expansion of $D_\psi$ about $(z_t,t)$ gives:
  \begin{equation}
  \begin{aligned}
  D_\psi(\widehat z_r,r)
  &= D_\psi\bigl(z_t-\Delta\,u_\theta,\;t-\Delta\bigr) \\
  &= D_\psi(z_t,t)
  -\Delta\,\partial_z D_\psi\,u_\theta
  -\Delta\,\partial_t D_\psi
  +\mathcal{O}(\Delta^2).
  \end{aligned}
  \end{equation}
  Substituting into the finite-interval score and dividing by $\Delta$,
  \begin{equation}
  \begin{aligned}
  \mathcal{A}^{\text{fake}}(z_t,\widehat z_r;t,r)
  &= \frac{D_\psi(z_t,t)-D_\psi(\widehat z_r,r)}{\Delta} \\
  &= \partial_z D_\psi\,u_\theta+\partial_t D_\psi+\mathcal{O}(\Delta) \\
  &= \texttt{JVP}_{z_t,t}\bigl(D_\psi;u_\theta,1\bigr)+\mathcal{O}(\Delta).
  \end{aligned}
  \label{eq:proof_fake_branch}
  \end{equation}

  \noindent\textbf{Real branch.}
  The real endpoint is not predicted but read off the interpolant. Under the
  linear path $z_t = t\,\epsilon+(1-t)\,x$ the trajectory is straight, so its
  velocity $v^{*}=\epsilon-x$ is constant for any $t$ and the average velocity over
  any interval $[r,t]$ coincides with it. Consequently:
  \begin{equation}
  z_r = z_t-\Delta\,v^{*},
  \label{eq:proof_real_endpoint}
  \end{equation}
  \emph{exactly}, with no discretization error. The real branch differs from the
  generated branch only in which velocity displaces the anchor. Expanding
  $D_\psi$ about the same point $(z_t,t)$,
  \begin{equation}
  \begin{aligned}
  \mathcal{A}^{\text{real}}(z_t,z_r;t,r)
  &= \frac{D_\psi(z_t,t)-D_\psi(z_r,r)}{\Delta} \\
  &= \partial_z D_\psi\,v^{*}+\partial_t D_\psi+\mathcal{O}(\Delta) \\
  &= \texttt{JVP}_{z_t,t}\bigl(D_\psi;v^{*},1\bigr)+\mathcal{O}(\Delta).
  \end{aligned}
  \label{eq:proof_real_branch}
  \end{equation}

  \noindent\textbf{Limit.}
  As $r\to t$ the $\mathcal{O}(\Delta)$ remainders vanish and
  $u_\theta(z_t,r,t)\to u_\theta(z_t,t,t)=v_\theta$, so
  Eq.~\eqref{eq:proof_fake_branch} converges to
  $\texttt{JVP}_{z_t,t}(D_\psi;v_\theta,1)$ and
  Eq.~\eqref{eq:proof_real_branch} to
  $\texttt{JVP}_{z_t,t}(D_\psi;v^{*},1)$. These are exactly the instantaneous
  velocity-space scores CAFM assigns to the generated and real velocity fields at
  $(z_t,t)$. For the centering penalty of $\mathcal{L}^{D}_{\texttt{CAMF}}$,
  continuity of $D_\psi$ and $z_r,\widehat z_r\to z_t$ give
  $D_\psi(z_r,r)^2,\,D_\psi(\widehat z_r,r)^2\to D_\psi(z_t,t)^2$, so the penalty tends to
  $3\lambda_{\mathrm{cp}}\,\mathbb{E}[D_\psi(z_t,t)^2]$, CAFM's penalty up to a factor of $3$.
  This establishes the claim.
  \end{proof}

\subsection{Proof of Proposition~\ref{prop:camf_conditioning}}
\label{app:x_conditioning}

We write
$\mathcal{A}^{\text{fake}}=\mathcal{A}^{\text{fake}}(z_t,\widehat z_r;t,r)$ and
$\mathcal{A}^{\text{real}}=\mathcal{A}^{\text{real}}(z_t,z_r;t,r)$.

\begin{proof}
\noindent\textbf{Generator gradient.}
In the generator loss of Eq.~\eqref{eq:cafm_gen}, only the fake score depends on
$\theta$, and it does so only through the predicted endpoint $\widehat z_r$. By
Eq.~\eqref{eq:MF-T_predicted_endpoint}, moving the generator output moves this
endpoint by
\begin{equation}
\partial_\theta\widehat z_r=-(t-r)\,\partial_\theta u_\theta .
\end{equation}
Differentiating Eq.~\eqref{eq:finite_interval_discriminator} with respect to
$\theta$, the anchor $D_\psi(z_t,t)$ drops out, and the factor $t-r$ above
cancels the $1/(t-r)$ of the score:
\begin{align}
\partial_\theta\mathcal{A}^{\text{fake}}
&=-\frac{1}{t-r}\,\nabla_z D_\psi(\widehat z_r,r)^{\top}\,\partial_\theta\widehat z_r \nonumber\\
&=\nabla_z D_\psi(\widehat z_r,r)^{\top}\,\partial_\theta u_\theta .
\end{align}
With $f_{\mathrm{ls}}(a,b)=(a-1)^2+(b+1)^2$, the generator receives
\begin{equation}
\nabla_\theta\mathcal{L}^{G}_{\texttt{CAMF}}
=2\,\mathbb{E}\Big[\big(\mathcal{A}^{\text{fake}}-1\big)\,
\nabla_z D_\psi(\widehat z_r,r)^{\top}\,\partial_\theta u_\theta\Big],
\end{equation}
which contains no $1/(t-r)$ factor. Its direction is set by the slope of the
potential at the model endpoint, and its size by how far the fake score is
from its target.

\noindent\textbf{Bounded scores.}
By assumption, $D_\psi$ is $L_z$-Lipschitz in $z$ and $L_t$-Lipschitz in $t$, so
$\|\nabla_z D_\psi\|\le L_z$ and $|\partial_t D_\psi|\le L_t$. Each
\texttt{JVP} inside the average of
Eq.~\eqref{eq:finite_interval_discriminator} then satisfies
\begin{equation}
\big|\texttt{JVP}_{\widehat z_s,s}(D_\psi;u_\theta,1)\big|
=\big|\nabla_z D_\psi\cdot u_\theta+\partial_t D_\psi\big|
\le L_z\|u_\theta\|+L_t .
\end{equation}
An average cannot exceed the bound on what it averages, so
$|\mathcal{A}^{\text{fake}}|\le L_t+L_z\|u_\theta\|$ for every $t-r$. The
same argument along the real path gives
$|\mathcal{A}^{\text{real}}|\le L_t+L_z\|v^{*}\|$.
\end{proof}

\subsection{Conditional versus Marginal Velocity}
\label{app:conditional_velocity}
The velocity $v^{*}=\epsilon-x$ in Eq. \eqref{eq:proof_real_endpoint} is \emph{conditional}: it depends on the specific pair $(x,\epsilon)$ that produced
$z_t$, not on $z_t$ alone. This mirrors flow matching, where the network regresses the same conditional target, whose expectation
$\mathbb{E}[\epsilon-x\mid z_t]$ is the marginal velocity field~\cite{lipman2023flow}. Our generated branch is therefore similarly trained to match the marginal field using the conditional velocity.

\subsection{Extended Centering Penalty}
\label{app:centering_ablation}
Since $\mathcal{A}$ depends only on differences of $D_\psi$, its absolute value can drift during training. We investigate the effect of extending CAFM's centering penalty to all three states evaluated by \texttt{CAMF}:
$z_t$, $z_r$, and $\widehat{z}_r$. As shown in Tab.~\ref{tab:cp-ablation}, our chosen centering-penalty formulation achieves the best performance among the evaluated variants.

\begin{table*}[t]
\centering
\begin{minipage}[t]{0.49\textwidth}
    \vspace{0pt}
    \centering
    \scriptsize
    \setlength{\tabcolsep}{3pt}
    \renewcommand{\arraystretch}{1.2}
    \caption{\textbf{$\boldsymbol{\mathcal{L}_\text{cp}}$ ablation}. iMF-XL/2 \texttt{MF-T}~+~\texttt{CAMF} at $4$ NFE. $D_{z_t}$:$D_{\psi}(z_t)$,
    $D_{z_r}$:$D_{\psi}(z_r)$, $D_{\widehat{z}_r}$:$D_{\psi}(\widehat{z}_r)$. \textbf{Bold}: best average.}
    \label{tab:cp-ablation}
    \vspace{-5pt}
    \resizebox{\linewidth}{!}{%
    \begin{tabular}{
    >{\scriptsize}l
    *{4}{>{\scriptsize}c}
    !{\color{black!20}\vrule}
    *{4}{>{\scriptsize}c}}
    \toprule
    \tblstrut
    & \multicolumn{4}{c}{\scriptsize\textbf{FID} $\downarrow$}
    & \multicolumn{4}{c}{\scriptsize\textbf{FDD} $\downarrow$} \\
    \cmidrule(lr){2-5}\cmidrule(lr){6-9}
    \textbf{$\bm{\mathcal{L}_{\mathrm{cp}}}$}
    & ArtB. & CUB & Cars & \textbf{Avg.}
    & ArtB. & CUB & Cars & \textbf{Avg.} \\
    \midrule
    \tblstrut
    $\varnothing$
    & ${\color{gray!90}7.49}$ & ${\color{gray!90}2.85}$ & ${\color{gray!90}3.13}$ & $4.49$
    & ${\color{gray!90}160.24}$ & ${\color{gray!90}86.10}$ & ${\color{gray!90}111.71}$ & $119.35$ \\
    $D_{z_t}^2$
    & ${\color{gray!90}7.81}$ & ${\color{gray!90}2.79}$ & ${\color{gray!90}3.24}$ & $4.61$
    & ${\color{gray!90}157.70}$ & ${\color{gray!90}82.70}$ & ${\color{gray!90}125.14}$ & $121.85$ \\
    $D_{z_t}^2 + D_{\widehat{z}_r}^2$
    & ${\color{gray!90}6.88}$ & ${\color{gray!90}2.80}$ & ${\color{gray!90}3.36}$ & $4.35$
    & ${\color{gray!90}151.95}$ & ${\color{gray!90}84.60}$ & ${\color{gray!90}128.25}$ & $121.60$ \\
    $D_{z_t}^2 + D_{z_r}^2$
    & ${\color{gray!90}8.11}$ & ${\color{gray!90}2.85}$ & ${\color{gray!90}3.24}$ & $4.73$
    & ${\color{gray!90}170.76}$ & ${\color{gray!90}83.89}$ & ${\color{gray!90}104.18}$ & $119.61$ \\
    $(D_{z_t}+D_{\hat z_r})^2$
    & ${\color{gray!90}7.68}$ & ${\color{gray!90}2.77}$ & ${\color{gray!90}3.27}$ & $4.57$
    & ${\color{gray!90}159.88}$ & ${\color{gray!90}82.37}$ & ${\color{gray!90}123.60}$ & $121.95$ \\
    $(D_{z_t}+D_{z_r}+D_{\widehat{z}_r})^2$
    & ${\color{gray!90}7.58}$ & ${\color{gray!90}2.70}$ & ${\color{gray!90}3.53}$ & $4.60$
    & ${\color{gray!90}154.68}$ & ${\color{gray!90}80.63}$ & ${\color{gray!90}118.21}$ & $117.84$ \\
  \rowcolor{gray!7} \tblstrut $\bm{D_{z_t}^2 + D_{z_r}^2 + D_{\widehat{z}_r}}^2$
    & ${\color{gray!90}6.71}$ & ${\color{gray!90}2.88}$ & ${\color{gray!90}3.07}$ & $\mathbf{4.22}$
    & ${\color{gray!90}146.40}$ & ${\color{gray!90}84.24}$ & ${\color{gray!90}102.37}$ & $\mathbf{111.00}$ \\
    \bottomrule
    \end{tabular}}
\end{minipage}
\hfill
\begin{minipage}[t]{0.49\textwidth}
  \centering
  \scriptsize
  \setlength{\tabcolsep}{3pt}
  \renewcommand{\arraystretch}{1.2}
  \caption{\textbf{$\boldsymbol{\mathcal{L}_{\mathrm{ot}}}$ ablation}. iMF-XL/2 at $4$ NFE. \textbf{Bold}: best average per group.}
  
  \label{tab:ot-ablation}
  \resizebox{\linewidth}{!}{%
  \begin{tabular}{
  >{\scriptsize}l
  >{\scriptsize}c
  *{4}{>{\scriptsize}c}
  !{\color{black!20}\vrule}
  *{4}{>{\scriptsize}c}}
  \toprule
  & & \multicolumn{4}{c}{\textbf{FID} $\downarrow$} & \multicolumn{4}{c}{\textbf{FDD} $\downarrow$} \\
  \cmidrule(lr){3-6}\cmidrule(lr){7-10}
  \textbf{Method} & $\lambda_{\mathrm{ot}}$ & ArtB. & Calt. & CUB & \textbf{Avg.} & ArtB. & Calt. & CUB & \textbf{Avg.} \\
  \midrule
  \texttt{CAMF} & $4$ & ${\color{gray!90}12.18}$ & ${\color{gray!90}59.80}$ & ${\color{gray!90}12.52}$ & $28.17$ & ${\color{gray!90}231.13}$ & ${\color{gray!90}974.38}$ & ${\color{gray!90}454.71}$ & $553.41$ \\
  \texttt{CAMF} & $0$ & ${\color{gray!90}10.80}$ & ${\color{gray!90}44.42}$ & ${\color{gray!90}9.34}$ & $\mathbf{21.52}$ & ${\color{gray!90}221.40}$ & ${\color{gray!90}812.31}$ & ${\color{gray!90}367.57}$ & $\mathbf{467.09}$ \\
  \arrayrulecolor{black!10}\cmidrule(lr){2-10}\arrayrulecolor{black}
  \texttt{MF-T}~+~\texttt{CAMF} & $4$ & ${\color{gray!90}7.59}$ & ${\color{gray!90}24.02}$ & ${\color{gray!90}3.23}$ & $11.61$ & ${\color{gray!90}159.11}$ & ${\color{gray!90}408.27}$ & ${\color{gray!90}98.86}$ & $222.08$ \\
  \rowcolor{gray!7}\texttt{MF-T}~+~\texttt{CAMF} & $0$ & ${\color{gray!90}6.71}$ & ${\color{gray!90}22.10}$ & ${\color{gray!90}2.88}$ & $\mathbf{10.56}$ & ${\color{gray!90}146.40}$ & ${\color{gray!90}378.09}$ & ${\color{gray!90}84.24}$ & $\mathbf{202.91}$ \\
  \bottomrule
  \end{tabular}}
\end{minipage}

\end{table*}

\subsection{Optimal-Transport Regularization}
\label{app:ot-ablation}
Following AFM and CAFM, we can add an optimal-transport norm term, $\mathcal{L}_\text{ot}=\lambda_\mathrm{ot}\,\mathbb{E}_{z_t,r,t}\,\|\frac{1}{n} u_\theta(z_t, r, t)\|^2$, to the generator.
AFM needs it to stabilize training from scratch~\citep{lin2025adversarial}.
CAFM sets $\lambda_\mathrm{ot}=0$ for post-training, to avoid the bias of the
Euclidean norm~\citep{lin2026continuous}. However, CAFM post-trains within the pre-training domain. We test whether a domain shift brings back the
need for this term. Tab.~\ref{tab:ot-ablation} compares $\lambda_\mathrm{ot}=4$
and $\lambda_\mathrm{ot}=0$ in two settings: \texttt{CAMF} applied directly to
the ImageNet-pretrained iMF-XL/2 model, and \texttt{MF-T}~+~\texttt{CAMF}. In both,
$\lambda_\mathrm{ot}=0$ improves FID and FDD on every dataset. We therefore use $\lambda_\mathrm{ot}=0$ in all experiments.

\section{Additional Experimental Details}
\label{sec:additional_experimental}

\subsection{Hyperparameters}
\label{app:Hyperparameters}

Tab.~\ref{tab:datasets} lists the number of classes and images of each target dataset. All images are used for training and as the FID/FDD reference set, at $256\times256$ resolution. Latent-space families  (iMF, SiT, DiT) train on pre-encoded VAE latents. Tab.~\ref{tab:hyperparameters} summarizes the hyperparameters and model configurations used for training FT, \texttt{CAMF}, and \texttt{MF-T} starting from all source models.  \texttt{MF-T} and FT are trained for $30$K steps on SiT and DiT and $40$K steps for JiT and iMF. \texttt{CAMF} and AFM run for $30$K generator steps when applied at the second stage and $60$K otherwise. \texttt{CAMF} uses a discriminator-only warm-up of $5$K steps, then $4$ discriminator updates per generator update. %

\begin{table}[t]
  \centering
  \scriptsize
  \setlength{\tabcolsep}{4pt}
  \renewcommand{\arraystretch}{1.1}
  \caption{\textbf{Selected target datasets.} All images of each dataset are used for training and as the FID/FDD reference set.}
  \label{tab:datasets}

  \resizebox{\columnwidth}{!}{%
  \begin{tabular}{l l c c}
  \toprule
  \textbf{Dataset} & \textbf{Domain} & \textbf{Classes} & \textbf{Images} \\
  \midrule
  ArtBench-10~\citep{liao2022artbench} & artworks & $10$ & $60{,}000$ \\
  Caltech-101~\citep{griffin2007caltech} & objects & $101$ & $8{,}677$ \\
  CUB-200~\citep{wah2011caltech} & birds & $200$ & $11{,}796$ \\
  Food-101~\citep{bossard2014food} & food & $101$ & $101{,}000$ \\
  Stanford Cars~\citep{krause20133d} & cars & $196$ & $16{,}139$ \\
  \bottomrule
  \end{tabular}%
  }
  \vspace{-10pt}
\end{table}

\subsection{Training Algorithms}
\label{app:training_algorithms}

We provide the pseudocode for \texttt{MF-T} and \texttt{CAMF}. Alg.~\ref{alg:mfa_init} describes the initialization of the target MF model from pretrained source models with different prediction parameterizations, mapping each source output space to a common instantaneous velocity under an aligned time parameterization. Alg.~\ref{alg:mfa_train} presents the \texttt{MF-T} training procedure, which jointly adapts the model to the target domain and enables few-step generation. Finally, Alg.~\ref{alg:caimf-training} describes the \texttt{CAMF} adversarial post-training stage used to refine the target-domain generator, alternating between discriminator and generator updates detailed in Algs.~\ref{alg:caimf-disc} and~\ref{alg:caimf-gen}.

\subsection{Computational Cost}
\label{app:compute_cost}

Tab.~\ref{tab:cost} reports training and inference cost on CUB-200. All runs use a single NVIDIA H100 80GB SXM GPU. Latent-space families (iMF, SiT, DiT) pre-encode the target dataset with the frozen VAE, a one-time cost excluded from the per-stage numbers. Training wall-clock includes periodic in-training evaluation, which accounts for a non-trivial fraction of the total.

\section{Additional Results}
\label{sec:additional_results}

\subsection{1-step Generation}
\label{app:one-step_results}
Tab.~\ref{tab:onestep} reports 1-step (NFE=$1$) generation results using the same trained checkpoints evaluated in Tab.~\ref{tab:main_labeled_init_teachers}. Across both iMF-XL/2 and SiT-XL/2 initializations, \texttt{MF-T}~+~\texttt{CAMF} achieves the best average FID and FDD at $1$ NFE. However, $1$-step generation quality under the limited data of transfer still lags behind its $4$-NFE quality. For DiT and JiT, $1$-step generations are perceptually indistinguishable from noise, so we omit them. We leave this limitation for future work.

\subsection{Inception Score}
\label{app:inception_score}
Tab.~\ref{tab:main_is} reports Inception Score (IS) as a complementary measure of sample quality and diversity. All scores use the same samples as FID and FDD. The IS results generally support the trends observed with FID and FDD. Since IS is biased towards ImageNet-like distributions, it should be treated only as a supporting metric, not a replacement for FID and FDD.

\subsection{Fidelity and Diversity of \texttt{CAMF}}
\label{app:fidelity_diversity}
Tab.~\ref{tab:precision_recall} tests whether \texttt{CAMF}'s FID gain comes at the cost of diversity, which is a known risk of adversarial objectives. Over two backbones and two datasets, and in both Inception and DINOv2 feature spaces, \texttt{CAMF} increases precision, density, and coverage on average, while slightly decreasing recall. \texttt{CAMF}'s improvement is therefore mainly a gain in fidelity. However, the small decline of recall does not suggest mode collapse, as coverage increases everywhere and density stays below one.

\subsection{2-stage vs.\ 1-stage \texttt{MF-T}}
\label{app:two_stage}
We compare 1-stage \texttt{MF-T} with a 2-stage transfer pipeline to assess whether joint adaptation and acceleration is more effective than performing them sequentially. The 2-stage pipeline first fine-tunes the source model on the target domain (FT), and then trains with the \texttt{MF-T} objective, using the same recipe and budget as the 1-stage run. It therefore uses twice the training steps in total ($60$K vs.\ $30$K for SiT and DiT, $80$K vs.\ $40$K for JiT). Tab.~\ref{tab:one_vs_two_nfe} reports per-dataset results across sampling budgets, summarized in Fig.~\ref{fig:main_one_vs_two}. As NFE decreases, 1-stage \texttt{MF-T} becomes increasingly advantageous despite half the training steps, and outperforms the 2-stage pipeline at the lowest NFE for SiT and DiT.

\subsection{Checkpoint Selection}
\label{app:checkpoint_selection}
For each model family, we use a fixed training budget~(Tab.~\ref{tab:hyperparameters}) and select checkpoints based on FID. Specifically, few-step models are selected using their lowest FID at $4$ NFE, while many-step models are selected at $64$ NFE. This protocol is used consistently across methods within each model family. To assess sensitivity to checkpoint selection, Tab.~\ref{tab:posttrain_final} reports results using the final training checkpoints. At both $4$ and $250$ NFE, \texttt{MF-T}~+~\texttt{CAMF} gives the lowest FID and FDD on both backbones, and on SiT it also surpasses the $250{\times}2$-NFE FT teacher.

\subsection{Data Efficiency vs. NFE}
\label{app:data_efficiency}

Tab.~\ref{tab:mft_data_efficiency} provides the per-dataset FID and FDD results underlying Fig.~\ref{fig:main_dataeff} across training-data fractions and sampling budgets. \texttt{MF-T}~+~\texttt{CAMF} consistently improves over \texttt{MF-T} at $2$ and $4$ NFE, even with only $10\%$ of the target training data, where it already outperforms \texttt{MF-T} trained on the full data ($4$-NFE FID $5.88$ vs.\ $8.48$, FDD $252.43$ vs.\ $293.58$).

\subsection{Additional Qualitative Results}
\label{app:additional_qualitative}
Fig.~\ref{fig:qualitative-appendix} presents additional
uncurated samples produced by \texttt{MF-T}~+~\texttt{CAMF} across the four source-model families and five target datasets. It transfers across source parameterizations while preserving target-domain structure and
fine-grained visual details at $4$ NFE.

\begin{table*}[t]
  \centering
  \scriptsize
  \setlength{\aboverulesep}{0.4ex}\setlength{\belowrulesep}{0.65ex}  %
  \setlength{\tabcolsep}{5pt}
  \caption{\textbf{Backbone configurations and hyperparameters.} All models are ImageNet-based and utilize a $256\times256$ resolution. NFE $250\xtwo$ denotes $[250~\text{denoising steps}~\times 2 ]$ due to CFG. Batch sizes are effective (per-device batch $\times$ gradient accumulation). }
  \label{tab:hyperparameters}
  \begin{tabular}{@{}l cccc@{}}
  \toprule
  & \textbf{iMF-XL/2} & \textbf{SiT-XL/2} & \textbf{DiT-XL/2} & \textbf{JiT-H/16} \\
  \cmidrule(l){2-5}
  & \multicolumn{4}{c}{\textit{Source model}} \\
  \arrayrulecolor{black!10}\cmidrule(lr){2-5}\arrayrulecolor{black}
  Parameters              & $610$M            & $675$M${}+{}$$49$M & $675$M${}+{}$$49$M & $953$M \\
  Reference NFE           & $1/2$             & $250\xtwo$      & $250\xtwo$      & $50\xtwo$ \\
  Reference FID           & $1.72/1.54$       & $2.06$            & $2.27$            & $1.86$ \\
  Prediction space        & $u$               & $v$               & $\epsilon$        & $x$ \\
  Loss space              & $v$               & $v$               & $\epsilon$        & $v$ \\
  Sampling space          & latent            & latent            & latent            & pixel \\
  \midrule
  & \multicolumn{4}{c}{\textit{Fine-tuning baseline} (FT)} \\
  \arrayrulecolor{black!10}\cmidrule(lr){2-5}\arrayrulecolor{black}
  Learning rate           & --                & $1\mathrm{e}{-4}$ & $1\mathrm{e}{-4}$ & $3\mathrm{e}{-6}$ \\
  Batch size (effective)  & --                & $32$              & $32$              & $32$ \\
  Training steps          & --                & $30$K             & $30$K             & $40$K \\
  Adam $\beta_2$          & --                & $0.999$           & $0.999$           & $0.95$ \\
  EMA                     & --                & --                & --                & $0.9998$ \\
  \midrule
  & \multicolumn{4}{c}{\textit{MeanFlow Transfer} (\texttt{MF-T})} \\
  \arrayrulecolor{black!10}\cmidrule(lr){2-5}\arrayrulecolor{black}
  Velocity map            & --                & --                & $\checkmark$      & $\checkmark$ \\
  Time flip / scale       & -- / --           & -- / --           & \cmark\ / $999$   & -- / -- \\
  Learning rate           & $1\mathrm{e}{-4}$ & $1\mathrm{e}{-4}$ & $1\mathrm{e}{-4}$ & $1\mathrm{e}{-5}$ \\
  Batch size              & $32$              & $32$              & $32$              & $32$ \\
  Training steps          & $40$K             & $30$K             & $30$K             & $40$K \\
  Adam $\beta_2$          & $0.95$            & $0.95$            & $0.95$            & $0.95$ \\
  Gradient clip           & --                & --                & $\checkmark$      & $\checkmark$ \\
  Training $w$            & $[1,8]$           & $2$             & $2$             & $2$ \\
  \midrule
  & \multicolumn{4}{c}{\textit{Continuous Adversarial MeanFlow} (\texttt{CAMF})} \\
  \arrayrulecolor{black!10}\cmidrule(lr){2-5}\arrayrulecolor{black}
  Generator / disc.\ LR   & $1\mathrm{e}{-5}$ & $1\mathrm{e}{-5}$ & $1\mathrm{e}{-5}$ & $1\mathrm{e}{-6}$ \\
  Batch size              & $2$               & $2$               & $2$               & $2$ \\
  Generator updates       & $30$K             & $30$K             & $30$K             & $30$K \\
  Discriminator updates   & \multicolumn{4}{c}{$4\times$ per generator update ($+5$K warm-up)} \\
  $\lambda_{\mathrm{adv}}$ & $1.0$            & $1.0$             & $1.0$             & $1.0$ \\
  $\lambda_{\mathrm{ot}}$  & $0$              & $0$               & $0$               & $0$ \\
  $\lambda_{\mathrm{cp}}$  & $0.001$          & $0.001$           & $0.001$           & $0.001$ \\
  Interval $\epsilon$     & $0.001$           & $0.001$           & $0.001$           & $0.001$ \\
  Adam $\beta_1$          & $0.0$             & $0.0$             & $0.0$             & $0.0$ \\
  Adam $\beta_2$          & $0.95$            & $0.95$            & $0.95$            & $0.95$ \\
  Weight decay            & $0$               & $0$               & $0$               & $0$ \\
  EMA decay               & $0.9999$          & $0.9999$          & $0.9999$          & $0.9999$ \\
  Class dropout           & $0.1$             & $0.1$             & $0.1$             & $0.1$ \\
  Training $w$            & $[1,8]$           & $2$             & $2$             & $2$ \\
  \midrule
  & \multicolumn{4}{c}{\textit{Evaluation operating point}} \\
  \arrayrulecolor{black!10}\cmidrule(lr){2-5}\arrayrulecolor{black}
  Sampling $\omega$       & $7.5$             & $2$             & $2$             & $2$ \\
  Sampling $\omega$ interval       & $[0.4,\,0.65]$    & $[0,\,1]$         & $[0,\,1]$         & $[0.1,\,1.0]$ \\
  Sampled NFE             & $\geq2$           & $\geq4$           & $\geq8$             & $\geq4$ \\
  \bottomrule
  \end{tabular}
\end{table*}

  \begin{figure*}[t!]
  \centering

  \begin{minipage}[t]{\textwidth}
  \hrule
  \vspace{3pt}
  \refstepcounter{algorithm}\label{alg:caimf-training}
  {\bfseries Algorithm~\thealgorithm:} \texttt{CAMF} post-training: alternating schedule.
  \vspace{3pt}
  \hrule
  \vspace{6pt}
  {\ttfamily\small
  \textcolor{algblue}{\# u\_theta: target MF generator (from \texttt{MF-T}), D\_psi: scalar discriminator potential}\\
  \textcolor{algblue}{\# p\_trg: target-domain data, c: condition, lambda\_adv, lambda\_cp >= 0}\\
  \textcolor{algblue}{\# N\_D: D updates per generator update, N\_warm: D-only warm-up steps}\\[1em]
  step \textcolor{algviolet}{=} 0\\
  \textcolor{algviolet}{while not} converged:\\
  \hspace*{1.5em}\textcolor{algblue}{\# Sample a target-domain transport interval}\\
  \hspace*{1.5em}x, c \textcolor{algviolet}{=} \textcolor{algpink}{sample\_target}(p\_trg)\\
  \hspace*{1.5em}e~~~~\textcolor{algviolet}{=} \textcolor{algpink}{randn\_like}(x)\\
  \hspace*{1.5em}r, t \textcolor{algviolet}{=} \textcolor{algpink}{sample\_interval}() \ \ \textcolor{algblue}{\# 0 <= r < t <= 1, t - r >= 1e-3}\\
  \hspace*{1.5em}w~~~~\textcolor{algviolet}{=} \textcolor{algpink}{sample\_cfg\_scale}()\\
  \hspace*{1.5em}delta \textcolor{algviolet}{=} t \textcolor{algviolet}{-} r\\[1em]
  \hspace*{1.5em}z\_t \textcolor{algviolet}{=} (1 \textcolor{algviolet}{-} t) \textcolor{algviolet}{*} x \textcolor{algviolet}{+} t \textcolor{algviolet}{*} e\\
  \hspace*{1.5em}z\_r \textcolor{algviolet}{=} (1 \textcolor{algviolet}{-} r) \textcolor{algviolet}{*} x \textcolor{algviolet}{+} r \textcolor{algviolet}{*} e\\
  \hspace*{1.5em}batch \textcolor{algviolet}{=} (z\_t, z\_r, r, t, delta, w, c)\\[1em]
  \hspace*{1.5em}\textcolor{algviolet}{if} step \textcolor{algviolet}{<} N\_warm \textcolor{algviolet}{or} (step \textcolor{algviolet}{-} N\_warm) \textcolor{algviolet}{\%} (N\_D \textcolor{algviolet}{+} 1) \textcolor{algviolet}{<} N\_D:\\
  \hspace*{3em}psi \textcolor{algviolet}{=} \textcolor{algpink}{DiscriminatorUpdate}(batch, psi) \ \ \ \textcolor{algblue}{~\# Alg.~\ref{alg:caimf-disc}}\\
  \hspace*{1.5em}\textcolor{algviolet}{else}:\\
  \hspace*{3em}theta \textcolor{algviolet}{=} \textcolor{algpink}{GeneratorUpdate}(batch, theta) \ \ \ \ \textcolor{algblue}{\# Alg.~\ref{alg:caimf-gen}}\\[1em]
  \hspace*{1.5em}step \textcolor{algviolet}{=} step \textcolor{algviolet}{+} 1\\
  }
  \vspace{3pt}
  \hrule
  \end{minipage}

  \vspace{10pt}

  \begin{minipage}[t]{0.485\textwidth}
  \hrule
  \vspace{3pt}
  \refstepcounter{algorithm}\label{alg:caimf-disc}
  {\bfseries Algorithm~\thealgorithm:} \texttt{DiscriminatorUpdate}.
  \vspace{3pt}
  \hrule
  \vspace{6pt}
  {\ttfamily\small
  \textcolor{algblue}{\# Freeze u\_theta, update D\_psi}\\[1.2em]
  \textcolor{algviolet}{def} \textcolor{algpink}{DiscriminatorUpdate}(batch, psi):\\[1em]
  \hspace*{1.5em}z\_t, z\_r, r, t, delta, w, c \textcolor{algviolet}{=} batch\\[1em]
  \hspace*{1.5em}u \textcolor{algviolet}{=} \textcolor{algpink}{stopgrad}(\textcolor{algpink}{u\_theta}(z\_t, r, t, w, c))\\
  \hspace*{1.5em}zhat\_r \textcolor{algviolet}{=} z\_t \textcolor{algviolet}{-} delta \textcolor{algviolet}{*} u\\[1em]
  \hspace*{1.5em}D\_zt \ \ \textcolor{algviolet}{=} \textcolor{algpink}{D\_psi}(z\_t, \ \ \ t, w, c)\\
  \hspace*{1.5em}D\_zr \ \ \textcolor{algviolet}{=} \textcolor{algpink}{D\_psi}(z\_r, \ \ \ r, w, c)\\
  \hspace*{1.5em}D\_zhat \textcolor{algviolet}{=} \textcolor{algpink}{D\_psi}(zhat\_r, r, w, c)\\[1em]
  \hspace*{1.5em}a\_real \textcolor{algviolet}{=} (D\_zt \textcolor{algviolet}{-} D\_zr) \ \ \textcolor{algviolet}{/} delta\\
  \hspace*{1.5em}a\_fake \textcolor{algviolet}{=} (D\_zt \textcolor{algviolet}{-} D\_zhat) \textcolor{algviolet}{/} delta\\[1em]
  \hspace*{1.5em}L\_D\_adv \textcolor{algviolet}{=} \textcolor{algpink}{mean}( (a\_real \textcolor{algviolet}{-} 1) \textcolor{algviolet}{**} 2\\
  \hspace*{10em}\textcolor{algviolet}{+} (a\_fake \textcolor{algviolet}{+} 1) \textcolor{algviolet}{**} 2 )\\
  \hspace*{1.5em}L\_D\_cp \ \textcolor{algviolet}{=} \textcolor{algpink}{mean}( D\_zt \textcolor{algviolet}{**} 2 \textcolor{algviolet}{+} D\_zr \textcolor{algviolet}{**} 2\\
  \hspace*{10em}\textcolor{algviolet}{+} D\_zhat \textcolor{algviolet}{**} 2 )\\
  \hspace*{1.5em}L\_D \textcolor{algviolet}{=} L\_D\_adv \textcolor{algviolet}{+} lambda\_cp \textcolor{algviolet}{*} L\_D\_cp\\[1em]
  \hspace*{1.5em}\textcolor{algviolet}{return} \textcolor{algpink}{optimizer\_D}( psi, \textcolor{algpink}{grad}(L\_D, psi) )\\
  }
  \vspace{3pt}
  \hrule
  \end{minipage}
  \hfill
  \begin{minipage}[t]{0.485\textwidth}
  \hrule
  \vspace{3pt}
  \refstepcounter{algorithm}\label{alg:caimf-gen}
  {\bfseries Algorithm~\thealgorithm:} \texttt{GeneratorUpdate}.
  \vspace{3pt}
  \hrule
  \vspace{6pt}
  {\ttfamily\small
  \textcolor{algblue}{\# Freeze D\_psi, update u\_theta}\\[1.2em]
  \textcolor{algviolet}{def} \textcolor{algpink}{GeneratorUpdate}(batch, theta):\\[1.5em]
  \hspace*{1.5em}z\_t, \_, r, t, delta, w, c \textcolor{algviolet}{=} batch\\[1.5em]
  \hspace*{1.5em}u \textcolor{algviolet}{=} \textcolor{algpink}{u\_theta}(z\_t, r, t, w, c)\\
  \hspace*{1.5em}zhat\_r \textcolor{algviolet}{=} z\_t \textcolor{algviolet}{-} delta \textcolor{algviolet}{*} u\\[1.5em]
  \hspace*{1.5em}D\_zt \ \ \textcolor{algviolet}{=} \textcolor{algpink}{D\_psi}(z\_t, \ \ \ t, w, c)\\
  \hspace*{1.5em}D\_zhat \textcolor{algviolet}{=} \textcolor{algpink}{D\_psi}(zhat\_r, r, w, c)\\[1.5em]
  \hspace*{1.5em}a\_fake \textcolor{algviolet}{=} (D\_zt \textcolor{algviolet}{-} D\_zhat) \textcolor{algviolet}{/} delta\\[1.5em]
  \hspace*{1.5em}L\_G\_adv \textcolor{algviolet}{=} \textcolor{algpink}{mean}( (a\_fake \textcolor{algviolet}{-} 1) \textcolor{algviolet}{**} 2 )\\[1.5em]
  \hspace*{1.5em}L\_G \textcolor{algviolet}{=} lambda\_adv \textcolor{algviolet}{*} L\_G\_adv\\[1.5em]
  \hspace*{1.5em}\textcolor{algviolet}{return} \textcolor{algpink}{optimizer\_G}( theta,\\
  \hspace*{8em}
  \textcolor{algpink}{grad}(L\_G, theta) )\\
  }
  \vspace{3pt}
  \hrule
  \end{minipage}
  \label{fig:caimf_algorithm}
  \end{figure*}

  \begin{figure*}[t!]
  \centering
  \begin{minipage}[t]{0.485\textwidth}

  \hrule
  \vspace{3pt}
  \refstepcounter{algorithm}\label{alg:mfa_init}
  {\bfseries Algorithm~\thealgorithm:} \texttt{MF-T} initialization.
  \vspace{3pt}
  \hrule
  \vspace{6pt}
  {\ttfamily\small
  \textcolor{algblue}{\# src\_phi: pretrained source model}\\
  \textcolor{algblue}{\# param\_type: one of \{x, eps, v, u\}}\\
  \textcolor{algblue}{\# T, flip: source time convention}\\[1em]
  \textcolor{algviolet}{def} \textcolor{algpink}{map\_time}(t):\\
  \hspace*{1.5em}\textcolor{algblue}{\# flow time -> source time}\\
  \hspace*{1.5em}\textcolor{algviolet}{return} ((1 \textcolor{algviolet}{-} t) \textcolor{algviolet}{if} flip \textcolor{algviolet}{else} t) \textcolor{algviolet}{*} T\\[1em]
  \textcolor{algviolet}{class} \textcolor{algpink}{VelocityMap}(src\_phi, param\_type):\\
  \hspace*{1.5em}\textcolor{algblue}{\# returns v\_phi(z\_t, t, c) (Tab.~\ref{tab:3_unified_velocity_recovery})}\\
  \hspace*{1.5em}\textcolor{algviolet}{def} \textcolor{algpink}{forward}(z\_t, t, c):\\
  \hspace*{3em}s \textcolor{algviolet}{=} \textcolor{algpink}{map\_time}(t)\\[1em]
  \hspace*{3em}\textcolor{algviolet}{if} param\_type \textcolor{algviolet}{==} "x":\\
  \hspace*{4.5em}x\_hat \textcolor{algviolet}{=} \textcolor{algpink}{src\_phi}(z\_t, s, c)\\
  \hspace*{4.5em}\textcolor{algviolet}{return} (z\_t \textcolor{algviolet}{-} x\_hat) \textcolor{algviolet}{/} t\\[1em]
  \hspace*{3em}\textcolor{algviolet}{elif} param\_type \textcolor{algviolet}{==} "eps":\\
  \hspace*{4.5em}eps\_hat \textcolor{algviolet}{=} \textcolor{algpink}{src\_phi}(z\_t, s, c)\\
  \hspace*{4.5em}\textcolor{algviolet}{return} (eps\_hat \textcolor{algviolet}{-} z\_t) \textcolor{algviolet}{/} (1 \textcolor{algviolet}{-} t)\\[1em]
  \hspace*{3em}\textcolor{algviolet}{elif} param\_type \textcolor{algviolet}{==} "v":\\
  \hspace*{4.5em}\textcolor{algviolet}{return} \textcolor{algpink}{src\_phi}(z\_t, s, c)\\[1em]
  \hspace*{3em}\textcolor{algviolet}{elif} param\_type \textcolor{algviolet}{==} "u":\\
  \hspace*{4.5em}\textcolor{algviolet}{return} \textcolor{algpink}{src\_phi}(z\_t, s, s, c)\\
  \hspace*{4.5em}\textcolor{algblue}{\# or use the v head}\\[1em]
  v\_phi \textcolor{algviolet}{=} \textcolor{algpink}{VelocityMap}(src\_phi, param\_type)\\
  u\_theta \textcolor{algviolet}{=} \textcolor{algpink}{deep\_copy}(v\_phi)\\
  u\_theta \textcolor{algviolet}{=} \textcolor{algpink}{add\_r\_w\_inputs}(u\_theta)\\
  }
  \hrule
  \end{minipage}
  \hfill
  \begin{minipage}[t]{0.485\textwidth}
  \hrule
  \vspace{3pt}
  \refstepcounter{algorithm}\label{alg:mfa_train}
  {\bfseries Algorithm~\thealgorithm:} \texttt{MF-T} training.
  \vspace{3pt}
  \hrule
  \vspace{6pt}
  {\ttfamily\small
  \textcolor{algblue}{\# u\_theta: target MF model (from Alg.~\ref{alg:mfa_init})}\\
  \textcolor{algblue}{\# v\_theta: instantaneous velocity (v head,}\\
  \textcolor{algblue}{\# \ \ or u\_theta(z\_t, t, t, w, c))}\\
  \textcolor{algblue}{\# x, c: target-domain batch, condition}\\[1em]
  r, t \textcolor{algviolet}{=} \textcolor{algpink}{sample\_interval}() \ \ \textcolor{algblue}{\# 0 <= r <= t <= 1}\\
  w~~~~\textcolor{algviolet}{=} \textcolor{algpink}{sample\_cfg\_scale}()\\
  e~~~~\textcolor{algviolet}{=} \textcolor{algpink}{randn\_like}(x)\\
  z\_t \textcolor{algviolet}{=} (1 \textcolor{algviolet}{-} t) \textcolor{algviolet}{*} x \textcolor{algviolet}{+} t \textcolor{algviolet}{*} e\\
  v\_star \textcolor{algviolet}{=} e \textcolor{algviolet}{-} x\\[1em]
  \textcolor{algblue}{\# guided target, from unit-guidance v}\\
  v\_c \textcolor{algviolet}{=} \textcolor{algpink}{v\_theta}(z\_t, t, 1, c)\\
  v\_u \textcolor{algviolet}{=} \textcolor{algpink}{v\_theta}(z\_t, t, 1, None)\\
  v\_g \textcolor{algviolet}{=} \textcolor{algpink}{stopgrad}(v\_star\\
  \hspace*{5em}\textcolor{algviolet}{+} (1 \textcolor{algviolet}{-} 1\textcolor{algviolet}{/}w) \textcolor{algviolet}{*} (v\_c \textcolor{algviolet}{-} v\_u))\\
  c \textcolor{algviolet}{=} \textcolor{algpink}{cond\_drop}(c)\\[1em]
  f \textcolor{algviolet}{=} \textcolor{algviolet}{lambda} z, r, t: \textcolor{algpink}{u\_theta}(z, r, t, w, c)\\
  u, dudt \textcolor{algviolet}{=} \textcolor{algpink}{jvp}(f, (z\_t, r, t),\\
  \hspace*{8.5em}(\textcolor{algpink}{stopgrad}(v\_c), 0, 1))\\
  v \textcolor{algviolet}{=} \textcolor{algpink}{v\_theta}(z\_t, t, w, c)\\
  V \textcolor{algviolet}{=} u \textcolor{algviolet}{+} (t \textcolor{algviolet}{-} r) \textcolor{algviolet}{*} \textcolor{algpink}{stopgrad}(dudt)\\[1em]
  loss \textcolor{algviolet}{=} \textcolor{algpink}{adp}(\textcolor{algpink}{sq}(V \textcolor{algviolet}{-} v\_g)) \textcolor{algviolet}{+}
  \textcolor{algpink}{adp}(\textcolor{algpink}{sq}(v \textcolor{algviolet}{-} v\_g))\\
  \textcolor{algblue}{\# adp: adaptive loss weighting, sq: squared}\\[1em]
  g \textcolor{algviolet}{=} \textcolor{algpink}{grad}(loss, theta)\\
  g \textcolor{algviolet}{=} \textcolor{algpink}{clip\_norm}(g, g\_max)\\
  theta \textcolor{algviolet}{=} \textcolor{algpink}{optimizer\_step}(theta, g)\\
  }
  \vspace{3pt}
  \hrule
  \end{minipage}
  \label{fig:mfa_algorithms}
  \end{figure*}

\begin{table*}[t]
    \centering
    \scriptsize
    \setlength{\tabcolsep}{4pt}
    \renewcommand{\arraystretch}{1.2}
    \caption{\textbf{Computational cost analysis on the CUB dataset}. single NVIDIA H100 80GB SXM. ``Updates'' counts generator optimizer steps. GPU-hours include periodic in-training evaluation (once every $2.5$K training steps, generating $10$K samples for four-step evaluation). Latency is the forward sampling time at batch $32$, warmup-then-timed and amortized per image.}
    \label{tab:cost}
    \begin{tabular}{
    >{\scriptsize}c
    >{\scriptsize}l
    *{3}{>{\scriptsize}r}
    !{\color{black!20}\vrule}
    *{3}{>{\scriptsize}r}}
    \toprule
    & & \multicolumn{3}{c}{\textbf{Training}} & \multicolumn{3}{c}{\textbf{Inference}} \\
    \cmidrule(lr){3-5}\cmidrule(lr){6-8}
    \textbf{ImageNet Init.} & \textbf{Stage} & \textbf{Updates} & \textbf{Batch} & \textbf{GPU-h} & \textbf{NFE} & \textbf{TFLOPs/img} & \textbf{Latency (ms)} \\
    \midrule
     & \texttt{MF-T}    & $40\text{K}$ & $32$ & $12.85$ & $4$ & $0.337$ & $17.48$ \\
    \multirow{-2}{*}{\begin{tabular}{@{}c@{}}\textbf{iMF-XL/2}\\[-1pt]{\scriptsize \cite{geng2026improved}}\end{tabular}}
     & $+$\texttt{CAMF} & $30\text{K}$ & $2$  & $12.88$ & $4$ & $0.337$ & $17.48$ \\
    \hline
     & FT               & $30\text{K}$ & $32$ & $1.54$ & $500$ & $33.33$ & $735.0$ \\
     & \texttt{MF-T}    & $30\text{K}$ & $32$ & $3.32$ & $4$   & $0.267$ & $5.88$ \\
    \multirow{-3}{*}{\begin{tabular}{@{}c@{}}\textbf{SiT-XL/2}\\[-1pt]{\scriptsize \cite{ma2024sit}}\end{tabular}}
     & $+$\texttt{CAMF} & $30\text{K}$ & $2$  & $6.99$ & $4$   & $0.267$ & $5.88$ \\
    \hline
     & FT               & $30\text{K}$ & $32$ & $1.56$ & $500$ & $33.33$ & $735.0$ \\
     & \texttt{MF-T}    & $30\text{K}$ & $32$ & $4.25$ & $4$   & $0.267$ & $5.88$ \\
    \multirow{-3}{*}{\begin{tabular}{@{}c@{}}\textbf{DiT-XL/2}\\[-1pt]{\scriptsize \cite{peebles2023scalable}}\end{tabular}}
     & $+$\texttt{CAMF} & $30\text{K}$ & $2$  & $7.24$ & $4$   & $0.267$ & $5.88$ \\
    \hline
     & FT               & $40\text{K}$ & $16$ & $5.78$ & $100$ & $0.799$ & $341.4$ \\
     & \texttt{MF-T}    & $40\text{K}$ & $16$ & $5.78$ & $4$   & $0.032$ & $13.64$ \\
    \multirow{-3}{*}{\begin{tabular}{@{}c@{}}\textbf{JiT-H/16}\\[-1pt]{\scriptsize \cite{li2026back}}\end{tabular}}
     & $+$\texttt{CAMF} & $30\text{K}$ & $2$  & $8.12$ & $4$   & $0.032$ & $13.64$ \\
    \bottomrule
    \end{tabular}
  \end{table*}

\begin{table*}[t]
  \centering
  \caption{\textbf{Velocity map vs. schedule alignment}. \textbf{Bold}: best average per group.}
    \centering
  \scriptsize
  \setlength{\tabcolsep}{5pt}
  \renewcommand{\arraystretch}{1.2}
  \begin{tabular}{l l c *{4}{c} !{\color{black!20}\vrule} *{4}{c}}
  \toprule
  &
  &
  & \multicolumn{4}{c}{\textbf{FID} $\downarrow$}
  & \multicolumn{4}{c}{\textbf{FDD} $\downarrow$} \\
  \cmidrule(lr){4-7}
  \cmidrule(lr){8-11}
  \textbf{Transformation}
  & \textbf{Source Prior}
  & \textbf{NFE}
  & ArtB.
  & CUB
  & Cars
  & \textbf{Avg.}
  & ArtB.
  & CUB
  & Cars
  & \textbf{Avg.} \\
  \midrule
  & DiT
    & $250\xtwo$
    & ${\color{gray!90}11.00}$
    & ${\color{gray!90}4.36}$
    & ${\color{gray!90}5.23}$
    & $\mathbf{6.87}$
    & ${\color{gray!90}227.2}$
    & ${\color{gray!90}112.8}$
    & ${\color{gray!90}167.5}$
    & $\mathbf{169.2}$ \\

  & DiT + velocity map + schedule alignment
    & $250\xtwo$
    & ${\color{gray!90}12.49}$
    & ${\color{gray!90}5.14}$
    & ${\color{gray!90}6.48}$
    & $8.03$
    & ${\color{gray!90}244.8}$
    & ${\color{gray!90}113.9}$
    & ${\color{gray!90}166.5}$
    & $175.1$ \\

  \rowcolor{gray!7}\cellcolor{white}
  & DiT + velocity map
    & $250\xtwo$
    & ${\color{gray!90}15.29}$
    & ${\color{gray!90}4.35}$
    & ${\color{gray!90}9.12}$
    & $9.59$
    & ${\color{gray!90}258.1}$
    & ${\color{gray!90}121.2}$
    & ${\color{gray!90}257.8}$
    & $212.4$ \\

  \arrayrulecolor{black!10}
  \cmidrule(lr){2-11}
  \arrayrulecolor{black}

  & DiT
    & $4\xtwo$
    & ${\color{gray!90}46.37}$
    & ${\color{gray!90}16.09}$
    & ${\color{gray!90}18.98}$
    & $27.14$
    & ${\color{gray!90}494.8}$
    & ${\color{gray!90}296.5}$
    & ${\color{gray!90}412.1}$
    & $\mathbf{401.1}$ \\

  & DiT + velocity map + schedule alignment
    & $4\xtwo$
    & ${\color{gray!90}43.57}$
    & ${\color{gray!90}73.07}$
    & ${\color{gray!90}88.88}$
    & $68.51$
    & ${\color{gray!90}538.0}$
    & ${\color{gray!90}852.4}$
    & ${\color{gray!90}1156.6}$
    & $849.0$ \\

  \rowcolor{gray!7}\cellcolor{white}\multirow{-6}{*}{\textbf{DiT} $\rightarrow$ \textbf{SiT}}
  & DiT + velocity map
    & $4\xtwo$
    & ${\color{gray!90}37.48}$
    & ${\color{gray!90}13.44}$
    & ${\color{gray!90}27.49}$
    & $\mathbf{26.14}$
    & ${\color{gray!90}464.4}$
    & ${\color{gray!90}300.1}$
    & ${\color{gray!90}540.2}$
    & $434.9$ \\

  \midrule

  & DiT
    & $4$
    & ${\color{gray!90}108.86}$
    & ${\color{gray!90}47.07}$
    & ${\color{gray!90}51.42}$
    & $69.12$
    & ${\color{gray!90}812.6}$
    & ${\color{gray!90}625.6}$
    & ${\color{gray!90}772.0}$
    & $736.7$ \\

  & DiT + velocity map + schedule alignment
    & $4$
    & ${\color{gray!90}200.73}$
    & ${\color{gray!90}244.58}$
    & ${\color{gray!90}298.71}$
    & $248.00$
    & ${\color{gray!90}2138.2}$
    & ${\color{gray!90}2839.2}$
    & ${\color{gray!90}3213.7}$
    & $2730.4$ \\

  \rowcolor{gray!7}\cellcolor{white}\multirow{-3}{*}{\textbf{DiT} $\rightarrow$ \textbf{\texttt{MF-T}}}
  & DiT + velocity map
    & $4$
    & ${\color{gray!90}44.52}$
    & ${\color{gray!90}17.22}$
    & ${\color{gray!90}24.97}$
    & $\mathbf{28.90}$
    & ${\color{gray!90}577.4}$
    & ${\color{gray!90}437.2}$
    & ${\color{gray!90}472.4}$
    & $\mathbf{495.7}$ \\
  \bottomrule
  \end{tabular}

  \label{tab:diff2flow_full}
\end{table*}

  \begin{table*}[t!]
    \centering
    \scriptsize
  \setlength{\aboverulesep}{0.4ex}\setlength{\belowrulesep}{0.65ex}  %
    \setlength{\tabcolsep}{2.6pt}
    \renewcommand{\arraystretch}{1.1}
    \caption{\textbf{1-step generation quality comparison}. \textcolor{afmblue}{AFM}, \textcolor{afmblue}{CAFM} and \textcolor{afmblue}{\texttt{CAMF}} as stand-alone methods are only tested on iMF-XL/2 where the initial model is MF-based (with \texttt{MF-T}~$\equiv$~FT on iMF). \underline{Underlined}: best average per NFE and model initialization. \textbf{Bold}: best average per model initialization.}
  \begin{tabular}{
  >{\scriptsize}c
  >{\scriptsize}l
  >{\scriptsize}c
  *{5}{>{\scriptsize}c}
  >{\scriptsize}c
  !{\color{black!20}\vrule}
  *{5}{>{\scriptsize}c}
  >{\scriptsize}c}
  \toprule
  & &
  & \multicolumn{6}{c}{\textbf{FID} $\downarrow$}
  & \multicolumn{6}{c}{\textbf{FDD} $\downarrow$} \\
  \cmidrule(lr){4-9}\cmidrule(lr){10-15}
  \textbf{ImageNet Init.} & \textbf{Method} & \textbf{NFE}
  & ArtB. & Calt. & CUB & Food & Cars & \textbf{Avg.}
  & ArtB. & Calt. & CUB & Food & Cars & \textbf{Avg.} \\
  \midrule

\multirow{7}{*}{
  \begin{tabular}{@{}c@{}}
  \textbf{iMF-XL/2}\\[-1pt]
  \cite{geng2026improved}
  \end{tabular}
}
& \textcolor{afmblue}{AFM} & $1$
  & ${\color{gray!90}82.14}$ & ${\color{gray!90}71.62}$ & ${\color{gray!90}23.76}$ & ${\color{gray!90}30.06}$ &
  ${\color{gray!90}25.27}$ & $46.57$
  & ${\color{gray!90}814.40}$ & ${\color{gray!90}1133.85}$ &${\color{gray!90}773.17}$ & ${\color{gray!90}817.06}$ &
  ${\color{gray!90}648.57}$ & $837.41$ \\
    & \textcolor{afmblue}{CAFM} & $1$ 
    & ${\color{gray!90}408.44}$ & ${\color{gray!90}304.60}$ & ${\color{gray!90}259.84}$ & ${\color{gray!90}351.87}$ & ${\color{gray!90}313.74}$ & $327.70$
    & ${\color{gray!90}3011.00}$ & ${\color{gray!90}2837.60}$ & ${\color{gray!90}3500.20}$ & ${\color{gray!90}2962.70}$ & ${\color{gray!90}3371.30}$ & $3136.56$ \\
  & \textcolor{afmblue}{\texttt{CAMF}} & $1$
  & ${\color{gray!90}31.75}$ & ${\color{gray!90}91.44}$ & ${\color{gray!90}32.00}$ & ${\color{gray!90}33.20}$ &
  ${\color{gray!90}59.83}$ & $49.64$
  & ${\color{gray!90}359.07}$ & ${\color{gray!90}1230.20}$ & ${\color{gray!90}775.99}$ & ${\color{gray!90}706.52}$ &
  ${\color{gray!90}826.35}$ & $779.63$ \\[-1.5pt]
  \arrayrulecolor{black!10}\cmidrule(lr){2-15}\arrayrulecolor{black}
  & \texttt{MF-T} & $1$
  & ${\color{gray!90}11.21}$ & ${\color{gray!90}34.92}$ & ${\color{gray!90}11.74}$ & ${\color{gray!90}12.20}$ & ${\color{gray!90}15.41}$ & $17.10$
  & ${\color{gray!90}268.85}$ & ${\color{gray!90}611.91}$ & ${\color{gray!90}423.50}$ & ${\color{gray!90}521.62}$ & ${\color{gray!90}403.39}$ & $445.85$\\

 & \texttt{MF-T}~+~AFM & $1$
  & ${\color{gray!90}26.89}$ & ${\color{gray!90}29.08}$ & ${\color{gray!90}6.03}$ & ${\color{gray!90}16.61}$ &
  ${\color{gray!90}8.51}$ & $17.42$
  & ${\color{gray!90}361.46}$ & ${\color{gray!90}513.52}$ & ${\color{gray!90}192.51}$ & ${\color{gray!90}452.02}$ &
  ${\color{gray!90}267.58}$ & $357.42$ \\
 & \texttt{MF-T}~+~CAFM & $1$ 
    & ${\color{gray!90}431.69}$ & ${\color{gray!90}271.25}$ & ${\color{gray!90}255.90}$ & ${\color{gray!90}337.70}$ & ${\color{gray!90}321.69}$ & $323.64$
    & ${\color{gray!90}2957.65}$ & ${\color{gray!90}2711.24}$ & ${\color{gray!90}3481.96}$ & ${\color{gray!90}3113.09}$ & ${\color{gray!90}3302.20}$ & $3113.23$ \\
  \rowcolor{gray!7}\cellcolor{white}{}& \texttt{MF-T}~+~\texttt{CAMF} & $1$
  & ${\color{gray!90}10.12}$ & ${\color{gray!90}29.07}$ & ${\color{gray!90}5.40}$ & ${\color{gray!90}8.87}$ &
  ${\color{gray!90}15.51}$ & $\underline{\mathbf{13.79}}$
  & ${\color{gray!90}190.04}$ & ${\color{gray!90}502.99}$ & ${\color{gray!90}179.49}$ & ${\color{gray!90}333.72}$ &
  ${\color{gray!90}352.35}$ & $\underline{\mathbf{311.72}}$ \\

  \hline

\multirow{5}{*}{
  \begin{tabular}{@{}c@{}}
  \textbf{SiT-XL/2}\\[-1pt]
  {\scriptsize \cite{ma2024sit}}
  \end{tabular}}  
  & FT\textsuperscript{\dag} & $1\xtwo$
  & ${\color{gray!90}459.70}$ & ${\color{gray!90}279.79}$ & ${\color{gray!90}301.13}$ & ${\color{gray!90}275.52}$ & ${\color{gray!90}327.76}$ & $328.78$
  & ${\color{gray!90}3263.97}$ & ${\color{gray!90}2918.49}$ & ${\color{gray!90}3468.16}$ & ${\color{gray!90}2900.06}$ & ${\color{gray!90}3589.74}$ & $3228.04$ \\
  & \texttt{MF-T} & $1$
  & ${\color{gray!90}21.79}$ & ${\color{gray!90}55.21}$ & ${\color{gray!90}13.72}$ &${\color{gray!90}19.93}$ & ${\color{gray!90}14.47}$ & $25.02$
  & ${\color{gray!90}424.89}$ & ${\color{gray!90}878.25}$ & ${\color{gray!90}503.24}$ & ${\color{gray!90}868.71}$ & ${\color{gray!90}476.18}$ & $630.25$\\
   & \texttt{MF-T}~+~AFM & $1$ 
    & ${\color{gray!90}39.62}$ & ${\color{gray!90}63.37}$ & ${\color{gray!90}20.90}$ & ${\color{gray!90}22.32}$ & ${\color{gray!90}21.66}$ & $33.58$
    & ${\color{gray!90}518.53}$ & ${\color{gray!90}933.42}$ & ${\color{gray!90}546.50}$ & ${\color{gray!90}714.41}$ & ${\color{gray!90}558.69}$ & $654.31$ \\
   & \texttt{MF-T}~+~CAFM & $1$ 
    & ${\color{gray!90}424.09}$ & ${\color{gray!90}269.55}$ & ${\color{gray!90}272.09}$ & ${\color{gray!90}363.46}$ & ${\color{gray!90}328.64}$ & $331.57$
    & ${\color{gray!90}2981.39}$ & ${\color{gray!90}2577.98}$ & ${\color{gray!90}3212.34}$ & ${\color{gray!90}2812.27}$ & ${\color{gray!90}3374.63}$ & $2991.72$ \\
  \rowcolor{gray!7}\cellcolor{white}{}&  \texttt{MF-T}~+~\texttt{CAMF} & $1$
  & ${\color{gray!90}18.35}$ & ${\color{gray!90}46.53}$ & ${\color{gray!90}8.68}$ &${\color{gray!90}15.10}$ & ${\color{gray!90}9.60}$ & $\underline{\mathbf{19.65}}$
  & ${\color{gray!90}375.82}$ & ${\color{gray!90}738.91}$ & ${\color{gray!90}361.83}$ & ${\color{gray!90}665.75}$ & ${\color{gray!90}343.27}$ & $\underline{\mathbf{497.12}}$
  \\
  \bottomrule
  \end{tabular}
    \label{tab:onestep}
    \begin{minipage}{\textwidth}
\raggedright
\scriptsize
\textsuperscript{\dag} At the corresponding NFE settings, FT produces samples that are perceptually indistinguishable from noise.
\end{minipage} 
\vspace{-10pt}
  \end{table*}

  \begin{table*}[t!]
    \centering
    \scriptsize
    \setlength{\tabcolsep}{3pt}
    \renewcommand{\arraystretch}{1.2}
    \caption{\textbf{Inception Score (IS) comparison}. \textcolor{afmblue}{AFM}, \textcolor{afmblue}{CAFM} and \textcolor{afmblue}{\texttt{CAMF}} as stand-alone methods are only tested on iMF-XL/2 where the initial model is MF-based (with \texttt{MF-T}~$\equiv$~FT on iMF). \underline{Underlined}: best average per NFE and model initialization. \textbf{Bold}: best average per model initialization.}

  \begin{tabular}{
  >{\scriptsize}c
  >{\scriptsize}l
  >{\scriptsize}c
  *{5}{>{\scriptsize}c}
  >{\scriptsize}c}
  \toprule
  & & & \multicolumn{6}{c}{\textbf{IS} $\uparrow$} \\
  \cmidrule(lr){4-9}
  \textbf{ImageNet Init.} & \textbf{Method} & \textbf{NFE}
  & ArtB. & Calt. & CUB & Food & Cars & \textbf{Avg.} \\
  \midrule
    \multirow{21}{*}{
  \begin{tabular}{@{}c@{}}
  \textbf{iMF-XL/2}\\[-1pt]
  \cite{geng2026improved}\\[-1pt]
  \end{tabular}}
    & \textcolor{afmblue}{AFM} & $250$
    & ${\color{gray!90}9.58}$ & ${\color{gray!90}21.78}$ & ${\color{gray!90}5.43}$ & ${\color{gray!90}5.29}$ & ${\color{gray!90}2.77}$ & $8.97$ \\  
    & \textcolor{afmblue}{CAFM} & $250$
    & ${\color{gray!90}8.65}$ & ${\color{gray!90}21.16}$ & ${\color{gray!90}6.12}$ & ${\color{gray!90}5.33}$ & ${\color{gray!90}3.29}$ & $8.91$ \\
    & \textcolor{afmblue}{\texttt{CAMF}} & $250$
  & ${\color{gray!90}9.43}$ & ${\color{gray!90}25.87}$ & ${\color{gray!90}6.79}$ & ${\color{gray!90}6.17}$ & ${\color{gray!90}3.66}$ &
  $10.38$ \\
  \arrayrulecolor{black!10}\cmidrule(lr){2-9}\arrayrulecolor{black}

    & \texttt{MF-T} & $250$
  & ${\color{gray!90}7.60}$ & ${\color{gray!90}33.21}$ & ${\color{gray!90}5.97}$ & ${\color{gray!90}5.27}$ & ${\color{gray!90}3.46}$ & $11.10$ \\  

    & \texttt{MF-T}~+~AFM & $250$
  & ${\color{gray!90}8.70}$ & ${\color{gray!90}37.32}$ & ${\color{gray!90}5.47}$ & ${\color{gray!90}5.53}$ & ${\color{gray!90}3.35}$ &
  $\underline{\mathbf{12.07}}$ \\  
    & \texttt{MF-T}~+~CAFM & $250$
  & ${\color{gray!90}7.41}$ & ${\color{gray!90}32.98}$ & ${\color{gray!90}6.36}$ & ${\color{gray!90}5.08}$ & ${\color{gray!90}3.30}$ & $11.03$ \\
    \rowcolor{gray!7}\cellcolor{white}{}& \texttt{MF-T}~+~\texttt{CAMF} & $250$
  & ${\color{gray!90}8.47}$ & ${\color{gray!90}36.39}$ & ${\color{gray!90}5.73}$ & ${\color{gray!90}5.56}$ & ${\color{gray!90}3.13}$ &
  $11.85$ \\  %
    \arrayrulecolor{black!10}\cmidrule(lr){2-9}\arrayrulecolor{black}
  
  & \textcolor{afmblue}{AFM} & $4$
    & ${\color{gray!90}7.67}$ & ${\color{gray!90}16.80}$ & ${\color{gray!90}4.90}$ & ${\color{gray!90}4.35}$ & ${\color{gray!90}3.07}$ & $7.36$ \\
  & \textcolor{afmblue}{CAFM} & $4$
    & ${\color{gray!90}5.13}$ & ${\color{gray!90}13.76}$ & ${\color{gray!90}7.23}$ & ${\color{gray!90}4.38}$ & ${\color{gray!90}3.60}$ & $6.82$ \\
  & \textcolor{afmblue}{\texttt{CAMF}} & $4$
  & ${\color{gray!90}7.94}$ & ${\color{gray!90}23.54}$ & ${\color{gray!90}6.16}$ & ${\color{gray!90}5.28}$ &
  ${\color{gray!90}3.66}$ & $9.32$  \\
  \arrayrulecolor{black!10}\cmidrule(lr){2-9}\arrayrulecolor{black}
  
  &  \texttt{MF-T} & $4$
  & ${\color{gray!90}7.36}$ & ${\color{gray!90}31.46}$ & ${\color{gray!90}6.08}$ & ${\color{gray!90}6.37}$ & ${\color{gray!90}3.41}$ & $10.94$ \\

    &   \texttt{MF-T}~+~AFM & $4$
  & ${\color{gray!90}7.62}$ & ${\color{gray!90}36.69}$ & ${\color{gray!90}5.67}$ & ${\color{gray!90}5.47}$ &
  ${\color{gray!90}3.36}$ & $11.76$ \\

  & \texttt{MF-T}~+~CAFM & $4$
  & ${\color{gray!90}4.56}$ & ${\color{gray!90}24.16}$ & ${\color{gray!90}9.34}$ & ${\color{gray!90}4.48}$ & ${\color{gray!90}3.67}$ & $9.24$ \\

   \rowcolor{gray!7}\cellcolor{white}{}&  \texttt{MF-T}~+~\texttt{CAMF} & $4$
  & ${\color{gray!90}8.03}$ & ${\color{gray!90}36.16}$ & ${\color{gray!90}5.69}$ & ${\color{gray!90}5.76}$ &
  ${\color{gray!90}3.23}$ & $\underline{11.77}$ \\
  \arrayrulecolor{black!10}\cmidrule(lr){2-9}\arrayrulecolor{black}

  & \textcolor{afmblue}{AFM} & $1$
  & ${\color{gray!90}6.27}$ & ${\color{gray!90}14.49}$ & ${\color{gray!90}5.41}$ & ${\color{gray!90}5.66}$ &
  ${\color{gray!90}3.5}$ & $7.07$ \\

  & \textcolor{afmblue}{CAFM} & $1$
    & ${\color{gray!90}1.31}$ & ${\color{gray!90}1.79}$ & ${\color{gray!90}1.59}$ & ${\color{gray!90}1.70}$ & ${\color{gray!90}1.23}$ & $1.52$ \\  
  
  & \textcolor{afmblue}{\texttt{CAMF}} & $1$
  & ${\color{gray!90}9.41}$ & ${\color{gray!90}14.44}$ & ${\color{gray!90}6.62}$ & ${\color{gray!90}8.35}$ &
  ${\color{gray!90}4.91}$ & $8.75$ \\
  \arrayrulecolor{black!10}\cmidrule(lr){2-9}\arrayrulecolor{black}
  
  &  \texttt{MF-T} & $1$
  & ${\color{gray!90}7.70}$ & ${\color{gray!90}28.43}$ & ${\color{gray!90}6.37}$ & ${\color{gray!90}6.31}$ & ${\color{gray!90}3.43}$ & $10.45$ \\

  &  \texttt{MF-T}~+~AFM & $1$
  & ${\color{gray!90}6.77}$ & ${\color{gray!90}32.34}$ & ${\color{gray!90}6.02}$ & ${\color{gray!90}5.39}$ &
  ${\color{gray!90}3.19}$ & $10.74$ \\

  & \texttt{MF-T}~+~CAFM & $1$
  & ${\color{gray!90}1.28}$ & ${\color{gray!90}2.01}$ & ${\color{gray!90}1.71}$ & ${\color{gray!90}1.75}$ & ${\color{gray!90}1.33}$ & $1.62$ \\  
  
  \rowcolor{gray!7}\cellcolor{white}{}&  \texttt{MF-T}~+~\texttt{CAMF} & $1$
  & ${\color{gray!90}7.69}$ & ${\color{gray!90}32.45}$ & ${\color{gray!90}6.04}$ & ${\color{gray!90}5.70}$ &
  ${\color{gray!90}3.05}$ & $\underline{10.99}$ \\
\hline

  \multirow{15}{*}{
  \begin{tabular}{@{}c@{}}
  \textbf{SiT-XL/2}\\[-1pt]
  {\scriptsize \cite{ma2024sit}}
  \end{tabular}} 
  & FT & $250\xtwo$
  & ${\color{gray!90}6.94}$ & ${\color{gray!90}35.23}$ & ${\color{gray!90}6.03}$ & ${\color{gray!90}5.45}$ & ${\color{gray!90}3.39}$ & $11.41$ \\
  & \texttt{MF-T} & $250$
  & ${\color{gray!90}7.70}$ & ${\color{gray!90}33.35}$ & ${\color{gray!90}6.10}$ & ${\color{gray!90}6.13}$ & ${\color{gray!90}3.54}$ &
  $11.37$ \\
  & \texttt{MF-T}~+~AFM & $250$
  & ${\color{gray!90}7.87}$ & ${\color{gray!90}32.83}$ & ${\color{gray!90}5.45}$ & ${\color{gray!90}5.98}$ & ${\color{gray!90}3.29}$ &
  $11.09$ \\  
    & \texttt{MF-T}~+~CAFM & $250$
  & ${\color{gray!90}6.81}$ & ${\color{gray!90}27.45}$ & ${\color{gray!90}6.57}$ & ${\color{gray!90}5.13}$ & ${\color{gray!90}3.41}$ & $9.87$ \\  
    \rowcolor{gray!7}\cellcolor{white}{}&  \texttt{MF-T}~+~\texttt{CAMF} & $250$
  & ${\color{gray!90}7.71}$ & ${\color{gray!90}34.14}$ & ${\color{gray!90}5.82}$ & ${\color{gray!90}6.10}$ & ${\color{gray!90}3.53}$ &
  $\underline{11.46}$ \\  
  
  \arrayrulecolor{black!10}\cmidrule(lr){2-9}\arrayrulecolor{black}
  & FT & $4\xtwo$
  & ${\color{gray!90}5.45}$ & ${\color{gray!90}31.99}$ & ${\color{gray!90}6.96}$ & ${\color{gray!90}5.35}$ & ${\color{gray!90}3.52}$ & $10.65$ \\

  &  \texttt{MF-T} & $4$
  & ${\color{gray!90}7.56}$ & ${\color{gray!90}33.21}$ & ${\color{gray!90}6.14}$ & ${\color{gray!90}6.03}$ & ${\color{gray!90}3.50}$ & $11.29$ \\
  
  & \texttt{MF-T}~+~AFM & $4$
  & ${\color{gray!90}7.68}$ & ${\color{gray!90}33.59}$ & ${\color{gray!90}5.60}$ & ${\color{gray!90}5.85}$ & ${\color{gray!90}3.29}$ & $11.20$ \\
      & \texttt{MF-T}~+~CAFM & $4$
  & ${\color{gray!90}5.16}$ & ${\color{gray!90}20.98}$ & ${\color{gray!90}9.28}$ & ${\color{gray!90}4.88}$ & ${\color{gray!90}3.48}$ & $8.76$ \\   
  
    \rowcolor{gray!7}\cellcolor{white}{}& \texttt{MF-T}~+~\texttt{CAMF} & $4$
  & ${\color{gray!90}7.59}$ & ${\color{gray!90}34.69}$ & ${\color{gray!90}5.83}$ & ${\color{gray!90}5.97}$ & ${\color{gray!90}3.45}$ & $\underline{\mathbf{11.51}}$
  \\
  \arrayrulecolor{black!10}\cmidrule(lr){2-9}\arrayrulecolor{black}
  
  & FT & $1\xtwo$
  & ${\color{gray!90}1.82}$ & ${\color{gray!90}2.76}$ & ${\color{gray!90}2.52}$ & ${\color{gray!90}1.77}$ & ${\color{gray!90}1.97}$ & $2.17$ \\
  & \texttt{MF-T} & $1$ & ${\color{gray!90}6.34}$ & ${\color{gray!90}23.64}$ & ${\color{gray!90}6.42}$ & ${\color{gray!90}5.80}$ & ${\color{gray!90}3.15}$ & $9.07$ \\

  & \texttt{MF-T}~+~AFM & $1$
  & ${\color{gray!90}6.40}$ & ${\color{gray!90}20.33}$ & ${\color{gray!90}5.49}$ & ${\color{gray!90}5.43}$ & ${\color{gray!90}2.77}$ & $8.08$ \\
  & \texttt{MF-T}~+~CAFM & $1$
  & ${\color{gray!90}1.41}$ & ${\color{gray!90}2.22}$ & ${\color{gray!90}1.93}$ & ${\color{gray!90}2.01}$ & ${\color{gray!90}1.45}$ & $1.80$ \\   
  \rowcolor{gray!7}\cellcolor{white}{}& \texttt{MF-T}~+~\texttt{CAMF} & $1$  & ${\color{gray!90}6.10}$ & ${\color{gray!90}25.28}$ & ${\color{gray!90}5.91}$ & ${\color{gray!90}5.52}$ & ${\color{gray!90}3.14}$ & $\underline{9.19}$ \\
  \hline

  \cellcolor{white}
   \multirow{4}{*}{%
      \begin{tabular}{@{}c@{}}
        \textbf{DiT-XL/2}\\[-1pt]
        {\scriptsize \cite{peebles2023scalable}}
      \end{tabular}
    }
    & FT & $250\xtwo$
  & ${\color{gray!90}7.15}$ & ${\color{gray!90}34.91}$ & ${\color{gray!90}6.52}$ & ${\color{gray!90}5.96}$ & ${\color{gray!90}3.45}$ & $11.60$ \\
  \arrayrulecolor{black!10}\cmidrule(lr){2-9}\arrayrulecolor{black}
  & FT & $8\xtwo$
  & ${\color{gray!90}5.34}$ & ${\color{gray!90}23.27}$ & ${\color{gray!90}6.45}$ & ${\color{gray!90}4.28}$ & ${\color{gray!90}4.80}$ & $8.83$ \\
  & \texttt{MF-T} & $8$
    & ${\color{gray!90}8.55}$ & ${\color{gray!90}33.02}$ & ${\color{gray!90}6.07}$ & ${\color{gray!90}7.19}$ & ${\color{gray!90}3.40}$ & $\underline{\mathbf{11.65}}$ \\
  
  \rowcolor{gray!7}\cellcolor{white}{}& \texttt{MF-T}~+~\texttt{CAMF} & $8$
    & ${\color{gray!90}7.42}$ & ${\color{gray!90}33.54}$ & ${\color{gray!90}5.84}$ & ${\color{gray!90}5.80}$ & ${\color{gray!90}3.42}$ & $11.20$ \\
  \arrayrulecolor{black!10}\cmidrule(lr){2-9}\arrayrulecolor{black}
  \hline

  \multirow{4}{*}{
  \begin{tabular}{@{}c@{}}
  \textbf{JiT-H/16}\\[-1pt]
  {\scriptsize \cite{li2026back}}
  \end{tabular}}
  & FT & $50\xtwo$
  & ${\color{gray!90}7.43}$ & ${\color{gray!90}28.18}$ & ${\color{gray!90}6.10}$ & ${\color{gray!90}5.57}$ & ${\color{gray!90}3.35}$ & $\mathbf{10.13}$ \\
  \arrayrulecolor{black!10}\cmidrule(lr){2-9}\arrayrulecolor{black}
  & FT & $4\xtwo$
  & ${\color{gray!90}6.75}$ & ${\color{gray!90}18.47}$ & ${\color{gray!90}5.80}$ & ${\color{gray!90}4.51}$ & ${\color{gray!90}3.84}$ & $7.88$ \\
  & \texttt{MF-T} & $4$
  & ${\color{gray!90}6.37}$ & ${\color{gray!90}18.43}$ & ${\color{gray!90}5.46}$ & ${\color{gray!90}5.36}$ & ${\color{gray!90}3.14}$ & $7.75$ \\
\rowcolor{gray!7}\cellcolor{white}{}& \texttt{MF-T}~+~\texttt{CAMF} & $4$
    & ${\color{gray!90}6.52}$ & ${\color{gray!90}20.31}$ & ${\color{gray!90}5.70}$ & ${\color{gray!90}4.97}$ & ${\color{gray!90}3.11}$ & $\underline{8.12}$ \\
  \bottomrule
  \end{tabular}
    \label{tab:main_is}
  \end{table*}

\begin{table*}[t]
  \centering
  \scriptsize
  \setlength{\tabcolsep}{4pt}
  \renewcommand{\arraystretch}{1.2}
 \caption{\textbf{Precision-Recall, and Density-Coverage}. $4$ NFE. evaluated in Inception and DINOv2 feature spaces with $k=3$~\cite{kynkaanniemi2019improved,naeem2020reliable}. \textbf{Bold:} best average.}
  \label{tab:precision_recall}
  \begin{tabular}{
  >{\scriptsize}c
  >{\scriptsize}c
  >{\scriptsize}l
  *{4}{>{\scriptsize}c}
  !{\color{black!20}\vrule}
  *{4}{>{\scriptsize}c}}
  \toprule
  & & & \multicolumn{4}{c}{\textbf{Inception}} & \multicolumn{4}{c}{\textbf{DINOv2}} \\
  \cmidrule(lr){4-7}\cmidrule(lr){8-11}
  \textbf{ImageNet Init.} & \textbf{Data} & \textbf{Method} & Precision$\uparrow$ & Recall$\uparrow$ & Density$\uparrow$ &Coverage$\uparrow$ & Precision$\uparrow$ & Recall$\uparrow$ & Density$\uparrow$ &Coverage$\uparrow$ \\
  \midrule
   & & \texttt{MF-T} & ${\color{gray!90}0.542}$ & ${\color{gray!90}0.730}$ & ${\color{gray!90}0.620}$ & ${\color{gray!90}0.706}$ & ${\color{gray!90}0.214}$ & ${\color{gray!90}0.772}$ & ${\color{gray!90}0.141}$ & ${\color{gray!90}0.249}$ \\
   & \multirow{-2}{*}{CUB} & \texttt{MF-T}~+~\texttt{CAMF} & ${\color{gray!90}0.662}$ & ${\color{gray!90}0.684}$ & ${\color{gray!90}0.856}$ & ${\color{gray!90}0.803}$ & ${\color{gray!90}0.360}$ & ${\color{gray!90}0.703}$ & ${\color{gray!90}0.278}$ & ${\color{gray!90}0.406}$ \\
  \arrayrulecolor{black!10}\cmidrule(lr){2-11}\arrayrulecolor{black}
   & & \texttt{MF-T} & ${\color{gray!90}0.465}$ & ${\color{gray!90}0.568}$ & ${\color{gray!90}0.463}$ & ${\color{gray!90}0.564}$ & ${\color{gray!90}0.146}$ & ${\color{gray!90}0.736}$ & ${\color{gray!90}0.073}$ & ${\color{gray!90}0.130}$ \\
  \multirow{-4}{*}{\begin{tabular}{@{}c@{}}\textbf{iMF-XL/2}\\[-1pt]{\scriptsize \cite{geng2026improved}}\end{tabular}} & \multirow{-2}{*}{Cars} & \texttt{MF-T}~+~\texttt{CAMF} & ${\color{gray!90}0.658}$ & ${\color{gray!90}0.593}$ & ${\color{gray!90}0.899}$ & ${\color{gray!90}0.788}$ & ${\color{gray!90}0.431}$ & ${\color{gray!90}0.728}$ & ${\color{gray!90}0.300}$ & ${\color{gray!90}0.370}$ \\
  \hline
   & & \texttt{MF-T} & ${\color{gray!90}0.554}$ & ${\color{gray!90}0.699}$ & ${\color{gray!90}0.573}$ & ${\color{gray!90}0.641}$ & ${\color{gray!90}0.138}$ & ${\color{gray!90}0.678}$ & ${\color{gray!90}0.071}$ & ${\color{gray!90}0.128}$ \\
   & \multirow{-2}{*}{CUB} & \texttt{MF-T}~+~\texttt{CAMF} & ${\color{gray!90}0.623}$ & ${\color{gray!90}0.649}$ & ${\color{gray!90}0.726}$ & ${\color{gray!90}0.731}$ & ${\color{gray!90}0.264}$ & ${\color{gray!90}0.608}$ & ${\color{gray!90}0.164}$ & ${\color{gray!90}0.255}$ \\
  \arrayrulecolor{black!10}\cmidrule(lr){2-11}\arrayrulecolor{black}
   & & \texttt{MF-T} & ${\color{gray!90}0.573}$ & ${\color{gray!90}0.630}$ & ${\color{gray!90}0.663}$ & ${\color{gray!90}0.701}$ & ${\color{gray!90}0.315}$ & ${\color{gray!90}0.724}$ & ${\color{gray!90}0.184}$ & ${\color{gray!90}0.289}$ \\
  \multirow{-4}{*}{\begin{tabular}{@{}c@{}}\textbf{SiT-XL/2}\\[-1pt]{\scriptsize \cite{ma2024sit}}\end{tabular}} & \multirow{-2}{*}{Cars} & \texttt{MF-T}~+~\texttt{CAMF} & ${\color{gray!90}0.668}$ & ${\color{gray!90}0.610}$ & ${\color{gray!90}0.884}$ & ${\color{gray!90}0.792}$ & ${\color{gray!90}0.478}$ & ${\color{gray!90}0.715}$ & ${\color{gray!90}0.326}$ & ${\color{gray!90}0.445}$ \\
  \hline
   & & \texttt{MF-T} & $0.534$ & $\mathbf{0.657}$ & $0.580$ & $0.653$ & $0.203$ & $\mathbf{0.728}$ & $0.117$ & $0.199$ \\
  \rowcolor{gray!7}\cellcolor{white}\multirow{-2}{*}{\textbf{Avg.}} & \cellcolor{white}{} & \texttt{MF-T}~+~\texttt{CAMF} & $\mathbf{0.653}$ & $0.634$ & $\mathbf{0.841}$ & $\mathbf{0.778}$ & $\mathbf{0.383}$ & $0.688$ & $\mathbf{0.267}$ & $\mathbf{0.369}$ \\
  \bottomrule
  \end{tabular}
\end{table*}

\begin{table*}[t]
  \centering
  \scriptsize
  \setlength{\tabcolsep}{2.6pt}
  \renewcommand{\arraystretch}{1.2}
  \caption{\textbf{2-stage vs. 1-stage \texttt{MF-T} across different NFEs.} Ratio: 2-stage/1-stage. \textbf{Bold}: best aggregated average per group, and Ratio $>1$.}
  \label{tab:one_vs_two_nfe}
  \resizebox{\textwidth}{!}{%
  \begin{tabular}{c l c *{5}{c} c !{\color{black!20}\vrule} *{5}{c} c !{\color{black!20}\vrule} c c}
  \toprule
  & & & \multicolumn{6}{c}{\textbf{FID} $\downarrow$} & \multicolumn{6}{c}{\textbf{FDD} $\downarrow$} & \multicolumn{2}{c}{\textbf{Ratio$\uparrow$}} \\
  \cmidrule(lr){4-9}\cmidrule(lr){10-15}\cmidrule(lr){16-17}
  \textbf{ImageNet Init.} & \textbf{Method} & \textbf{NFE} & ArtB. & Calt. & CUB & Food & Cars & \textbf{Avg.} & ArtB. & Calt. & CUB & Food & Cars & \textbf{Avg.} & FID & FDD \\
  \midrule
  \multirow{8}{*}{\begin{tabular}{@{}c@{}}\textbf{SiT-XL/2}\\[-1pt]{\scriptsize \cite{ma2024sit}}\end{tabular}} & \texttt{MF-T} (1-stage) & \multirow{2}{*}{$8$} & ${\color{gray!90}14.63}$ & ${\color{gray!90}25.46}$ & ${\color{gray!90}4.82}$ & ${\color{gray!90}8.70}$ & ${\color{gray!90}4.90}$ & $11.70$ & ${\color{gray!90}268.58}$ & ${\color{gray!90}462.12}$ & ${\color{gray!90}189.68}$ & ${\color{gray!90}416.50}$ & ${\color{gray!90}175.28}$ & $302.43$ & \multirow{2}{*}{$0.88$} & \multirow{2}{*}{$0.79$} \\
   & \texttt{MF-T} (2-stage) &  & ${\color{gray!90}12.96}$ & ${\color{gray!90}24.24}$ & ${\color{gray!90}3.63}$ & ${\color{gray!90}6.06}$ & ${\color{gray!90}4.73}$ & $10.32$ & ${\color{gray!90}261.82}$ & ${\color{gray!90}398.62}$ & ${\color{gray!90}110.30}$ & ${\color{gray!90}241.80}$ & ${\color{gray!90}176.52}$ & $237.81$ &  &  \\
  \arrayrulecolor{black!10}\cmidrule(lr){2-17}\arrayrulecolor{black}
   & \texttt{MF-T} (1-stage) & \multirow{2}{*}{$4$} & ${\color{gray!90}14.56}$ & ${\color{gray!90}25.80}$ & ${\color{gray!90}5.25}$ & ${\color{gray!90}9.44}$ & ${\color{gray!90}4.77}$ & $11.96$ & ${\color{gray!90}274.67}$ & ${\color{gray!90}474.45}$ & ${\color{gray!90}205.34}$ & ${\color{gray!90}431.78}$ & ${\color{gray!90}178.80}$ & $313.01$ & \multirow{2}{*}{$0.90$} & \multirow{2}{*}{$0.80$} \\
   & \texttt{MF-T} (2-stage) &  & ${\color{gray!90}13.28}$ & ${\color{gray!90}25.00}$ & ${\color{gray!90}3.93}$ & ${\color{gray!90}6.48}$ & ${\color{gray!90}5.16}$ & $10.77$ & ${\color{gray!90}269.53}$ & ${\color{gray!90}414.55}$ & ${\color{gray!90}124.29}$ & ${\color{gray!90}248.52}$ & ${\color{gray!90}188.78}$ & $249.14$ &  &  \\
  \arrayrulecolor{black!10}\cmidrule(lr){2-17}\arrayrulecolor{black}
   & \texttt{MF-T} (1-stage) & \multirow{2}{*}{$2$} & ${\color{gray!90}15.66}$ & ${\color{gray!90}29.30}$ & ${\color{gray!90}6.67}$ & ${\color{gray!90}11.97}$ & ${\color{gray!90}5.68}$ & $13.86$ & ${\color{gray!90}302.33}$ & ${\color{gray!90}546.43}$ & ${\color{gray!90}258.08}$ & ${\color{gray!90}526.96}$ & ${\color{gray!90}226.76}$ & $372.11$ & \multirow{2}{*}{$0.93$} & \multirow{2}{*}{$0.84$} \\
   & \texttt{MF-T} (2-stage) &  & ${\color{gray!90}14.09}$ & ${\color{gray!90}28.23}$ & ${\color{gray!90}5.50}$ & ${\color{gray!90}9.33}$ & ${\color{gray!90}7.56}$ & $12.94$ & ${\color{gray!90}286.64}$ & ${\color{gray!90}472.00}$ & ${\color{gray!90}186.28}$ & ${\color{gray!90}348.53}$ & ${\color{gray!90}263.68}$ & $311.43$ &  &  \\
  \arrayrulecolor{black!10}\cmidrule(lr){2-17}\arrayrulecolor{black}
   & \texttt{MF-T} (1-stage) & \multirow{2}{*}{$1$} & ${\color{gray!90}21.79}$ & ${\color{gray!90}55.21}$ & ${\color{gray!90}13.72}$ & ${\color{gray!90}19.93}$ & ${\color{gray!90}14.47}$ & $25.02$ & ${\color{gray!90}424.89}$ & ${\color{gray!90}878.25}$ & ${\color{gray!90}503.24}$ & ${\color{gray!90}868.71}$ & ${\color{gray!90}476.18}$ & $630.25$ & \multirow{2}{*}{$\mathbf{1.29}$} & \multirow{2}{*}{$\mathbf{1.14}$} \\
   & \texttt{MF-T} (2-stage) &  & ${\color{gray!90}18.70}$ & ${\color{gray!90}61.99}$ & ${\color{gray!90}17.96}$ & ${\color{gray!90}31.89}$ & ${\color{gray!90}30.77}$ & $32.26$ & ${\color{gray!90}380.92}$ & ${\color{gray!90}898.96}$ & ${\color{gray!90}541.72}$ & ${\color{gray!90}1046.22}$ & ${\color{gray!90}735.12}$ & $720.59$ &  &  \\
  \hline
  \multirow{6}{*}{\begin{tabular}{@{}c@{}}\textbf{DiT-XL/2}\\[-1pt]{\scriptsize \cite{peebles2023scalable}}\end{tabular}} & \texttt{MF-T} (1-stage) & \multirow{2}{*}{$8$} & ${\color{gray!90}33.48}$ & ${\color{gray!90}32.17}$ & ${\color{gray!90}8.27}$ & ${\color{gray!90}33.44}$ & ${\color{gray!90}9.83}$ & $23.44$ & ${\color{gray!90}325.72}$ & ${\color{gray!90}475.90}$ & ${\color{gray!90}222.81}$ & ${\color{gray!90}614.30}$ & ${\color{gray!90}259.65}$ & $379.68$ & \multirow{2}{*}{$0.95$} & \multirow{2}{*}{$0.97$} \\
   & \texttt{MF-T} (2-stage) &  & ${\color{gray!90}35.40}$ & ${\color{gray!90}35.94}$ & ${\color{gray!90}6.85}$ & ${\color{gray!90}24.34}$ & ${\color{gray!90}8.64}$ & $22.23$ & ${\color{gray!90}468.08}$ & ${\color{gray!90}488.95}$ & ${\color{gray!90}186.47}$ & ${\color{gray!90}505.63}$ & ${\color{gray!90}201.49}$ & $370.13$ &  &  \\
  \arrayrulecolor{black!10}\cmidrule(lr){2-17}\arrayrulecolor{black}
   & \texttt{MF-T} (1-stage) & \multirow{2}{*}{$4$} & ${\color{gray!90}47.26}$ & ${\color{gray!90}48.64}$ & ${\color{gray!90}22.38}$ & ${\color{gray!90}57.86}$ & ${\color{gray!90}40.68}$ & $43.36$ & ${\color{gray!90}608.87}$ & ${\color{gray!90}609.83}$ & ${\color{gray!90}512.51}$ & ${\color{gray!90}1026.30}$ & ${\color{gray!90}629.40}$ & $677.38$ & \multirow{2}{*}{$\mathbf{1.35}$} & \multirow{2}{*}{$\mathbf{1.09}$} \\
   & \texttt{MF-T} (2-stage) &  & ${\color{gray!90}87.18}$ & ${\color{gray!90}80.98}$ & ${\color{gray!90}37.40}$ & ${\color{gray!90}53.64}$ & ${\color{gray!90}32.75}$ & $58.39$ & ${\color{gray!90}812.47}$ & ${\color{gray!90}818.88}$ & ${\color{gray!90}596.28}$ & ${\color{gray!90}919.40}$ & ${\color{gray!90}536.77}$ & $736.76$ &  &  \\
  \arrayrulecolor{black!10}\cmidrule(lr){2-17}\arrayrulecolor{black}
   & \texttt{MF-T} (1-stage) & \multirow{2}{*}{$2$} & ${\color{gray!90}156.91}$ & ${\color{gray!90}108.48}$ & ${\color{gray!90}107.04}$ & ${\color{gray!90}172.32}$ & ${\color{gray!90}199.28}$ & $148.81$ & ${\color{gray!90}1683.62}$ & ${\color{gray!90}1370.87}$ & ${\color{gray!90}1453.51}$ & ${\color{gray!90}2140.40}$ & ${\color{gray!90}2029.54}$ & $1735.59$ & \multirow{2}{*}{$\mathbf{1.65}$} & \multirow{2}{*}{$\mathbf{1.51}$} \\
   & \texttt{MF-T} (2-stage) &  & ${\color{gray!90}325.94}$ & ${\color{gray!90}239.73}$ & ${\color{gray!90}309.26}$ & ${\color{gray!90}136.75}$ & ${\color{gray!90}214.52}$ & $245.24$ & ${\color{gray!90}2937.09}$ & ${\color{gray!90}2728.36}$ & ${\color{gray!90}3297.82}$ & ${\color{gray!90}2103.31}$ & ${\color{gray!90}2039.06}$ & $2621.13$ &  &  \\
  \hline
  \multirow{6}{*}{\begin{tabular}{@{}c@{}}\textbf{JiT-H/16}\\[-1pt]{\scriptsize \cite{li2026back}}\end{tabular}} & \texttt{MF-T} (1-stage) & \multirow{2}{*}{$8$} & ${\color{gray!90}16.44}$ & ${\color{gray!90}45.74}$ & ${\color{gray!90}11.66}$ & ${\color{gray!90}20.60}$ & ${\color{gray!90}19.07}$ & $22.70$ & ${\color{gray!90}283.12}$ & ${\color{gray!90}724.43}$ & ${\color{gray!90}333.12}$ & ${\color{gray!90}691.47}$ & ${\color{gray!90}523.49}$ & $511.12$ & \multirow{2}{*}{$0.89$} & \multirow{2}{*}{$0.82$} \\
   & \texttt{MF-T} (2-stage) &  & ${\color{gray!90}17.05}$ & ${\color{gray!90}42.61}$ & ${\color{gray!90}9.30}$ & ${\color{gray!90}13.45}$ & ${\color{gray!90}18.40}$ & $20.16$ & ${\color{gray!90}281.28}$ & ${\color{gray!90}643.28}$ & ${\color{gray!90}259.23}$ & ${\color{gray!90}466.06}$ & ${\color{gray!90}452.04}$ & $420.38$ &  &  \\
  \arrayrulecolor{black!10}\cmidrule(lr){2-17}\arrayrulecolor{black}
   & \texttt{MF-T} (1-stage) & \multirow{2}{*}{$4$} & ${\color{gray!90}24.07}$ & ${\color{gray!90}60.98}$ & ${\color{gray!90}14.65}$ & ${\color{gray!90}24.80}$ & ${\color{gray!90}21.22}$ & $29.14$ & ${\color{gray!90}360.67}$ & ${\color{gray!90}913.56}$ & ${\color{gray!90}417.08}$ & ${\color{gray!90}768.58}$ & ${\color{gray!90}619.10}$ & $615.80$ & \multirow{2}{*}{$0.91$} & \multirow{2}{*}{$0.86$} \\
   & \texttt{MF-T} (2-stage) &  & ${\color{gray!90}24.93}$ & ${\color{gray!90}53.40}$ & ${\color{gray!90}12.43}$ & ${\color{gray!90}20.57}$ & ${\color{gray!90}21.60}$ & $26.59$ & ${\color{gray!90}345.38}$ & ${\color{gray!90}817.29}$ & ${\color{gray!90}330.56}$ & ${\color{gray!90}586.64}$ & ${\color{gray!90}563.40}$ & $528.65$ &  &  \\
  \arrayrulecolor{black!10}\cmidrule(lr){2-17}\arrayrulecolor{black}
   & \texttt{MF-T} (1-stage) & \multirow{2}{*}{$2$} & ${\color{gray!90}65.31}$ & ${\color{gray!90}99.45}$ & ${\color{gray!90}30.15}$ & ${\color{gray!90}75.46}$ & ${\color{gray!90}33.75}$ & $60.82$ & ${\color{gray!90}657.93}$ & ${\color{gray!90}1317.43}$ & ${\color{gray!90}732.27}$ & ${\color{gray!90}1345.74}$ & ${\color{gray!90}868.36}$ & $984.35$ & \multirow{2}{*}{$0.96$} & \multirow{2}{*}{$0.93$} \\
   & \texttt{MF-T} (2-stage) &  & ${\color{gray!90}65.48}$ & ${\color{gray!90}84.02}$ & ${\color{gray!90}23.70}$ & ${\color{gray!90}78.76}$ & ${\color{gray!90}38.71}$ & $58.13$ & ${\color{gray!90}620.25}$ & ${\color{gray!90}1163.13}$ & ${\color{gray!90}637.81}$ & ${\color{gray!90}1247.32}$ & ${\color{gray!90}888.40}$ & $911.38$ &  &  \\
  \hline
    
    \multirow{6}{*}{\textbf{Avg.}}
    & \cellcolor{gray!7}\texttt{MF-T} (1-stage)
    & \multirow{2}{*}{$8$}
    & \cellcolor{gray!7}$21.52$
    & \cellcolor{gray!7}$34.46$
    & \cellcolor{gray!7}$8.25$
    & \cellcolor{gray!7}$20.91$
    & \cellcolor{gray!7}$11.27$
    & \cellcolor{gray!7}$19.28$
    & \cellcolor{gray!7}$292.47$
    & \cellcolor{gray!7}$554.15$
    & \cellcolor{gray!7}$248.54$
    & \cellcolor{gray!7}$574.09$
    & \cellcolor{gray!7}$319.47$
    & \cellcolor{gray!7}$397.74$
    & \multirow{2}{*}{$0.91$}
    & \multirow{2}{*}{$0.86$} \\
    
    & \texttt{MF-T} (2-stage) &
    & $21.80$ & $34.26$ & $6.59$
    & $14.62$ & $10.59$ & $\mathbf{17.57}$
    & $337.06$ & $510.28$ & $185.33$
    & $404.50$ & $276.68$ & $\mathbf{342.77}$
    & & \\
    \arrayrulecolor{black!10}\cmidrule(lr){2-17}\arrayrulecolor{black}
    
    & \cellcolor{gray!7}\texttt{MF-T} (1-stage)
    & \multirow{2}{*}{$4$}
    & \cellcolor{gray!7}$28.63$
    & \cellcolor{gray!7}$45.14$
    & \cellcolor{gray!7}$14.09$
    & \cellcolor{gray!7}$30.70$
    & \cellcolor{gray!7}$22.22$
    & \cellcolor{gray!7}$\mathbf{28.15}$
    & \cellcolor{gray!7}$414.74$
    & \cellcolor{gray!7}$665.95$
    & \cellcolor{gray!7}$378.31$
    & \cellcolor{gray!7}$742.22$
    & \cellcolor{gray!7}$475.77$
    & \cellcolor{gray!7}$535.40$
    & \multirow{2}{*}{$\mathbf{1.13}$}
    & \multirow{2}{*}{$0.94$} \\
    
    & \texttt{MF-T} (2-stage) &
    & $41.80$ & $53.13$ & $17.92$
    & $26.90$ & $19.84$ & $31.92$
    & $475.79$ & $683.57$ & $350.38$
    & $584.85$ & $429.65$ & $\mathbf{504.85}$
    & & \\
    \arrayrulecolor{black!10}\cmidrule(lr){2-17}\arrayrulecolor{black}
    
    & \cellcolor{gray!7}\texttt{MF-T} (1-stage)
    & \multirow{2}{*}{$2$}
    & \cellcolor{gray!7}$79.29$
    & \cellcolor{gray!7}$79.08$
    & \cellcolor{gray!7}$47.95$
    & \cellcolor{gray!7}$86.58$
    & \cellcolor{gray!7}$79.57$
    & \cellcolor{gray!7}$\mathbf{74.50}$
    & \cellcolor{gray!7}$881.29$
    & \cellcolor{gray!7}$1078.24$
    & \cellcolor{gray!7}$814.62$
    & \cellcolor{gray!7}$1337.70$
    & \cellcolor{gray!7}$1041.55$
    & \cellcolor{gray!7}$\mathbf{1030.68}$
    & \multirow{2}{*}{$\mathbf{1.42}$}
    & \multirow{2}{*}{$\mathbf{1.24}$} \\
    
    & \texttt{MF-T} (2-stage) &
    & $135.17$ & $117.33$ & $112.82$
    & $74.95$ & $86.93$ & $105.44$
    & $1281.33$ & $1454.50$ & $1373.97$
    & $1233.05$ & $1063.71$ & $1281.31$
    & & \\
    \bottomrule

  \end{tabular}%
  }
\end{table*}

\begin{table*}[t]
  \centering
  \scriptsize
  \setlength{\tabcolsep}{4pt}
  \renewcommand{\arraystretch}{1.2}
  \caption{\textbf{Generation quality comparison on final checkpoints}. No checkpoint selection on the evaluation metric. CAFM is shown at $4$ NFE only. \underline{Underlined}: best average per NFE and model initialization. \textbf{Bold}: best average per model initialization.}
  \label{tab:final_ckpt}\label{tab:posttrain_final}
  \begin{tabular}{c l c c c c !{\color{black!20}\vrule} c c c}
  \toprule
  & & & \multicolumn{3}{c}{\textbf{FID} $\downarrow$} & \multicolumn{3}{c}{\textbf{FDD} $\downarrow$} \\
  \cmidrule(lr){4-6}\cmidrule(lr){7-9}
  \textbf{ImageNet Init.} & \textbf{Method} & \textbf{NFE} & CUB & Cars & \textbf{Avg.} & CUB & Cars & \textbf{Avg.} \\
  \midrule
   & \texttt{MF-T} & $250$ & ${\color{gray!90}6.94}$ & ${\color{gray!90}15.19}$ & $11.07$ & ${\color{gray!90}183.65}$ & ${\color{gray!90}306.26}$ & $244.96$ \\
   & \texttt{MF-T}~+~AFM & $250$ & ${\color{gray!90}3.85}$ & ${\color{gray!90}5.41}$ & $4.63$ & ${\color{gray!90}125.31}$ & ${\color{gray!90}147.80}$ & $136.56$ \\
 \rowcolor{gray!7}\cellcolor{white}{} & \texttt{MF-T}~+~\texttt{CAMF} & $250$ & ${\color{gray!90}3.64}$ & ${\color{gray!90}3.87}$ & $\underline{3.75}$ & ${\color{gray!90}105.22}$ & ${\color{gray!90}78.92}$ & $\underline{92.07}$ \\
\arrayrulecolor{black!10}
\cmidrule(lr){2-9}
\arrayrulecolor{black}
   & \texttt{MF-T} & $4$ & ${\color{gray!90}8.78}$ & ${\color{gray!90}19.07}$ & $13.93$ & ${\color{gray!90}226.46}$ & ${\color{gray!90}378.24}$ & $302.35$ \\
   & \texttt{MF-T}~+~AFM & $4$ & ${\color{gray!90}3.91}$ & ${\color{gray!90}4.83}$ & $4.37$ & ${\color{gray!90}126.27}$ & ${\color{gray!90}129.79}$ & $128.03$ \\
   & \texttt{MF-T}~+~CAFM & $4$ & ${\color{gray!90}77.30}$ & ${\color{gray!90}82.33}$ & $79.81$ & ${\color{gray!90}884.94}$ & ${\color{gray!90}1076.53}$ & $980.74$ \\
 \rowcolor{gray!7}\cellcolor{white}\multirow{-7}{*}{\shortstack{\textbf{iMF-XL/2}\\[-1pt]{\scriptsize \cite{geng2026improved}}}} & \texttt{MF-T}~+~\texttt{CAMF} & $4$ & ${\color{gray!90}3.25}$ & ${\color{gray!90}3.31}$ & $\underline{\mathbf{3.28}}$ & ${\color{gray!90}91.93}$ & ${\color{gray!90}66.31}$ & $\underline{\mathbf{79.12}}$ \\
  \hline
   & FT & \multirow{1}{*}{$250{\times}2$} & ${\color{gray!90}4.25}$ & ${\color{gray!90}6.03}$ & $5.14$ & ${\color{gray!90}103.36}$ & ${\color{gray!90}165.36}$ & $134.36$ \\
   & \texttt{MF-T} & $250$ & ${\color{gray!90}6.96}$ & ${\color{gray!90}5.71}$ & $6.33$ & ${\color{gray!90}224.03}$ & ${\color{gray!90}184.35}$ & $204.19$ \\
   & \texttt{MF-T}~+~AFM & $250$ & ${\color{gray!90}6.61}$ & ${\color{gray!90}13.94}$ & $10.27$ & ${\color{gray!90}220.38}$ & ${\color{gray!90}324.12}$ & $272.25$ \\
 \rowcolor{gray!7}\cellcolor{white}{} & \texttt{MF-T}~+~\texttt{CAMF} & $250$ & ${\color{gray!90}4.02}$ & ${\color{gray!90}4.80}$ & $\underline{4.41}$ & ${\color{gray!90}129.23}$ & ${\color{gray!90}113.86}$ & $\underline{121.55}$ \\
\arrayrulecolor{black!10}
\cmidrule(lr){2-9}
\arrayrulecolor{black}
   & FT & \multirow{1}{*}{$4{\times}2$} & ${\color{gray!90}14.15}$ & ${\color{gray!90}31.84}$ & $23.00$ & ${\color{gray!90}289.02}$ & ${\color{gray!90}582.85}$ & $435.93$ \\
   & \texttt{MF-T} & $4$ & ${\color{gray!90}6.86}$ & ${\color{gray!90}4.77}$ & $5.81$ & ${\color{gray!90}231.36}$ & ${\color{gray!90}178.80}$ & $205.08$ \\
   & \texttt{MF-T}~+~AFM & $4$ & ${\color{gray!90}5.20}$ & ${\color{gray!90}6.96}$ & $6.08$ & ${\color{gray!90}192.85}$ & ${\color{gray!90}184.07}$ & $188.46$ \\
   & \texttt{MF-T}~+~CAFM & $4$ & ${\color{gray!90}437.77^\dagger}$ & ${\color{gray!90}49.80}$ & -- & ${\color{gray!90}3610.90^\dagger}$ & ${\color{gray!90}805.80}$ & -- \\
 \rowcolor{gray!7}\cellcolor{white}\multirow{-9}{*}{\shortstack{\textbf{SiT-XL/2}\\[-1pt]{\scriptsize \cite{ma2024sit}}}} & \texttt{MF-T}~+~\texttt{CAMF} & $4$ & ${\color{gray!90}3.52}$ & ${\color{gray!90}3.40}$ & $\underline{\mathbf{3.46}}$ & ${\color{gray!90}117.19}$ & ${\color{gray!90}81.10}$ & $\underline{\mathbf{99.14}}$ \\
  \bottomrule
  \end{tabular}

\raggedright
\scriptsize
\vspace{2pt}
$^\dagger$: At the corresponding setting, training diverged when run to the full budget, so its average is omitted.
\end{table*}

\begin{table*}[t]
\centering
\scriptsize
  \setlength{\tabcolsep}{4pt}
  \renewcommand{\arraystretch}{1.2}

\caption{\textbf{Target data availability vs. NFE}. Training subsets are sampled randomly within each class with fixed seed. \textbf{Bold}: best average of the two  \colorbox{gray!7}{Gray} rows.}
\label{tab:mft_data_efficiency}

\begin{tabular}{cccl *{3}{c} !{\color{black!20}\vrule} *{3}{c}}
\toprule
& & & &
\multicolumn{3}{c}{\textbf{FID} $\downarrow$} &
\multicolumn{3}{c}{\textbf{FDD} $\downarrow$} \\
\cmidrule(lr){5-7}
\cmidrule(lr){8-10}

\textbf{ImageNet Init.} & \textbf{Data} & \textbf{NFE} & \textbf{Method}
& CUB & Cars & \textbf{Avg.}
& CUB & Cars & \textbf{Avg.} \\
\midrule
 &  &  & \texttt{MF-T} & $13.20$ & $14.66$ & $13.93$ & $508.52$ & $419.28$ & $463.90$ \\
 &  & \multirow{-2}{*}{$4$} & \texttt{MF-T}~+~\texttt{CAMF} & $6.19$ & $5.57$ & $5.88$ & $314.13$ & $190.74$ & $252.43$ \\
\arrayrulecolor{black!10}\cmidrule(lr){3-10}\arrayrulecolor{black}
 &  &  & \texttt{MF-T} & $15.50$ & $15.41$ & $15.46$ & $555.06$ & $446.31$ & $500.69$ \\
\rowcolor{gray!7}\cellcolor{white} & \cellcolor{white} & \cellcolor{white}\multirow{-2}{*}{$2$} & \texttt{MF-T}~+~\texttt{CAMF} & $6.97$ & $6.75$ & $\mathbf{6.86}$ & $351.53$ & $220.08$ & $\mathbf{285.81}$ \\[-1pt]
\arrayrulecolor{black!10}\cmidrule(lr){3-10}\arrayrulecolor{black}
 &  &  & \texttt{MF-T} & $16.24$ & $17.42$ & $16.83$ & $601.55$ & $522.54$ & $562.05$ \\
 & \multirow{-6}{*}{$10\%$} & \multirow{-2}{*}{$1$} & \texttt{MF-T}~+~\texttt{CAMF} & $10.39$ & $23.93$ & $17.16$ & $472.30$ & $503.15$ & $487.73$ \\
\cmidrule(lr){2-10}
 &  &  & \texttt{MF-T} & $11.53$ & $11.49$ & $11.51$ & $529.71$ & $298.55$ & $414.13$ \\
 &  & \multirow{-2}{*}{$4$} & \texttt{MF-T}~+~\texttt{CAMF} & $5.95$ & $3.26$ & $4.61$ & $337.38$ & $100.72$ & $219.05$ \\
\arrayrulecolor{black!10}\cmidrule(lr){3-10}\arrayrulecolor{black}
 &  &  & \texttt{MF-T} & $12.16$ & $15.34$ & $13.75$ & $543.37$ & $382.85$ & $463.11$ \\
 &  & \multirow{-2}{*}{$2$} & \texttt{MF-T}~+~\texttt{CAMF} & $6.52$ & $4.14$ & $5.33$ & $352.05$ & $135.58$ & $243.81$ \\
\arrayrulecolor{black!10}\cmidrule(lr){3-10}\arrayrulecolor{black}
 &  &  & \texttt{MF-T} & $13.68$ & $19.16$ & $16.42$ & $585.98$ & $473.30$ & $529.64$ \\[-1pt]
 & \multirow{-6}{*}{$25\%$} & \multirow{-2}{*}{$1$} & \texttt{MF-T}~+~\texttt{CAMF} & $7.78$ & $19.86$ & $13.82$ & $404.15$ & $429.93$ & $417.04$ \\
\cmidrule(lr){2-10}
\rowcolor{gray!7}\cellcolor{white} & \cellcolor{white} & \cellcolor{white} & \texttt{MF-T} & $7.26$ & $9.69$ & $8.48$ & $310.62$ & $276.53$ & $293.58$ \\
 &  & \multirow{-2}{*}{$4$} & \texttt{MF-T}~+~\texttt{CAMF} & $2.88$ & $3.07$ & $2.98$ & $84.24$ & $102.37$ & $93.31$ \\
\arrayrulecolor{black!10}\cmidrule(lr){3-10}\arrayrulecolor{black}
 &  &  & \texttt{MF-T} & $8.46$ & $12.32$ & $10.39$ & $222.08$ & $346.75$ & $284.42$ \\
 &  & \multirow{-2}{*}{$2$} & \texttt{MF-T}~+~\texttt{CAMF} & $3.24$ & $3.38$ & $3.31$ & $95.20$ & $120.44$ & $107.82$ \\
\arrayrulecolor{black!10}\cmidrule(lr){3-10}\arrayrulecolor{black}
 &  &  & \texttt{MF-T} & $11.74$ & $15.41$ & $13.58$ & $423.50$ & $403.39$ & $413.45$ \\
\multirow{-18}{*}{\begin{tabular}{@{}c@{}}\textbf{iMF-XL/2}\\[-1pt]\cite{geng2026improved}\end{tabular}} & \multirow{-6}{*}{$100\%$} & \multirow{-2}{*}{$1$} & \texttt{MF-T}~+~\texttt{CAMF} & $5.40$ & $15.51$ & $10.46$ & $179.49$ & $352.35$ & $265.92$ \\
\bottomrule
\end{tabular}
\vspace{2pt}
\end{table*}

\newcommand{\dittwolinenfeZ}{\includegraphics[width=\colw,height=\colw,keepaspectratio]{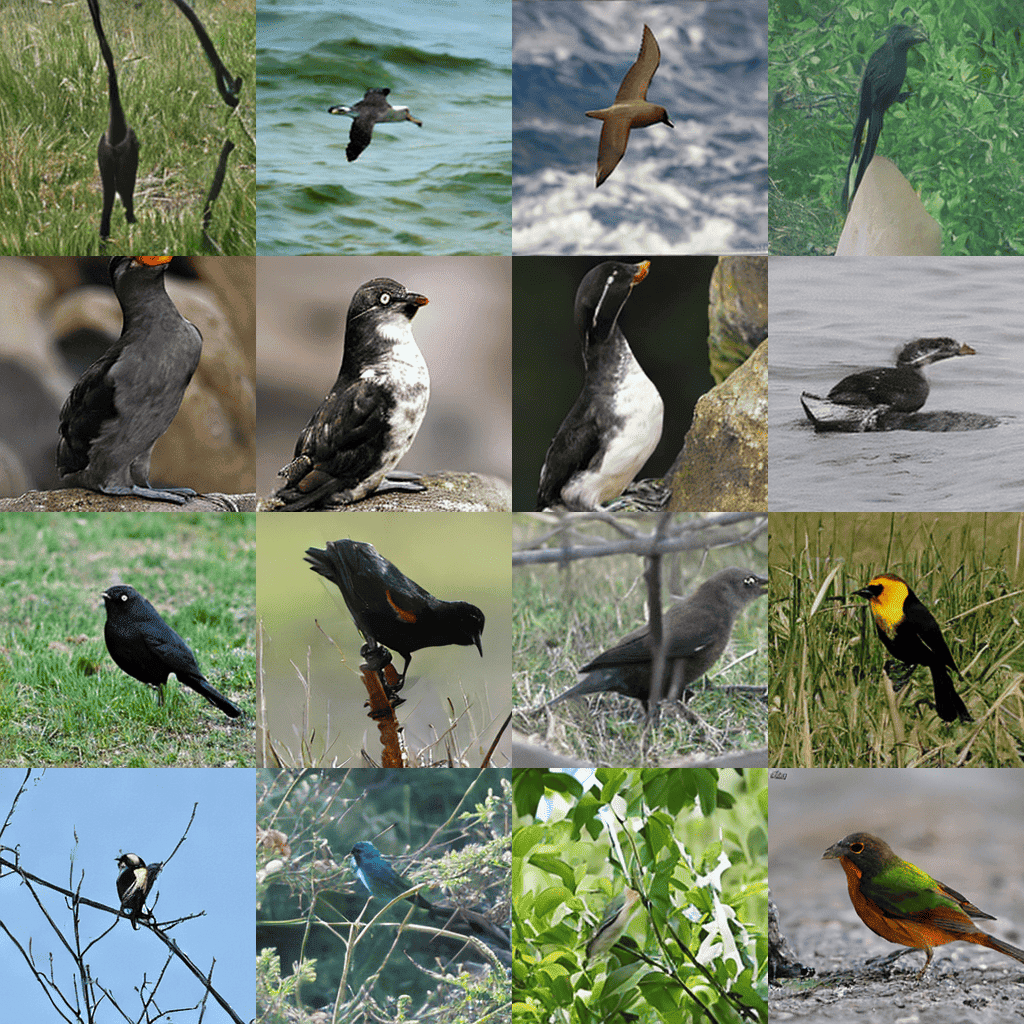}}

\newcommand{\jittwolinenfeZ}{\includegraphics[width=\colw,height=\colw,keepaspectratio]{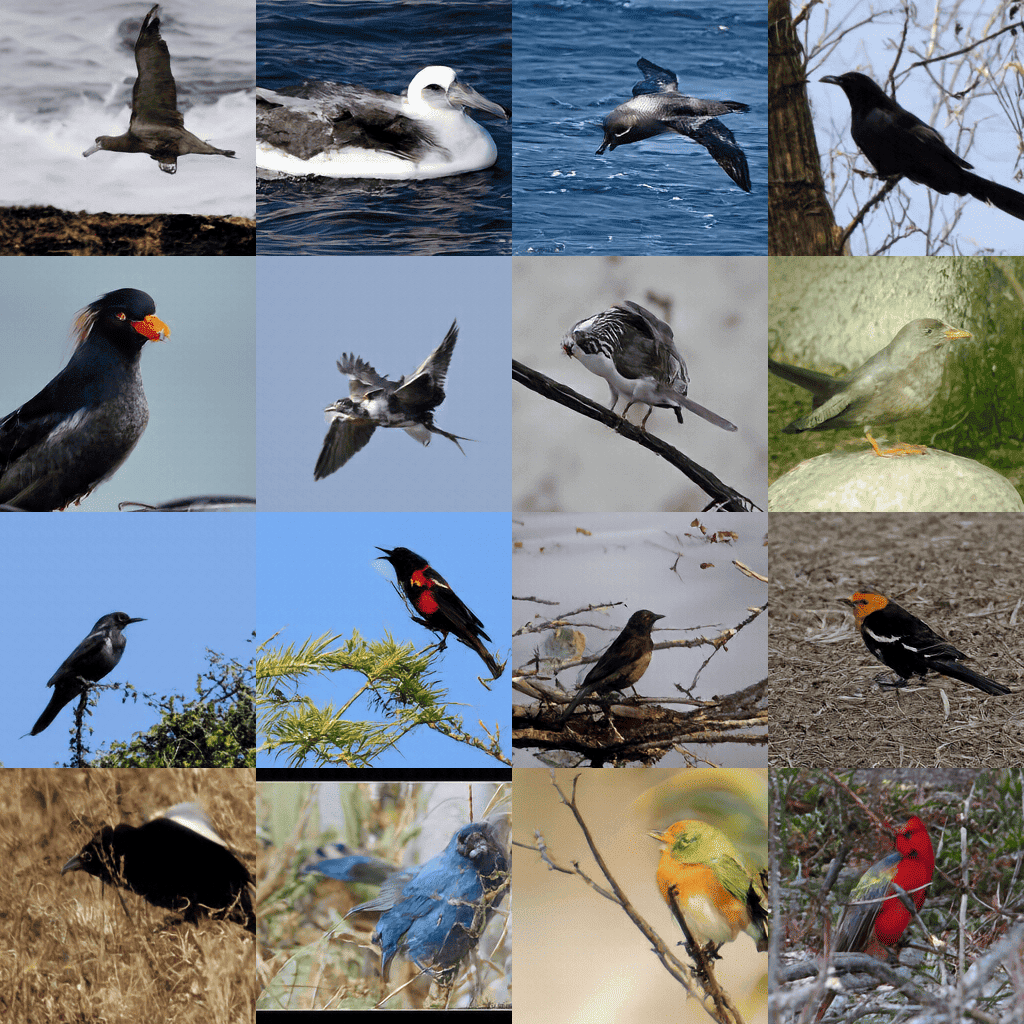}}

\newcommand{\imftwolinesfournfeZ}{\includegraphics[width=\colw,height=\colw,keepaspectratio]{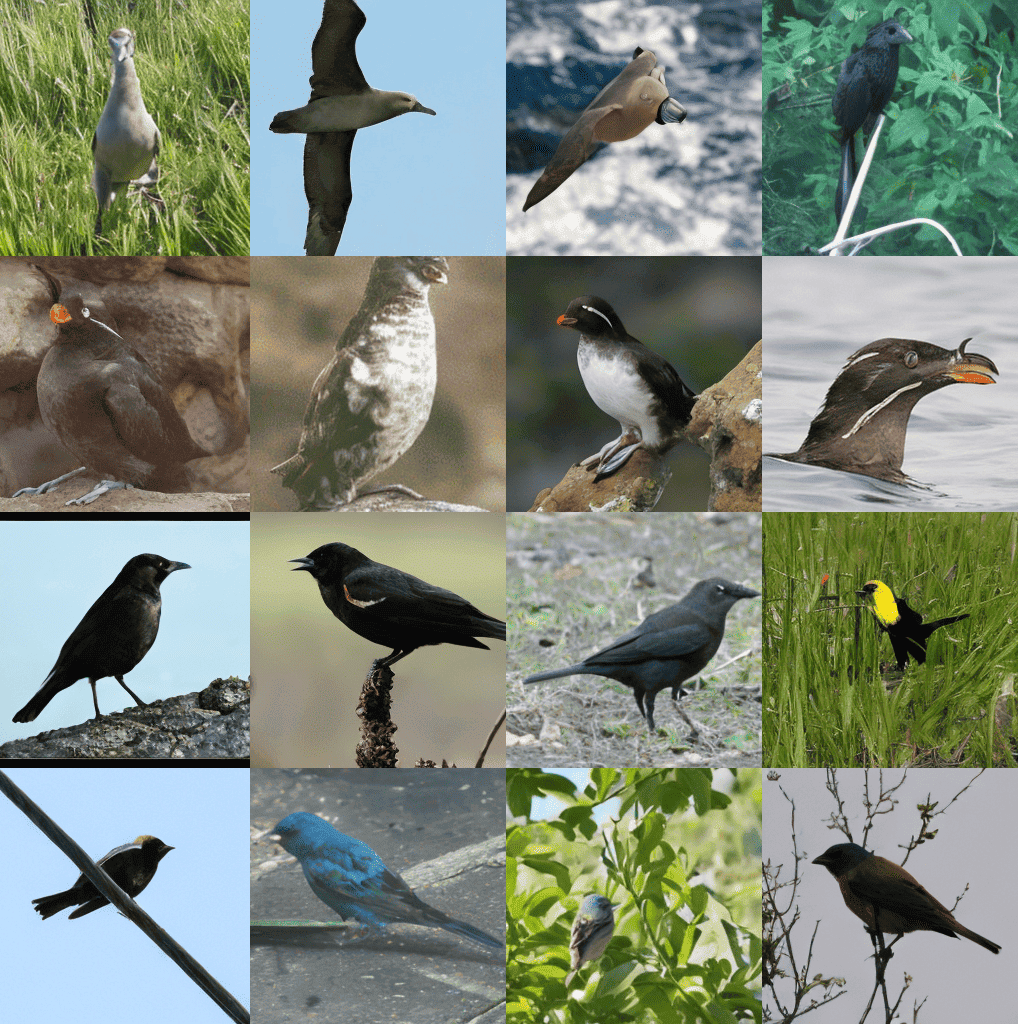}}

\newcommand{\sittwolinenfeZ}{\includegraphics[width=\colw,height=\colw,keepaspectratio]{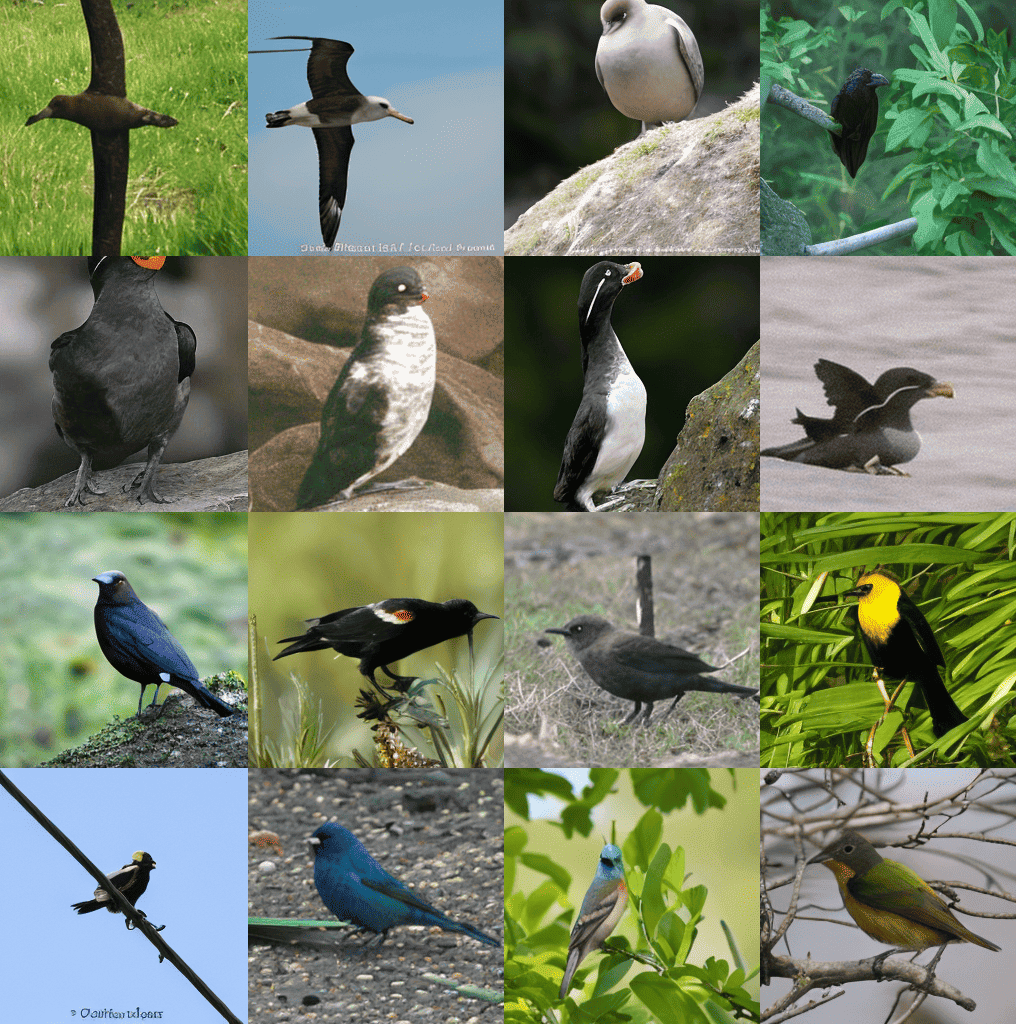}}

\newcommand{\dittwolinenfeartZ}{\includegraphics[width=\colw,height=\colw,keepaspectratio]{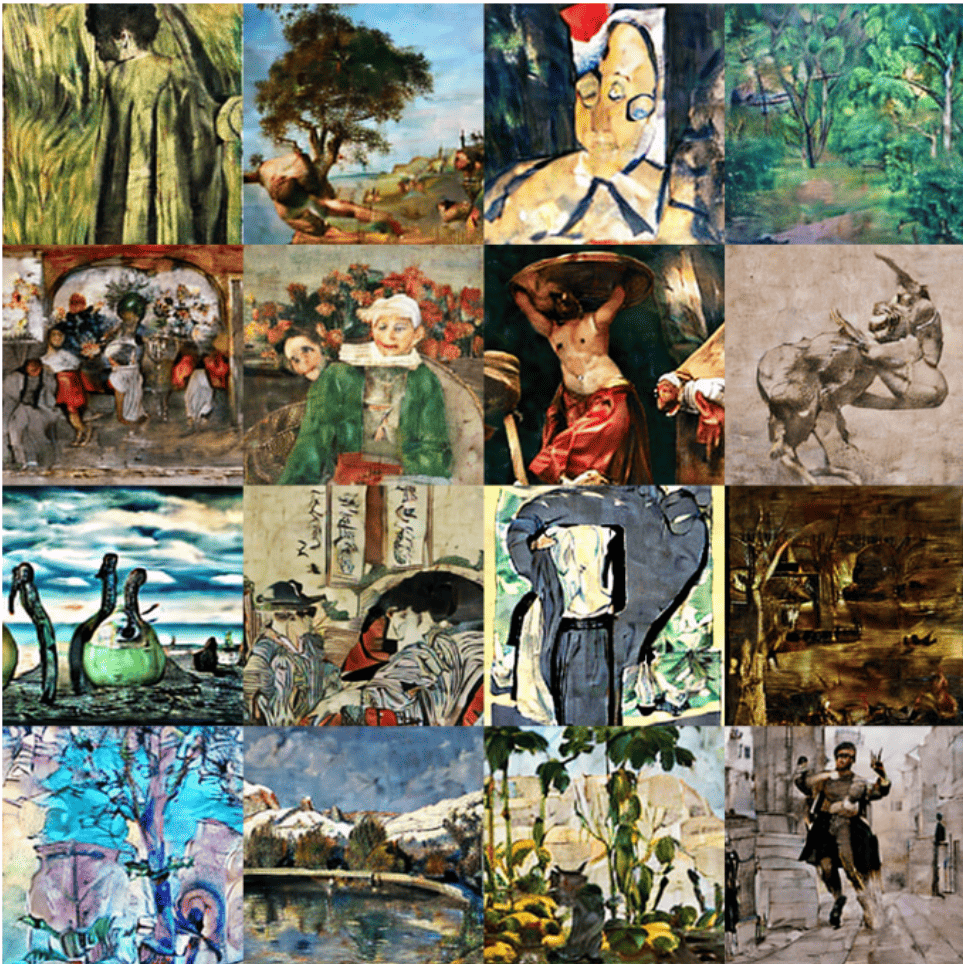}}

\newcommand{\jittwolinenfeartZ}{\includegraphics[width=\colw,height=\colw,keepaspectratio]{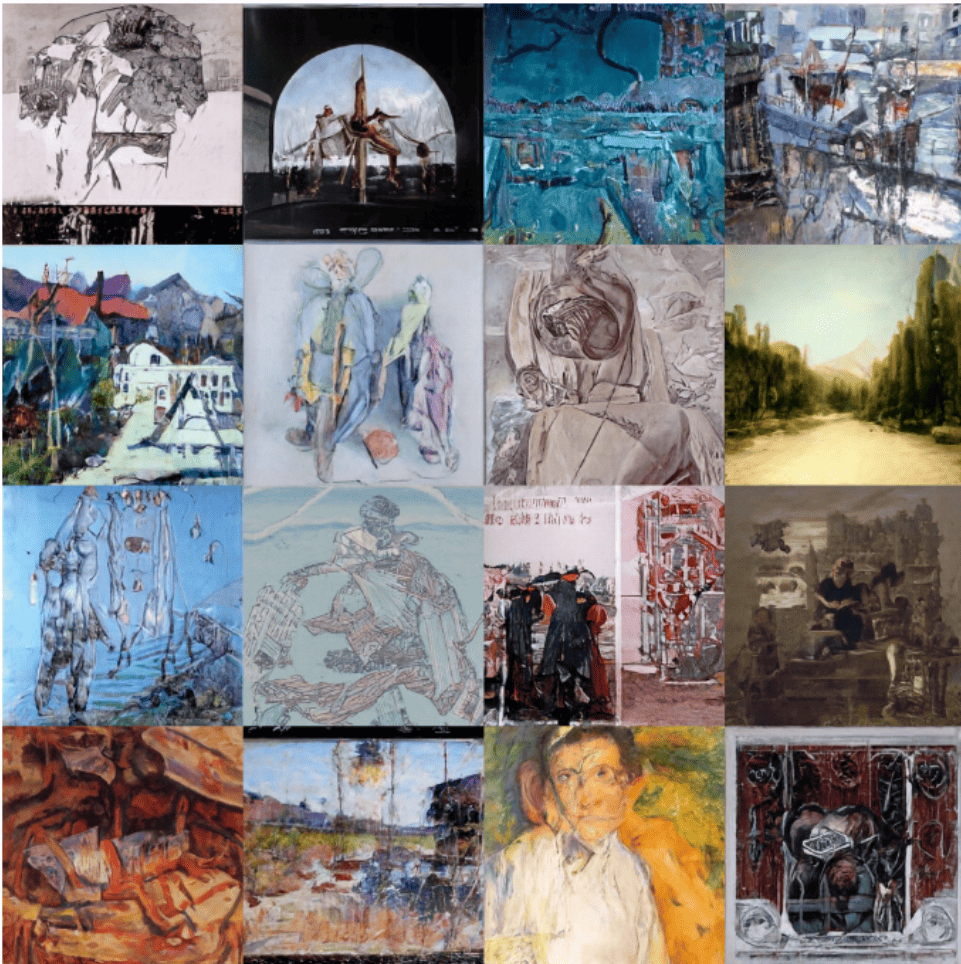}}

\newcommand{\imftwolinesfournfeartZ}{\includegraphics[width=\colw,height=\colw,keepaspectratio]{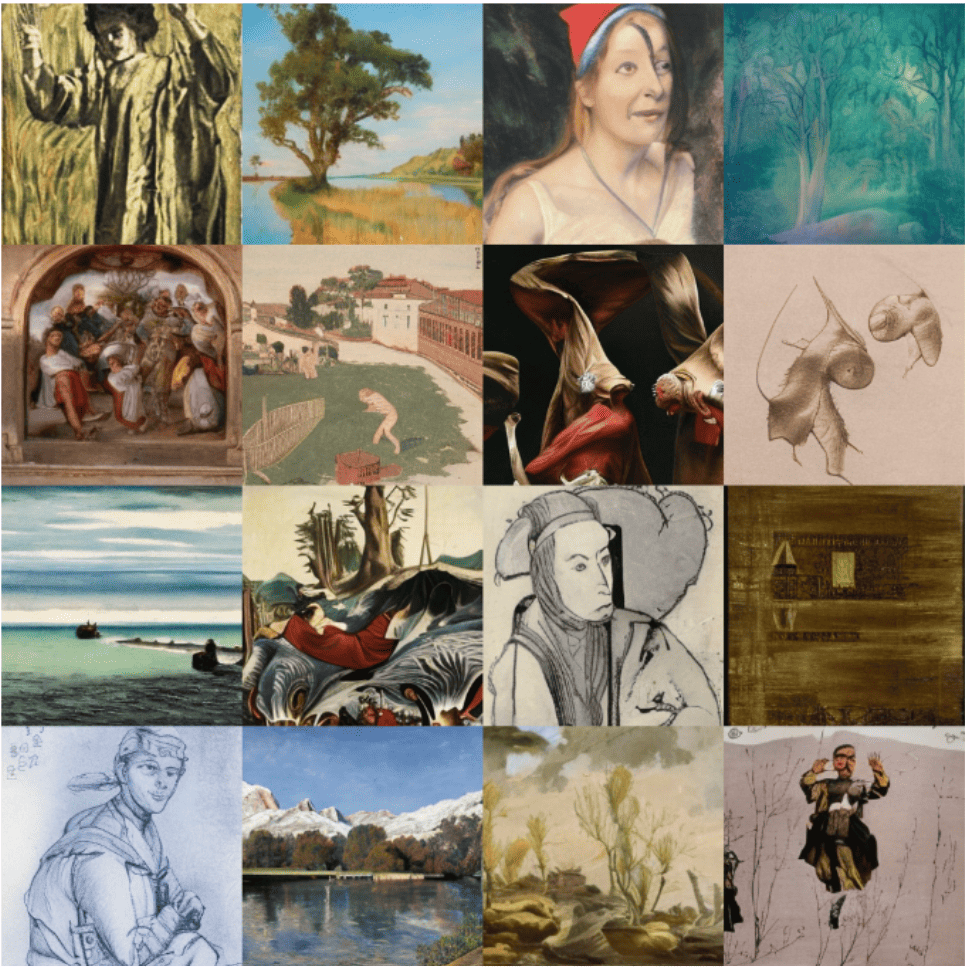}}

\newcommand{\sittwolinenfeartZ}{\includegraphics[width=\colw,height=\colw,keepaspectratio]{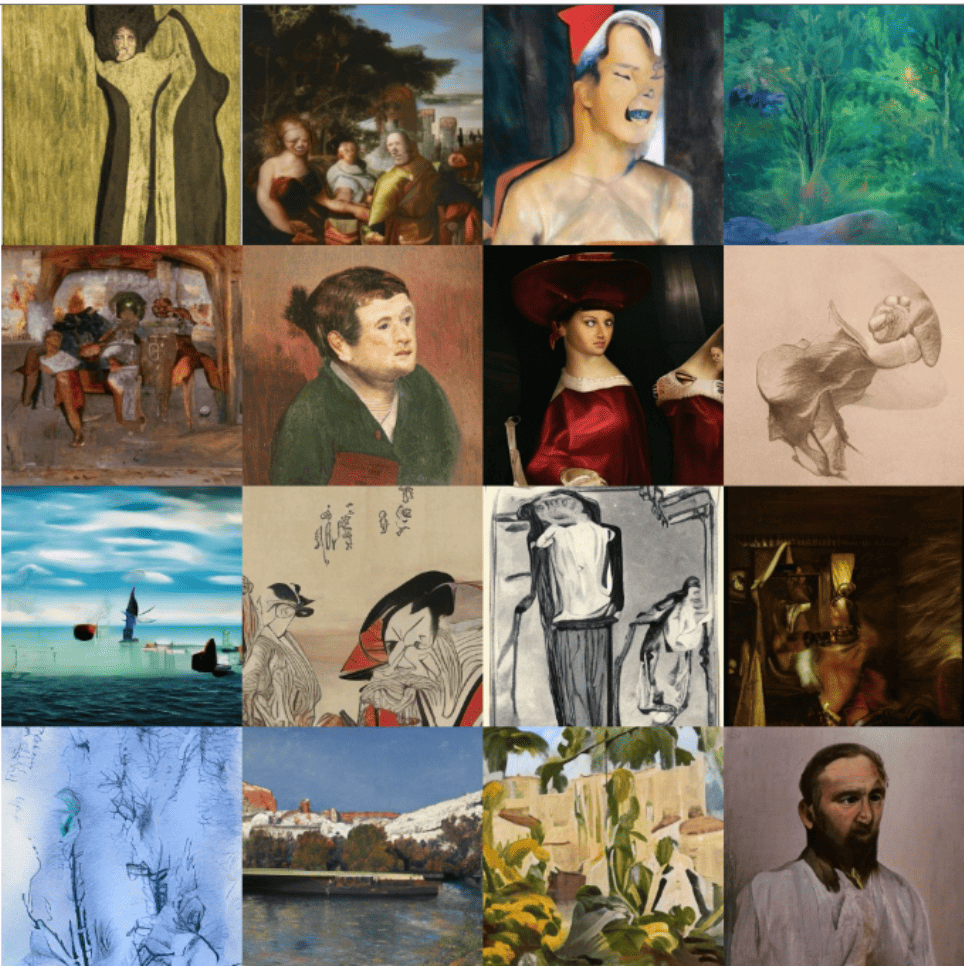}}

\newcommand{\dittwolinenfecarsZ}{\includegraphics[width=\colw,height=\colw,keepaspectratio]{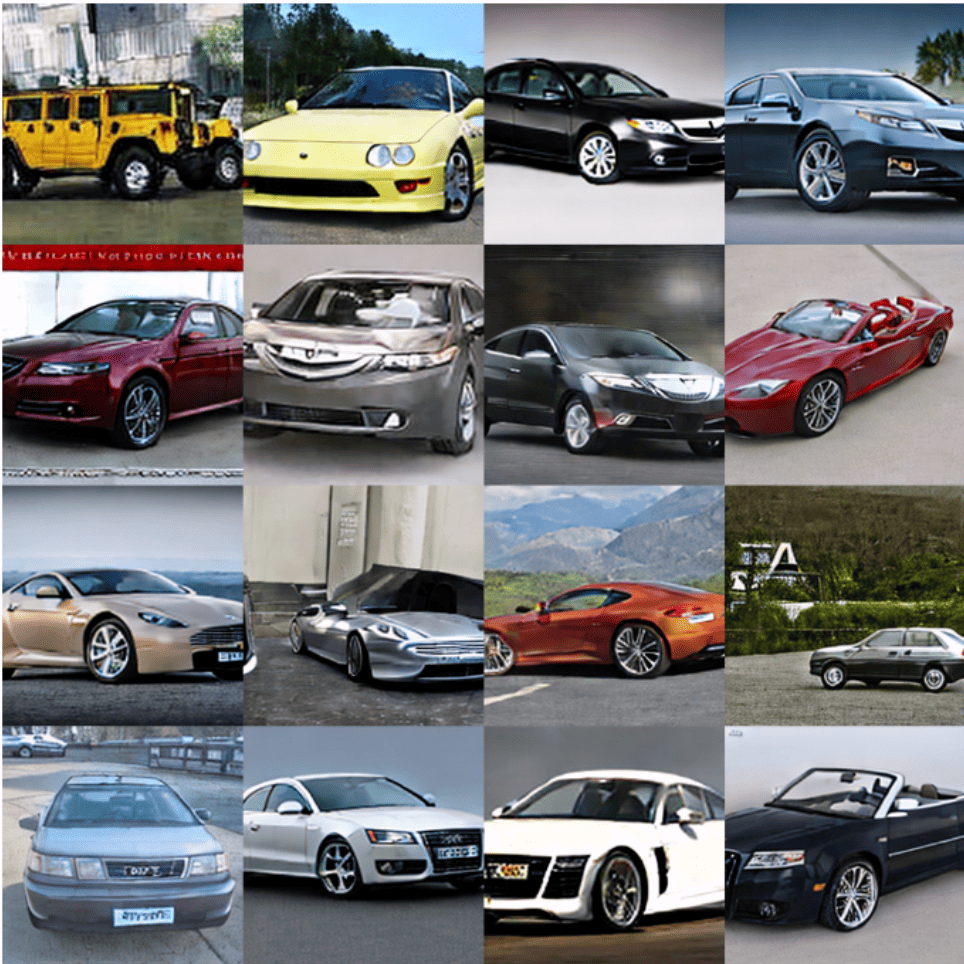}}

\newcommand{\jittwolinenfecarsZ}{\includegraphics[width=\colw,height=\colw,keepaspectratio]{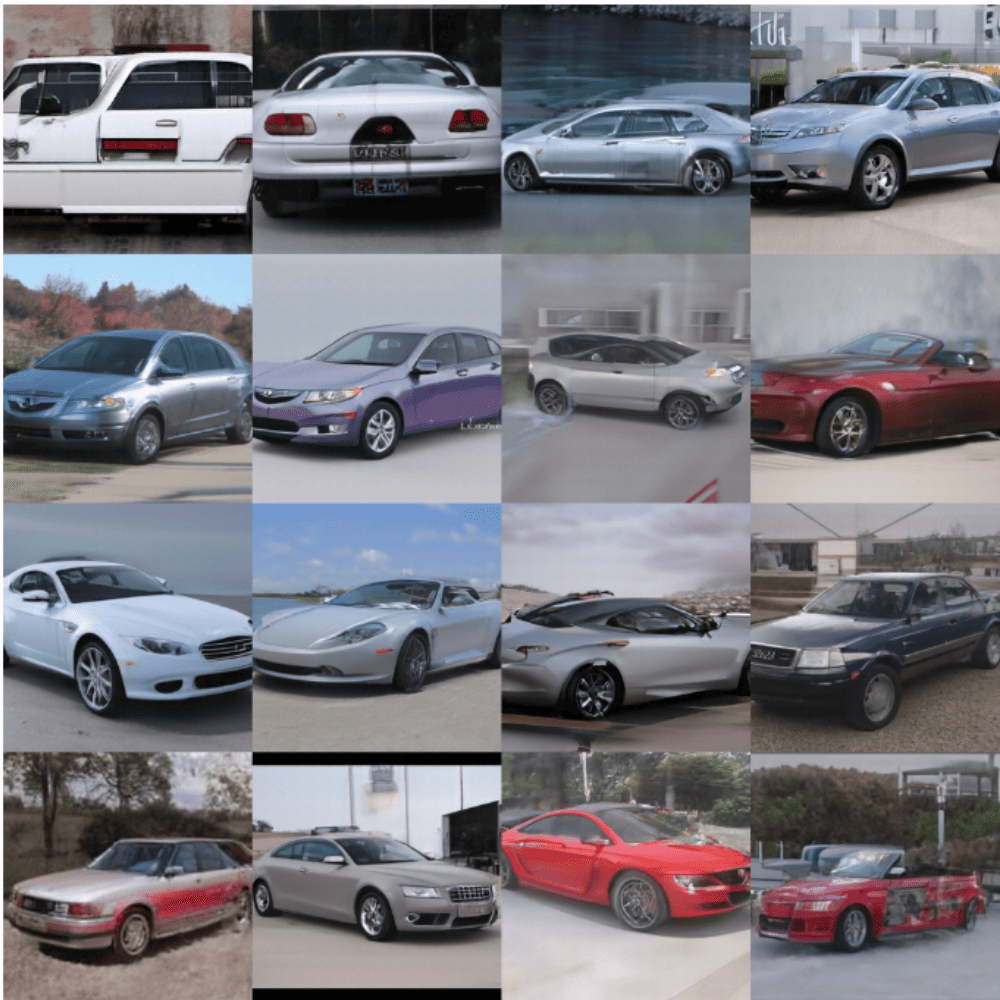}}

\newcommand{\imftwolinesfournfecarsZ}{\includegraphics[width=\colw,height=\colw,keepaspectratio]{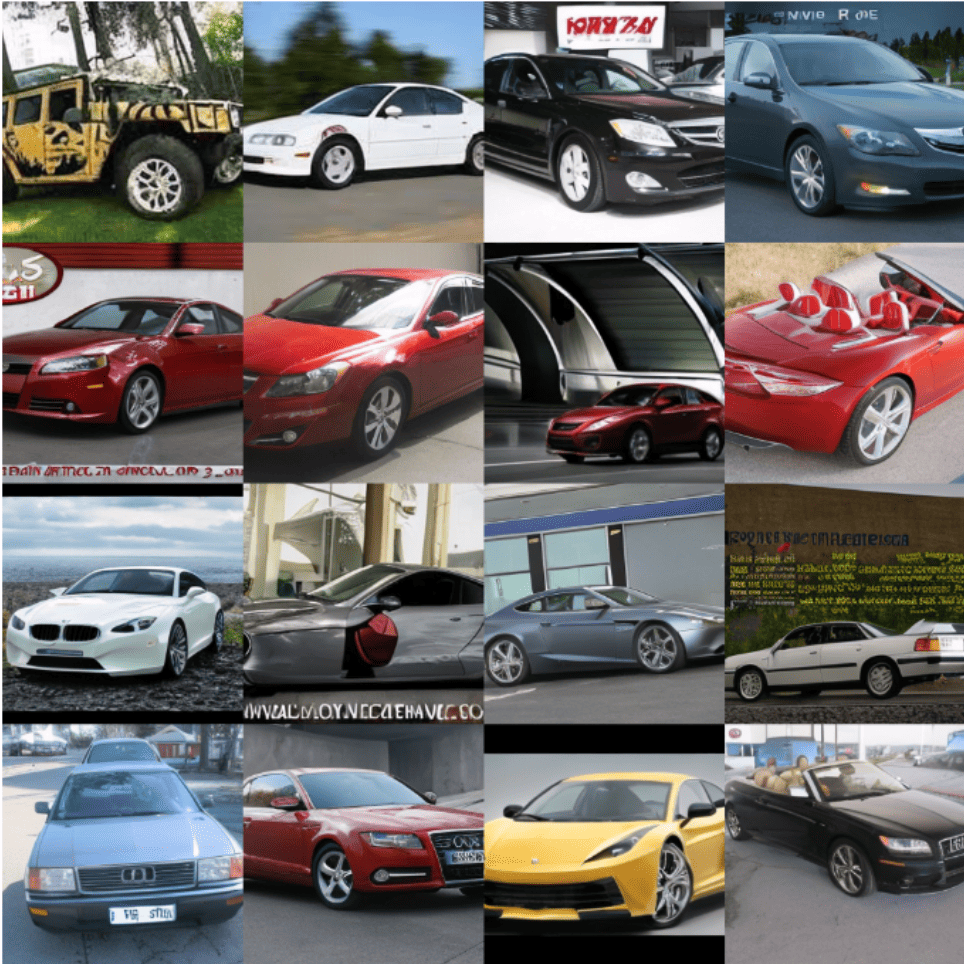}}

\newcommand{\sittwolinenfecarsZ}{\includegraphics[width=\colw,height=\colw,keepaspectratio]{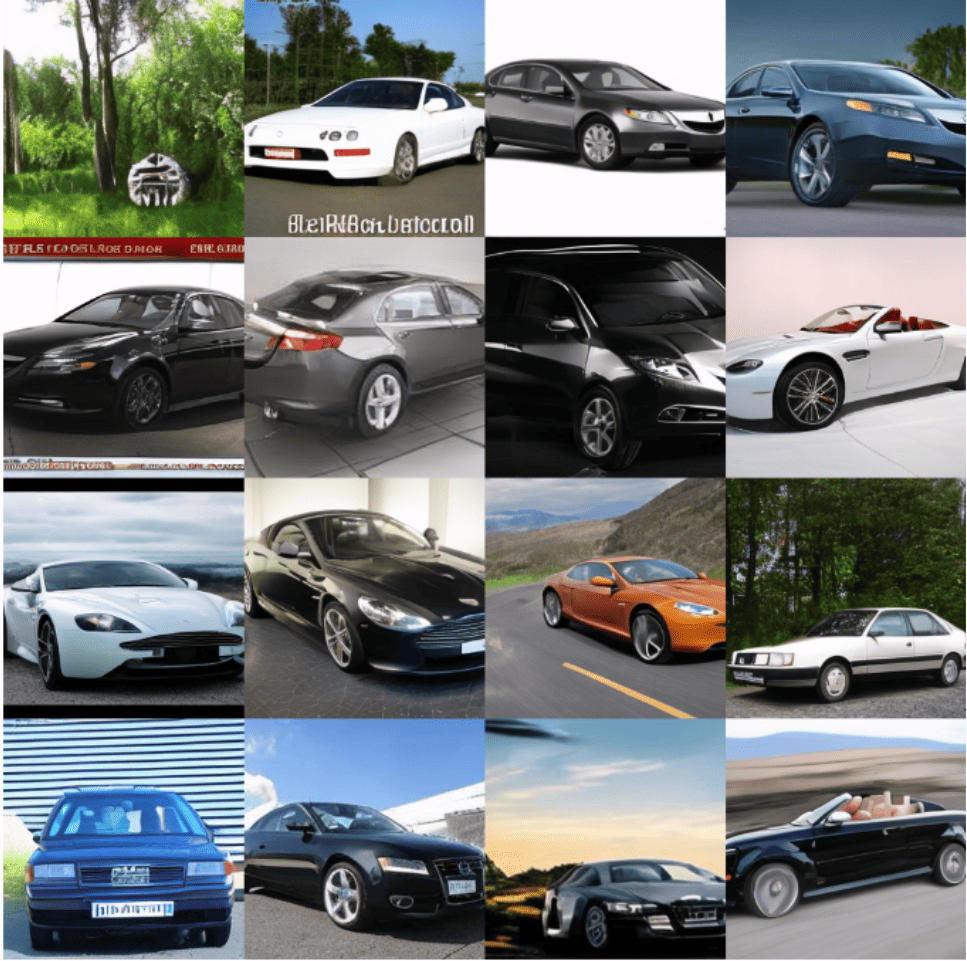}}

\newcommand{\dittwolinenfecaltZ}{\includegraphics[width=\colw,height=\colw,keepaspectratio]{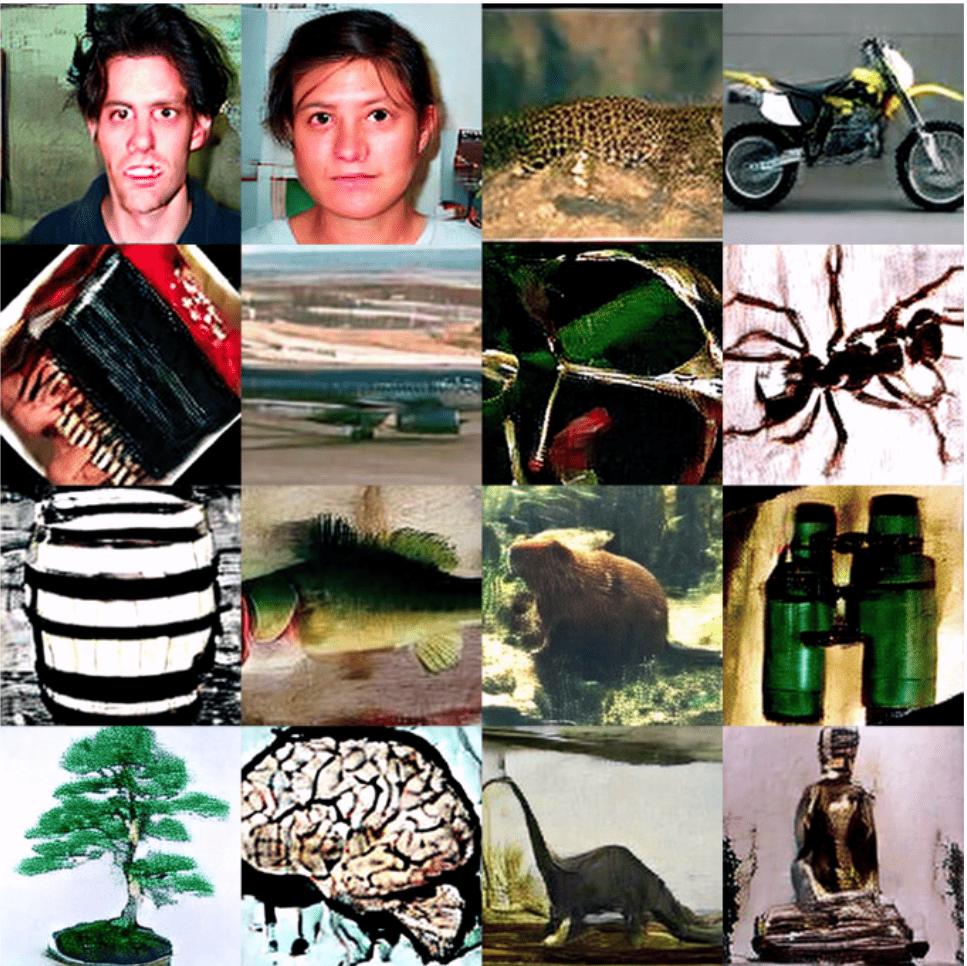}}

\newcommand{\jittwolinenfecaltZ}{\includegraphics[width=\colw,height=\colw,keepaspectratio]{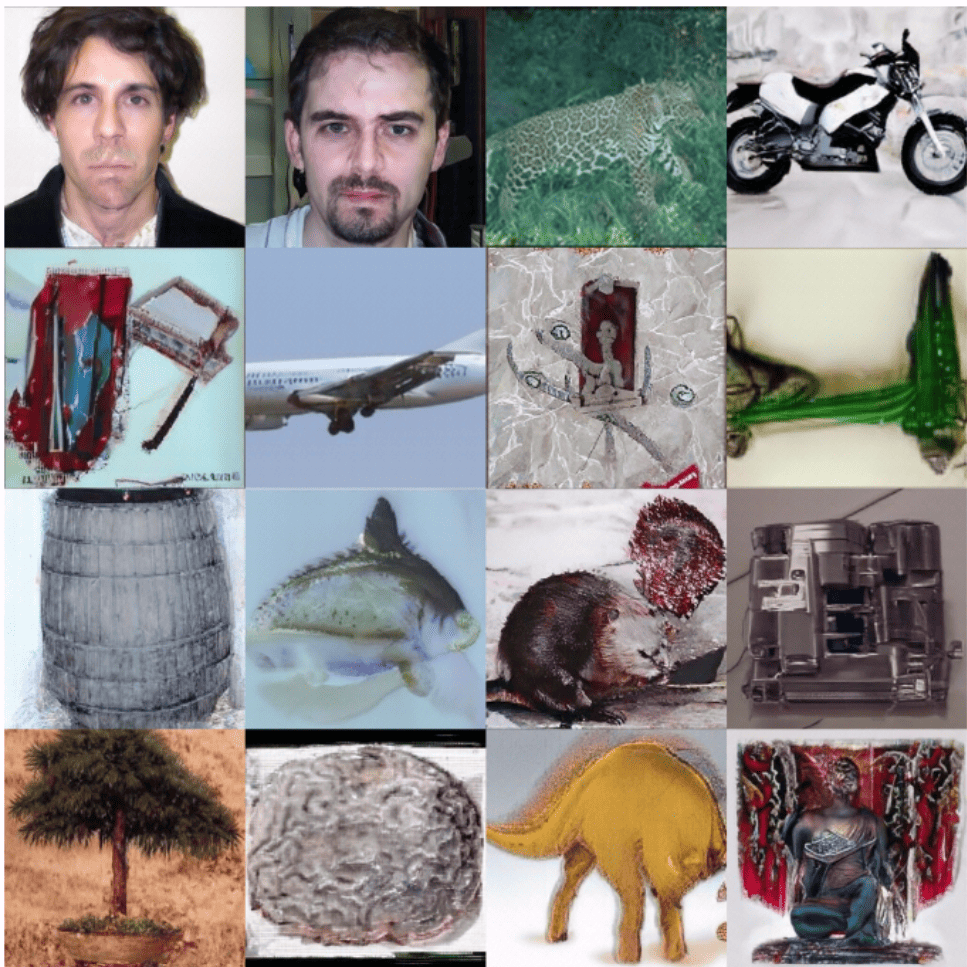}}

\newcommand{\imftwolinesfournfecaltZ}{\includegraphics[width=\colw,height=\colw,keepaspectratio]{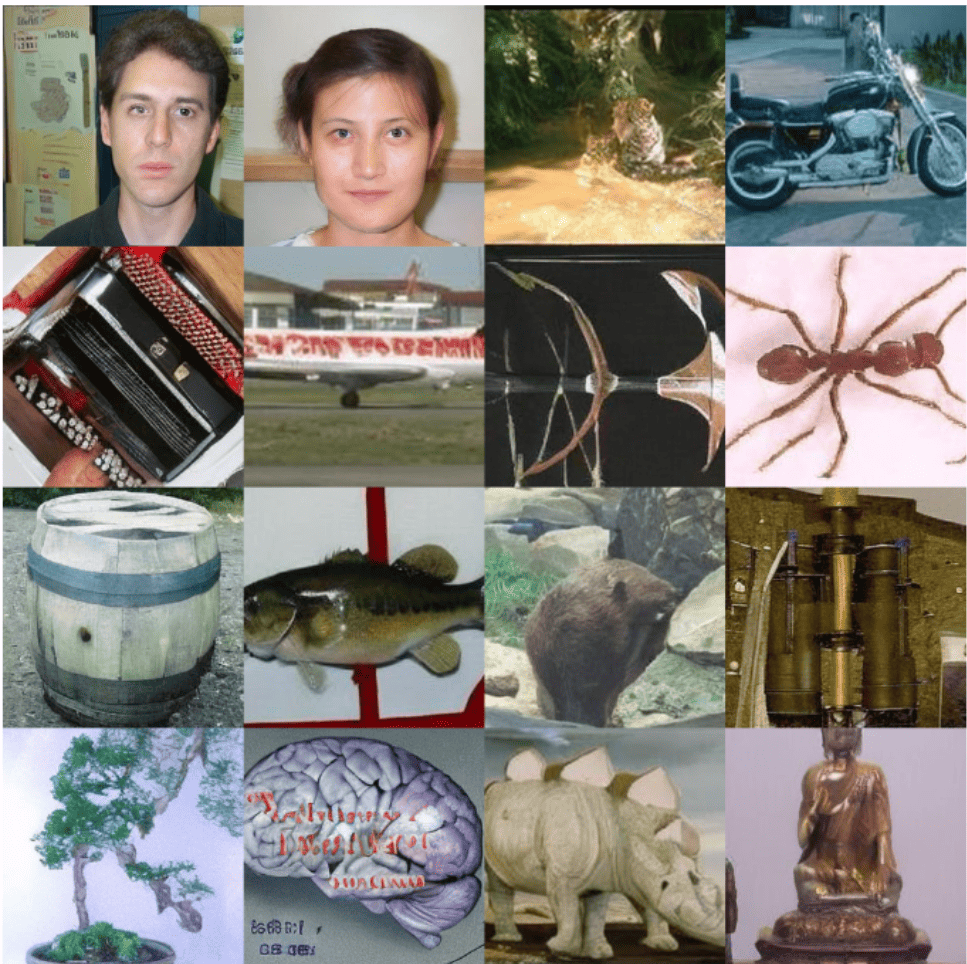}}

\newcommand{\sittwolinenfecaltZ}{\includegraphics[width=\colw,height=\colw,keepaspectratio]{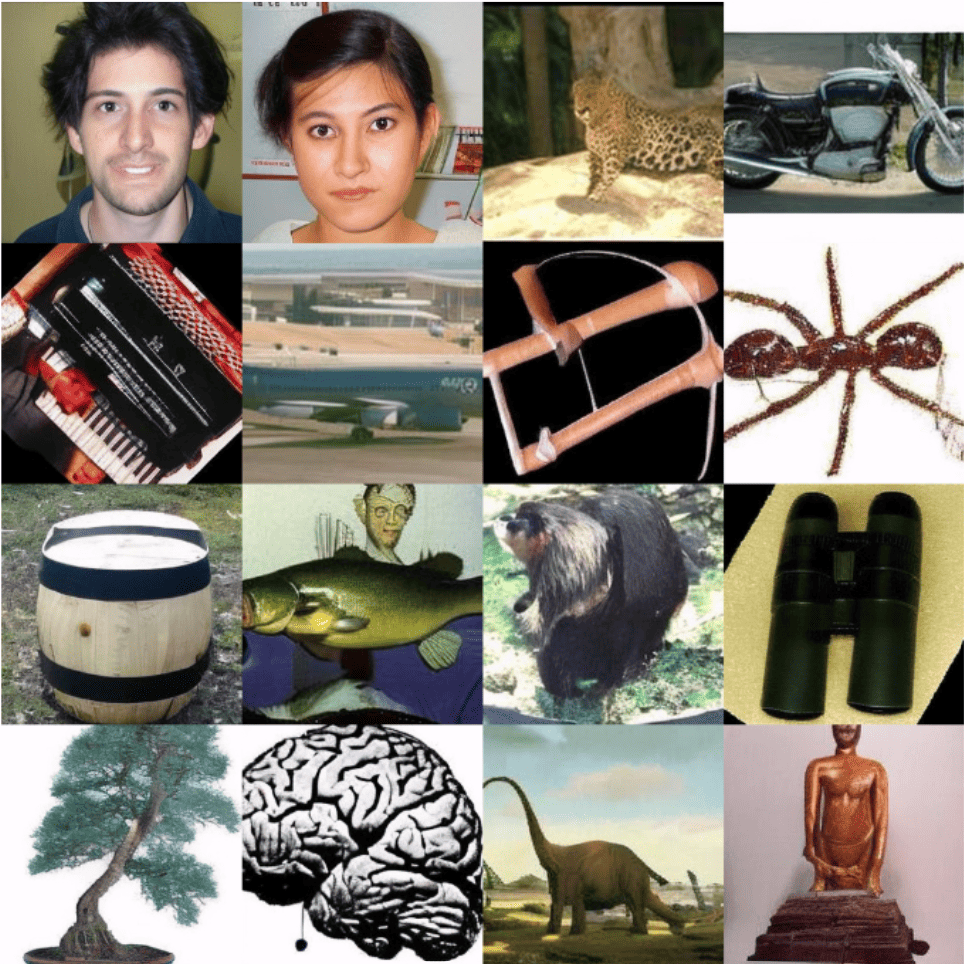}}

\newcommand{\dittwolinenfefoodZ}{\includegraphics[width=\colw,height=\colw,keepaspectratio]{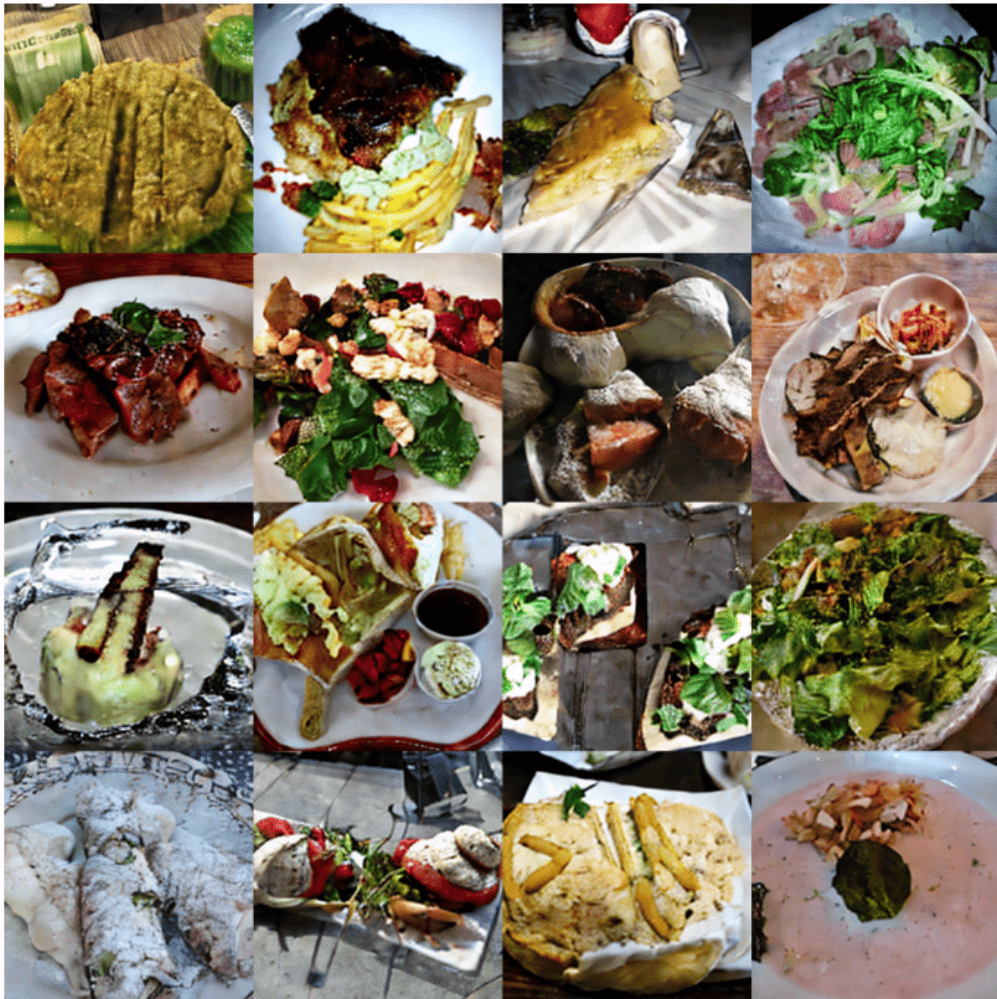}}

\newcommand{\jittwolinenfefoodZ}{\includegraphics[width=\colw,height=\colw,keepaspectratio]{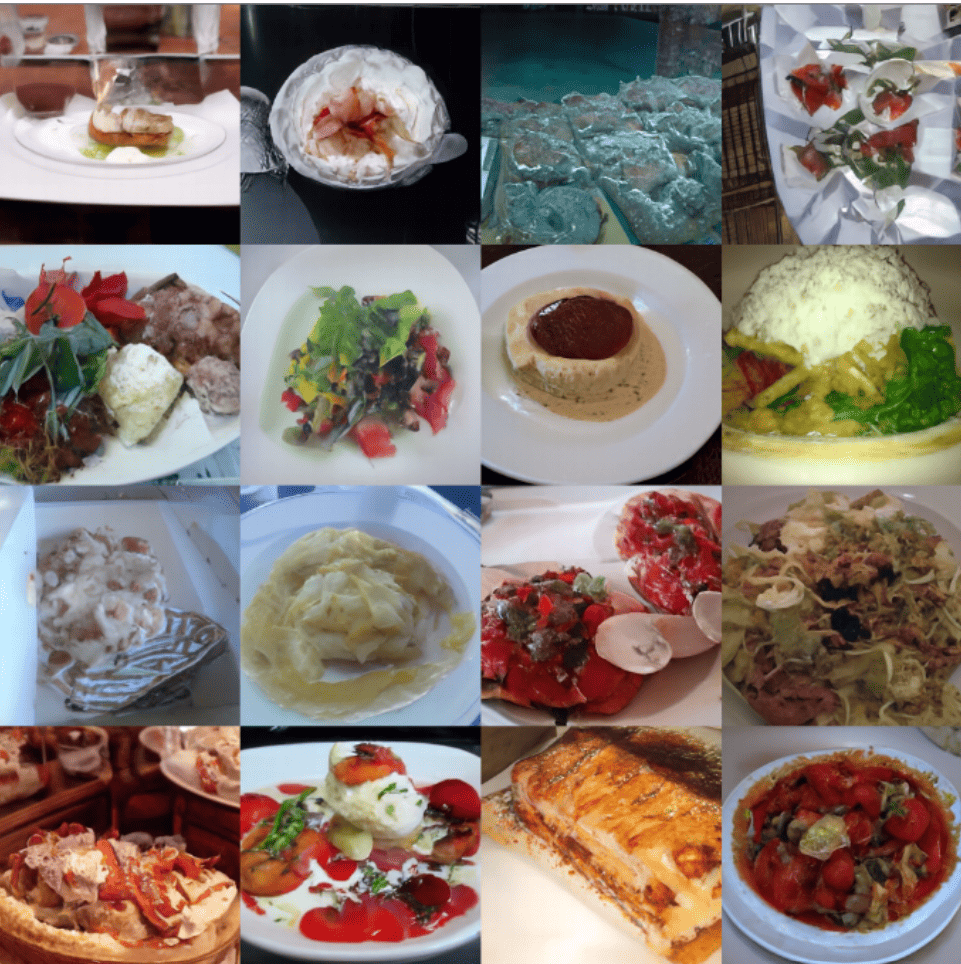}}

\newcommand{\imftwolinesfournfefoodZ}{\includegraphics[width=\colw,height=\colw,keepaspectratio]{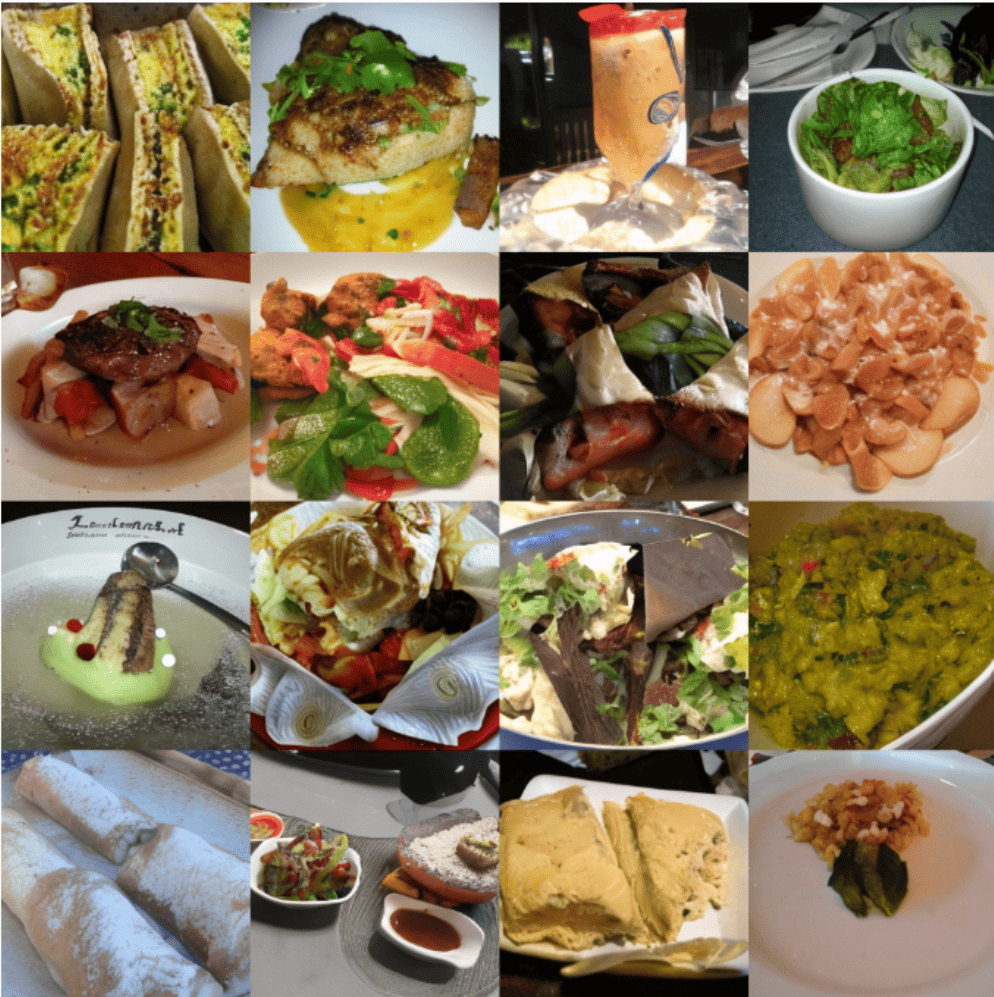}}

\newcommand{\sittwolinenfefoodZ}{\includegraphics[width=\colw,height=\colw,keepaspectratio]{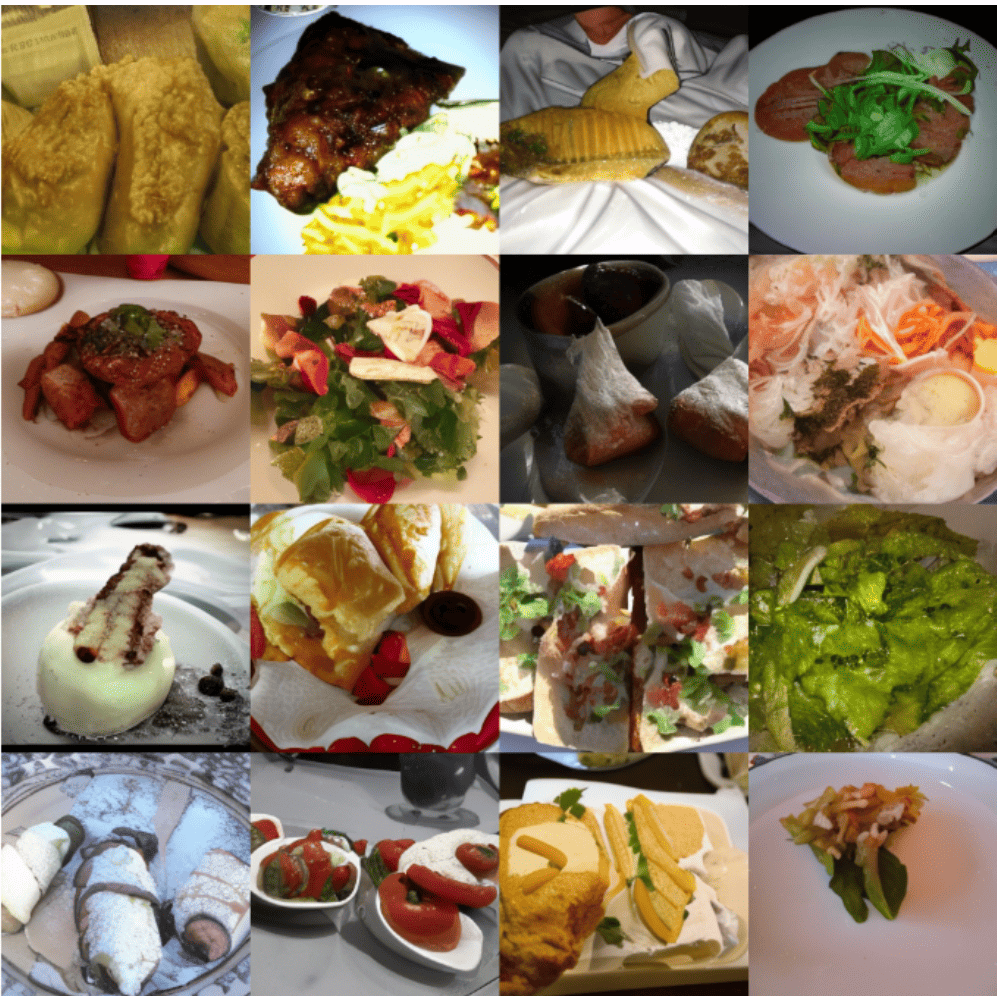}}
\newcolumntype{Z}[1]{>{\centering\arraybackslash}m{#1}}

\begin{figure*}[h]
\setlength{\aboverulesep}{0.4ex}\setlength{\belowrulesep}{0.65ex}%
\centering
\scriptsize

\vspace{-5pt}

\begin{tabularx}{\linewidth}{
    @{}
    Z{0.03\textwidth}
    Y Y Y Y
    @{}
}
\toprule
\textbf{Dataset}
&
\textbf{iMF-XL/2} ($4$ NFE)
&
\textbf{SiT-XL/2} ($4$ NFE) 
&
\textbf{DiT-XL/2} ($4$ NFE)
&
\textbf{JiT-H/16} ($4$ NFE) 

\\

\midrule
\textbf{ArtB.}
&
\imftwolinesfournfeartZ
&
\sittwolinenfeartZ
&
\dittwolinenfeartZ
&
\jittwolinenfeartZ
\\[-2pt]

\midrule
\textbf{Calt.}
&
\imftwolinesfournfecaltZ
&
\sittwolinenfecaltZ
&
\dittwolinenfecaltZ
&
\jittwolinenfecaltZ
\\[-2pt]

\midrule
 \textbf{CUB}
&
\imftwolinesfournfeZ
&
\sittwolinenfeZ
&
\dittwolinenfeZ
&
\jittwolinenfeZ
\\[-2pt]

\midrule
\textbf{Food}
&
\imftwolinesfournfefoodZ
&
\sittwolinenfefoodZ
&
\dittwolinenfefoodZ
&
\jittwolinenfefoodZ
\\[-2pt]

\midrule
\textbf{Cars}
&
\imftwolinesfournfecarsZ
&
\sittwolinenfecarsZ
&
\dittwolinenfecarsZ
&
\jittwolinenfecarsZ
\\[-2pt]

\bottomrule

\end{tabularx}
\vspace{-6pt}
\caption{\textbf{\texttt{MF-T}~+~\texttt{CAMF} target-domain generated samples}. Initialized from different source models. Images are uncurated.} 
\label{fig:qualitative-appendix}

\vspace{2pt}
\end{figure*}

\end{document}